\documentclass[twoside,11pt]{article}

\usepackage{jmlr2e}

\usepackage[utf8]{inputenc} 
\usepackage[T1]{fontenc}    
\usepackage{url}            
\usepackage{booktabs}       
\usepackage{amsfonts}       
\usepackage{nicefrac}       
\usepackage{microtype}      
\usepackage{amsmath}
\usepackage{amsfonts}
\usepackage{bm}
\usepackage{mleftright}
\usepackage{xparse}
\usepackage{subfig}
\usepackage[noend]{algpseudocode}
\usepackage{dsfont}
\usepackage{floatrow}
\usepackage[symbol]{footmisc}
\usepackage{float}
\usepackage{algorithm}
\usepackage{caption}
\usepackage{dirtytalk}
\usepackage{multirow}
\usepackage{bbm}
\usepackage{mdframed}
\usepackage{wrapfig}
\usepackage{url}
\usepackage{booktabs}
\usepackage{amsfonts}
\usepackage{nicefrac}
\usepackage{microtype}
\usepackage{apalike}
\usepackage{dsfont}
\usepackage{xcolor}
\usepackage{siunitx}
\usepackage{appendix}
\usepackage{xspace}
\usepackage{stmaryrd}
\usepackage{thmtools}
\usepackage{thmbox}
\usepackage{thm-restate}
\usepackage{tikz}
\usepackage{pgfplots}
\usepackage{adjustbox}
\usepackage{pgf}
\usepackage{natbib}
\usepackage{hyperref}

\usepgfplotslibrary{groupplots}

\allowdisplaybreaks

\newcolumntype{R}[2]{%
    >{\adjustbox{angle=#1,lap=\width-(#2)}\bgroup}%
    l%
    <{\egroup}%
}

\setcitestyle{citesep={,}}  

\newmdenv[topline=false,rightline=false,bottomline=false,nobreak=false]{leftlinebox}

\newtheorem{proposition2}{Proposition}

\newtheorem{lemma}{Lemma}

\newtheorem{utheorem}{Theorem}{}{}

\newcommand{\w}{\mathbf{w}}

\newcommand{\wbar}{\underline{\mathbf{w}}}

\newcommand{\wstar}{\mathbf{w_\star}}

\newcommand{\Z}{\mathcal{Z}}
\def\vv{{\bm{v}}}

\newcommand{\fstar}{f_\star}

\def\vw{{\w}}

\newtheorem{innercustomgeneric}{\customgenericname}
\providecommand{\customgenericname}{}

\newfloatcommand{capbtabbox}{table}[][\FBwidth]

\newsavebox{\leftbox}
\newsavebox{\rightbox}

\newcommand{\Arrow}[1]{%
\parbox{#1}{\tikz{\draw[->](0,0)--(#1,0);\draw[<-](0,-15)--(#1,-15);}}
}

\DeclareMathOperator*{\argmax}{argmax}
\DeclareMathOperator*{\argmin}{argmin}

\jmlrheading{20}{2021}{1-28}{6/20; Revised 04/21}{11/21}{20-1248}{Alasdair Paren, Leonard Berrada, Rudra P. K. Poudel and M. Pawan Kumar}

\ShortHeadings{A Stochastic Bundle Method for Interpolation}{Paren, Berrada, Poudel and Kumar}

\firstpageno{1}

\title{A Stochastic Bundle Method for Interpolating Networks}

\author{\name Alasdair Paren \footnotemark \email alasdair.paren@gmail.com \\
       \addr Department of Engineering Science\\
       University of Oxford\\
       Oxford, UK
       \AND
       \name Leonard Berrada \footnotemark[\value{footnote}] 
       \email  lberrada@robots.ox.ac.uk \\
       \addr Department of Engineering Science\\
       University of Oxford\\
       Oxford, UK
       \AND
       \name Rudra P. K. Poudel
       \email rudra.poudel@crl.toshiba.co.uk \\
       \addr Cambridge Research Laboratory,\\
            Toshiba Europe Ltd,\\
            Cambridge, UK.
       \AND
       \name M. Pawan Kumar \email pawan@robots.ox.ac.uk \\
       \addr Department of Engineering Science\\
       University of Oxford\\
       Oxford, UK.
       \footnotetext{These authors contributed equally to this work.}
       }

\editor{Julien Mairal}

\begin{document}

\maketitle

\begin{abstract}%
We propose a novel method for training deep neural networks that are capable of interpolation, that is, driving the empirical loss to zero. At each iteration, our method constructs a stochastic approximation of the learning objective. The approximation, known as a bundle, is a pointwise maximum of linear functions. Our bundle contains a constant function that lower bounds the empirical loss. This enables us to compute an automatic adaptive learning rate, thereby providing an accurate solution. In addition, our bundle includes linear approximations computed at the current iterate and other linear estimates of the DNN parameters. The use of these additional approximations makes our method significantly more robust to its hyperparameters. Based on its desirable empirical properties, we term our method Bundle Optimisation for Robust and Accurate Training (BORAT). In order to operationalise BORAT, we design a novel algorithm for optimising the bundle approximation efficiently at each iteration. We establish the theoretical convergence of BORAT in both convex and non-convex settings. Using standard publicly available data sets, we provide a thorough comparison of BORAT to other single hyperparameter optimisation algorithms. Our experiments demonstrate BORAT matches the state-of-the-art generalisation performance for these methods and is the most robust.


\setcounter{footnote}{1}
\footnotetext{Authors contributed equally to this work.}
\end{abstract}


\section{Introduction}
Training a deep neural network (DNN) is a challenging optimization problem: it involves minimizing the average of many high-dimensional non-convex loss functions. In practice, the main algorithms utilised are Stochastic Gradient Descent (SGD) \citep{robbins1951} and adaptive gradient methods such as AdaGrad \citep{duchi2011} or Adam \citep{kingma2014}. It has been observed that SGD tends to provide better generalization performance than adaptive gradient methods \citep{Wilson2017}. 
However, the downside of SGD is that it requires the manual design of a learning-rate schedule. Many forms of schedule have been proposed in the literature, including piece wise constant \citep{Huang2017a}, geometrically decreasing \citep{Szegedy2015} and warm starts with cosine annealing \citep{loshchilov2016}. Examples of these schemes are plotted in Figure \ref{fig:sgd_learning_rate_schedules}. Consequently, practitioners who wish to use SGD in a novel setting need to select a learning-rate schedule for their learning task. To that end, they first need to choose the parameterization of the schedule (e.g. picking one of the examples given above). Then they need to tune the corresponding parameters to get good empirical performance.  This typically results in a cross-validation that searches over many critical and sensitive hyper-parameters. For example, a piece wise linear scheme requires a starting learning rate value, a decay factor and a list or metric to determine at which points in training to decay the learning rate. Due to the high dimensionality of this search space performing a grid search can mean training a large number of models. This number increases exponentially in combination with other hyperparameters such as regularisation and batch size. Thus, finding an SGD learning rate schedule that produces strong generalisation performance for a new task is time and computationally demanding, often requiring hundreds of GPU hours.
\begin{figure}[H]
  \centering
  \includegraphics[width=\textwidth]{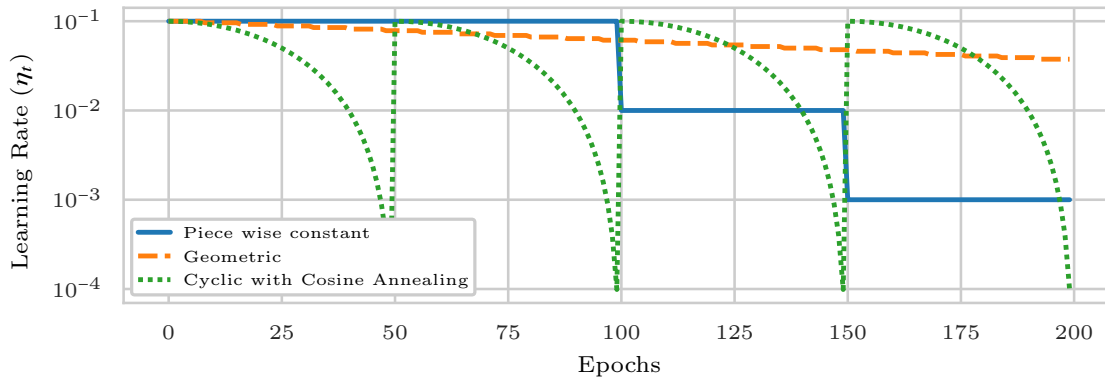}
  \caption{SGD Learning rate schedules proposed in the literature}
  \label{fig:sgd_learning_rate_schedules}
\end{figure}
\noindent
In this work, we alleviate this issue by presenting a family of algorithms that achieve comparable generalisation performance to SGD with a highly refined learning rate schedule, while requiring far less tuning of hyperparameters. This in turn leads to a reduction in the time, cost and energy required when finding a highly accurate network for a new task.\\
In more detail, we present a novel family of proximal algorithms for the optimisation of DNNs that are capable of interpolation. A DNN is said to interpolate a data set if it has the ability to simultaneously drive the loss of all the training samples in a data set to their minimum value. Thus a lower bound of the objective function is known: for instance, it is close to zero for a model trained with the cross-entropy loss. Our algorithms approximate the loss within each proximal problem by a bundle, that is, a point-wise maximum over linear functions. By including the interpolation lower bound within this bundle, we obtain the following two modelling benefits: (i) our model more closely mimics the true loss function than existing baselines like SGD, and; (ii) the learning rate gets automatically re-scaled at each iteration of the algorithm, thereby providing accurate updates. By increasing the number of linear approximations in the bundle, the true loss is better modelled. As an upshot, we obtain additional stability to optimisation and task-specific hyperparameters. Based on its highly desirable empirical properties, we term these methods Bundle optimization for Robust and Accurate Training (BORAT). \\
All the variants of BORAT use a single learning rate hyperparameter that requires minimal tuning. In particular, the learning rate hyperparameter is kept constant throughout the training procedure, unlike the learning rate of SGD that needs to be decayed for good generalization. The BORAT family of algorithms obtain state-of-the-art empirical performance for single hyperparameter training of neural networks.

\paragraph{Contributions}

\begin{itemize}
\item We design a family of adaptive algorithms that have a single hyperparameter that does not need any decaying schedule. In contrast, the related APROX \citep{asi2019} and L4 \citep{Rolinek2019} use respectively two and four hyperparameters for their learning-rate.

\item For the deep learning setting we establish a link between stochastic optimisation methods with adaptive learning rates and proximal bundle methods.

\item We provide convergence rates in various stochastic convex settings and for a class of non-convex problems. 

\item  We derive a novel algorithm for solving small quadratic programs where the constraints define the probability simplex. This algorithm permits a parallel implementation allowing for efficient solution on modern hardware. 

\item  We empirically demonstrate the increased stability to hyperparameters when increasing the bundle size. We show this on the CIFAR100 and Tiny ImageNet data sets. 

\item We achieve state-of-the-art results for learning a differentiable neural computer; training variants of residual networks on the SVHN and CIFAR data sets, and training a Bi-LSTM on the Stanford Natural Language Inference data set.

\end{itemize}
A preliminary version of this work appeared in the proceedings of ICML 2020 \citep{Berrada2019a}. While this previous work has considered bundles of size 2 resulting in the ALI-G algorithm  detailed with in Section \ref{sec:alig}. This article significantly differs from the previous work by (i) considering bundles of size greater than 2; (ii) introducing an novel algorithm to compute the exact solution of each bundle; (iii) investigating the robustness  towards hyperparameters; and (iv) showing how the use of large bundles permits easy application of BORAT to challenging non-smooth losses.

\section{Preliminaries}\label{sec:preliminaries}

\paragraph{Loss Function.}
We consider a supervised learning task where the model is parameterized by $\vw \in \mathbb{R}^d$.
Usually, the objective function can be expressed as an expectation over $z \in \Z$, a random variable indexing the samples of the training set:
\begin{equation}
    f(\vw) \triangleq \mathbb{E}_{z \in \Z}[\ell_z(\vw)],
\end{equation}
where each $\ell_z$ is the loss function associated with the sample $z$. We assume that each $\ell_z$ admits a known lower bound B. For simplicity we will often assume this lower bound is 0, which is the case for the large majority of loss functions used in machine learning. For instance, suppose that the model is a deep neural network with weights $\vw$ performing classification. Then for each sample $z$, $\ell_z(\vw)$ can represent the cross-entropy loss, which is always non-negative. Other non-negative loss functions include the structured or multi-class hinge loss, and the $\ell_1$ or $\ell_2$ loss functions for regression. Note that it is always possible to subtract a non-zero lower bound B from the loss function to define a new equivalent problem that satisfies the aforementioned assumption.

\paragraph{Interpolation.}
In this work we consider DNNs that can interpolate the data. Formally, we assume: 
\begin{equation}
   \exists \hspace{0.1cm }\wstar : \hspace{0.1cm } \forall {z \in \Z}, \ell_z(\wstar) \leq \epsilon, 
\end{equation}
where $\epsilon$ is a tolerance on the amount of error in the interpolation assumption. We will often want to make reference to the case when $\epsilon=0$. Following previous work \citep{ma2018} we will refer to this setting as perfect interpolation. The interpolation property is satisfied in many practical cases, since modern neural networks are typically trained in an over-parameterized regime where the parameters of the model far exceed the size of the training data \citep{Mingchen2020}. Additionally most modern DNN architectures can be easily increased in size and depth, allowing them to interpolate all but the largest data sets. Note, the data has to be labelled consistently for this to be possible. For instance it is impossible to interpolate a data set with two instances of the same image that have two different labels.

\paragraph{Regularisation.}
It is often desirable to encourage generalisation by the addition of a regularisation function $\phi(\vw)$ to the objective. Typical choices for $\phi$ include $\lambda \|\vw\|_2$ and $\lambda \|\vw\|_1$ where $\lambda$ governs the strength of the regularisation. However, in this work we incorporate such regularisation as a constraint on the feasible domain: $\Omega = \left\{ \vw \in \mathbb{R}^d: \ \phi(\vw) \leq r \right\}$ for some value of $r$. In the deep learning setting, this will allow us to assume that the objective function can be driven close to zero without unrealistic assumptions about the value of the regularisation term for the final set of parameters. Our framework can handle any constraint set $\Omega$ on which Euclidean projections are computationally efficient. This includes the feasible set induced by $\ell_2$ regularization: $\Omega = \left\{ \vw \in \mathbb{R}^d: \  \| \vw \|_2^2 \leq r \right\}$, for which the projection is given by a simple rescaling of $\vw$. Finally, note that if we do not wish to use any regularization, we define $\Omega = \mathbb{R}^d$ and the corresponding projection is the identity.

\paragraph{Problem Formulation.}
The learning task can be expressed as the problem $(\mathcal{P})$ of finding a feasible vector of parameters $\wstar \in \Omega$ that minimizes $f$:
\begin{equation} \tag{$\mathcal{P}$} \label{eq:main_problem}
    \wstar \in \argmin\limits_{\vw \in \Omega} f(\vw).
\end{equation}
We refer to the minimum value of $f$ over $\Omega$ as $\fstar$: $\fstar \triangleq \min_{\vw \in \Omega} f(\vw)$.

\paragraph{Proximal Perspective of Projected Stochastic Gradient Descent.}
In order to best introduce BORAT, we first detail the proximal interpretation of projected stochastic gradient descent (PSGD). The PSGD algorithm can be seen as solving a sequence of proximal problems. Within each proximal problem, a minimisation is performed over an approximate local model of the loss. This approximation is the first order Taylor's expansion of $\ell_{z_t}$ around the current iterate and a proximal term. At time step $t$, the PSGD proximal problem has the form: 
\begin{equation}\label{eq:proxsgd}
\vw_{t+1} = \argmin_{\vw \in \Omega}\left\{\frac{1}{2\eta_t}\|\vw - \vw_{t}\|^2 + \ell_{z_t}(\vw_t) +
\nabla \ell_{z_t}(\vw_t)^\top(\vw - \vw_{t}) \right\}.
\end{equation}
Here $\vw_t$ is the current iterate, $z_t$ is the index of the sample chosen and $\eta_t$ is the learning rate. For convex $\Omega$, problem (\ref{eq:proxsgd}) can be solved in two steps: first solving the unconstrained problem and then using Euclidean projection onto $\Omega$. Setting $\Omega = \mathbb{R}^d$ removes the need for projection and we recover SGD. When clear from the context we will use SGD to refer to both projected and un-projected variants. To solve the unconstrained problem one only needs to set the gradient to zero to recover the familiar closed form SGD update. 

\section{The BORAT Algorithm}
In this section we detail the BORAT algorithm. We start by introducing BORAT's proximal problem that is exactly solved at each iteration. We explain its advantages and disadvantages in relation to the SGD proximal problem (\ref{eq:proxsgd}). Each BORAT proximal problem is best solved in the dual. Hence we next introduce the dual problem, which permits a far more efficient solution due to its low dimensionality.  We then consider a special case of BORAT with minimal bundle size which we call Adaptive Learning-rates for Interpolation with Gradients (ALI-G). ALI-G permits a closed form solution and results in an automatically scaled gradient descent step. Specifically, it recovers a stochastic variant of the Polyak step size \citep{Polyak1969}, which offers competitive results in practice. This special case is used extensively in our experiments and in our analysis to establish the convergence rate of BORAT. Lastly, we detail our novel direct method for efficiently solving the general dual problem. This algorithm exploits the small size of the bundle to compute the exact optimum by solving a finite number of linear systems, removing the need for an inner iterative optimisation algorithm.

\subsection{Primal Problem}\label{Primal_Problem}
For tackling problems of type (\ref{eq:main_problem}) we identify two deficiencies in the proximal view of SGD (\ref{eq:proxsgd}). First, the approximation of the loss permits negative values even though the loss for (\ref{eq:main_problem}) is defined to be non-negative. Second, the accuracy of a linear model quickly deteriorates for functions with high curvature away from the site of the approximation. Due to this crude model, the selection of $\eta_t$ for all $t$ is critical for achieving good performance with SGD. BORAT aims to address these deficiencies by using a model composed of a point-wise maximum over $N$ linear approximations to better model the loss. One of the linear approximations is chosen to be a constant function equal to the loss lower bound, that is, 0. By including this linear approximation, we address the first deficiency. The second deficiency is addressed by making use of the remaining $N-1$ linear approximations. These extra approximations allow us to model variation over the parameter space and positive curvature of the loss. A model of this form in combination with a proximal term is commonly known as a bundle of size $N$. The main disadvantage of bundles is that they require multiple gradient evaluations to be performed and then held in memory. Hence we in this work only consider $N\leq5$, except where mentioned otherwise. Each linear approximation of the loss is constructed at a point $\hat{\w}_{t}^n$ using a different loss function $\ell_{z_t^n}$. The subscript $t$ indicates the iteration number, and $n$ indexes over the $N$ linear approximations. With this notation we first introduce a bundle of size 1 as:
\begin{align}\label{single_linear_approx}
\w_{t+1} = \argmin_{\vw \in \Omega}\left\{\frac{1}{2\eta}\|\vw - \vw_{t}\|^2 + \ell_{z_{t}^1}(\hat{\w}_{t}^1) + \nabla \ell_{z_{t}^1}(\hat{\w}_{t}^1)^\top(\w - \hat{\w}_{t}^1)  \right\}.
\end{align}
If we set $\hat{\w}_{t}^1$ to $\w_t$ we recover the SGD proximal problem. Thus SGD effectively uses a bundle of size $N=1$. We next introduce an expanded expression for a bundle of size $N$, before showing how to convert this into the compact form of  $\max_{n\in[N]}\left\{\bm{a}^{n\top} (\w-\w_t)+b^n\right\}$. Within a bundle each linear approximation is formed around a different point $\hat{\w}_{t}^n$. Hence in order to write each linear approximation in the aforementioned compact form we split each linear term in two. The first piece is a constant term, that does not depend on $\w$, and is a multiple of the distance between the current iterate and the centre of each approximation $\hat{\w}_{t}^n - \w_t$. The second term depends on the distance $\w-\w_t$, for all linear approximations. This gives the following expanded form for a bundle of size $N$ as:
\begin{align}\label{max_of_linear_approx}
\max_{n\in[N]}\left\{\ell_{z_{t}^n}(\hat{\w}_{t}^n) - \nabla \ell_{z_{t}^n}(\hat{\w}_{t}^n)^\top (\hat{\w}_{t}^n- \w_t)+ \nabla \ell_{z_{t}^n}(\hat{\w}_{t}^n)^\top(\w - \w_{t}) \right\},
\end{align}
where we use the notation $[N]$ to define the set of integers $\{1,...,N\}$. If we define $b_{t}^n = \ell_{z_{t}^n}(\hat{\w}_{t}^n) - \nabla \ell_{z_{t}^n}(\hat{\w}_{t}^n)^\top (\hat{\w}_{t}^n- \w_t)$, we can thus simplify the expression into the desired compact form. We now introduce the BORAT proximal problem at time $t$ with a bundle of size $N$, which can be stated as:
\begin{align}\label{primal}
\w_{t+1} = \argmin_{\w \in \Omega}\left\{ \frac{1}{2\eta}\|\w - \w_{t}\|^2 + \max_{n\in[N]}
\left\{\nabla \ell_{z_{t}^n}(\hat{\w}_{t}^n)^\top(\w - \w_{t}) + b_{t}^n\right\} \right\}.
\end{align}
For BORAT we always set $\hat{\w}_{t}^1 = \w_t$. We additionally use the last linear approximation to enforce the lower bound on the loss. This is done by setting $\nabla \ell_{z_t}(\hat{\w}_{t}^N) = [0,...,0]^\top$, $b_{t}^N = \text{B} = 0$. We give details on how we select $\hat{\w}_{t}^n$ for $n\in\{2,...,N-1\}$ in Section \ref{Bundle_Construction}. Thus each bundle is composed of $N-1$ linear approximations of the function, and the lower bound on the loss. These linear approximations of the loss need to be stored in memory during each step. Hence, in order to fit on a single GPU we only consider small bundle sizes in this work $(N \leq 5)$. For clarity we depict a 1D example for a bundle with $N=3$ in Figure \ref{fig:nr_borat}.\\
Unlike SGD, the BORAT proximal problem (\ref{primal}) is not smooth and hence cannot be solved by simply setting the derivatives to zero. Instead we choose to solve each proximal problem in the dual.

\begin{figure}[H]
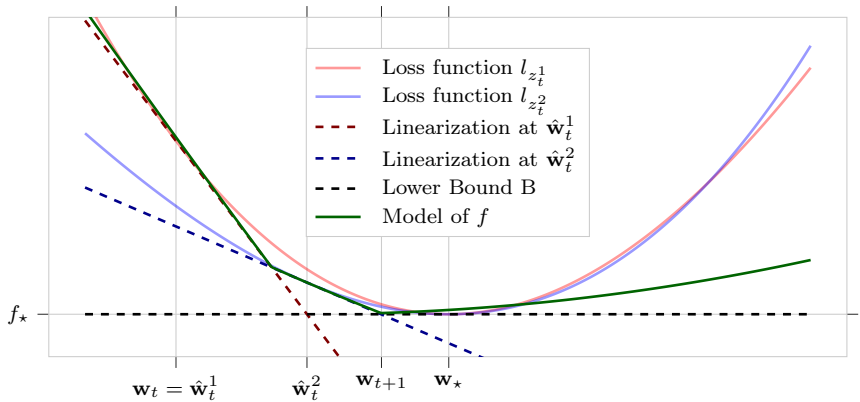

\centering
\scriptsize

\caption{\em Illustration of a BORAT bundle ($N=3$) in 1D, shown in green. Two stochastic samples $\ell_{z_t^1}$ and $\ell_{z_t^2}$ of the loss function $f$ are shown in red and blue (solid lines). The bundle is formed of the max of three linear approximations (dashed lines) and a proximal term. Two of these linear approximations are formed using the loss functions $\ell_{z_t^1}$ and $\ell_{z_t^2}$, and the last enforces the known lower bound on the loss. Here the BORAT approximation gives a more accurate model than approximation used by ALI-G, which would only include the linearization at $\hat{w}_t^1$ and the lower bound B. In this simple example this improved accuracy allows for a larger step to be taken towards the minimum.}
\label{fig:nr_borat}
\end{figure}

\subsection{Dual Problem}
The dual of (\ref{primal}) is a constrained concave quadratic maximisation over $N$ dual variables $\alpha^1, ..., \alpha^N$, and can be concisely written as follows (see supplementary material for derivation):
\begin{align}\label{dual}
\bm{\alpha}_t = \argmax_{\bm{\alpha} \in \Delta_N}D(\bm{\alpha}), \hspace{0.1cm}\text{where}\hspace{0.5cm} D(\bm{\alpha}) =  -\frac{\eta}{2}\bm{\alpha}^\top A_{t}^\top A_{t}\bm{\alpha} + \bm{\alpha}^\top \bm{b}_{t}.
\end{align}
Here $A_{t}$ is a $N \times d$ matrix whose $n^{th}$ row is $\nabla \ell_{z_t}(\hat{\w}_{t}^n)$. We define $\bm{b}_{t} = [b_{t}^1,...,b_{t}^N]^\top$, $\bm{\alpha} = [\alpha^{1}, \alpha^{2}, \dots, \alpha^{N}]^\top$ and $\Delta_N$ is a probability simplex over the $N$ variables. The dual problem (\ref{dual}) has a number of features that make it more appealing for optimisation than the primal (\ref{primal}). First, the primal problem is defined over the parameters space $\w\in\mathbb{R}^d$, where $p$ is in the order thousands if not millions for modern DNNs. In contrast, the dual variables are of dimension N, where is $N$ typically a small number $>10$ due to the memory requirements. Second, the dual is smooth and hence allows for faster convergence with standard optimisation techniques. Furthermore, as will be seen shortly, we use the fact that the dual feasible region is a tractable probability simplex to design a customised algorithm for its solution. We detail this algorithm in this Section \ref{dual_solution}. Once we have found the dual solution $\alpha_t$, we recover the following update from the KKT conditions:
\begin{align}\label{update}
 \w_{t+1} = \w_t - \eta A_{t}\bm{\alpha}_t. 
\end{align}
The form of the update (\ref{update}) deserves some attention. Since each row of $A_t$ is either the gradient of the loss $\ell_{z_t^n}$ or a zero vector, and $\alpha_t$ belongs to the probability simplex, the update step moves in the direction of a non-negative linear combination of negative gradients $-\nabla \ell_{z_t^n}(\hat{\w}_{t}^n)$. Due to the definitions of $\nabla \ell_{z_t}(\hat{\w}_{t}^N) = [0,...,0]^\top$ and $b_{t}^N = 0$, any weight given to $\alpha^N$ reduces the magnitude of the resulting step. This has the effect that as the loss value gets close to zero BORAT automatically decreases the size of the step taken.

\subsection{ALI-G (N=2)}\label{sec:alig}
We now consider a special case of BORAT with $N=2$. Here the bundle is the point wise maximum over the linear approximation of the loss around the current point and the global lower bound B, which we assume is 0. This special case is worthy of extra attention for the following four reasons. First, this special case only requires one gradient evaluation per step and has a similar time complexity to SGD. Second, it admits a closed form solution. Third, we use this special case extensively in our analysis of the convergence rate of BORAT. Fourth, given the definitions $\hat{\w}_{t}^1=\w_t$ and $\hat{\w}_{t}^N=0$, if we set $N=2$ we recover an algorithm that simply scales the SGD learning rate. Specifically, it automatically scales down a maximal learning rate $\eta$ by a factor $\alpha^1\in[0,1]$ to an appropriate value close to optimality. This is clear from the simplified version of Equation (\ref{update}), which has the following form:
\begin{align}\label{alig_update}
 \w_{t+1} = \w_t -  \alpha^1 \eta \nabla l_{z_{t}}(\vw_t).
\end{align}
Hence we will call this special case Adaptive Learning-rates for Interpolation with Gradients (ALI-G). For ALI-G the primal problem (\ref{primal}) simplifies to the following: 
\begin{align}\label{alig_prox}
\argmin_{\w \in \Omega}\left\{\frac{1}{2\eta}\|\w - \w_{t}\|^2 + \max\{\text{0},\ell_{z_t}(\w_t) +
\nabla \ell_{z_t}(\w_t)^\top(\w - \w_{t})\}\right\}.
\end{align}
Likewise, the dual problem (\ref{dual}) can be reduced to the following:
\begin{align}\label{alig_dual}
\alpha^1 = \argmax_{\alpha^1 \in [0,1]}\left\{ -\frac{\eta}{2}\|\alpha^1\nabla l_{z_{t}}(\vw_t)\|^2 + \alpha^1l_{z_{t}}(\vw_t) \right\}.
\end{align}
The ALI-G dual is a maximisation over a constrained concave function in one dimension. Hence we can obtain the optimal point by projecting the unconstrained solution on to the feasible region. This results in the following closed form solution:
\begin{align}\label{alig_alpha}
\alpha^1 = \min\left\{\frac{l_{z_{t}}(\vw_t)}{\eta\|\nabla l_{z_{t}}(\vw_t)\|^2},1\right\}.
\end{align}
This value of $\alpha^1$ is then used in (\ref{alig_update}). The ALI-G update can be viewed as a stochastic analog of the Polyak step size \citep{Polyak1969} with the addition of a maximal value $\eta$. Recall that, from the interpolation assumption, we have $\fstar=0$. The ALI-G update is computationally cheap with the evaluation of $\|\nabla l_{z_{t}}(\vw_t)\|^2$ being the only extra computation required over SGD. Hence when the interpolation assumption holds we suggest that ALI-G can be easily used in place of SGD.

\subsection{The Advantage of Bundles with more than Two Pieces}\label{General Case}

While ALI-G has many favourable qualities, its local model of the loss is still crude. We next give two motivating examples to help demonstrate why using a more complex model of the loss often increases the robustness to $\eta$. Thus it may prove useful to use larger values of $N$ in  settings that are sensitive to step size $\eta$.\\
In the convex setting any function can be perfectly modelled by the point-wise maximum over an infinite number of linear approximations. While intractable, performing a minimisation over this model would recover the true optimum by definition. With this perfect model any value of $\eta$ could be used. Setting $\eta$ to a large enough value would recover the optimum in a single step. This example demonstrates that, at least asymptotically as the accuracy of the local model increases we can expect a reduced dependence on the correct scale of the step size.\\

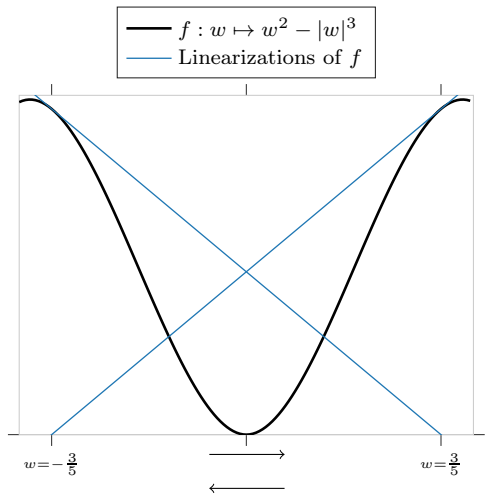
\begin{figure}[H]
\centering
\scriptsize
\begin{tikzpicture}

\definecolor{color0}{rgb}{0.12156862745098,0.466666666666667,0.705882352941177}

\begin{axis}[
axis line style={white!80!black},
compat=newest,
height=0.4\textwidth,
legend cell align={left},
legend style={at={(0.5,1.05)}, anchor=south, draw=lightgray!20.0!black},
tick align=outside,
tick pos=both,
width=0.5\textwidth,
x grid style={white!80!black},
x grid style={white},
xmajorgrids,
xmajorticks=true,
xmin=-0.7, xmax=0.7,
xtick style={color=white!15!black},
xtick={-0.6,0,0.6},
xticklabels={\(\displaystyle \scriptscriptstyle w=-\frac{3}{5}\),\Arrow{1.0cm},\(\displaystyle \scriptscriptstyle w=\frac{3}{5}\)},
y grid style={white!80!black},
ymajorgrids,
ymajorticks=true,
ymin=0, ymax=0.15,
yminorticks=false,
ytick style={color=white!15!black},
ytick={0},
yticklabels={\phantom{a}}
]
\addplot [line width=1pt, black]
table {%
-0.699999988079071 0.14699998497963
-0.689999997615814 0.1475909948349
-0.680000007152557 0.147967994213104
-0.670000016689301 0.148137003183365
-0.660000026226044 0.148104012012482
-0.649999976158142 0.147874981164932
-0.639999985694885 0.147455990314484
-0.629999995231628 0.146852999925613
-0.620000004768372 0.146072000265121
-0.610000014305115 0.14511901140213
-0.600000023841858 0.143999993801117
-0.589999973773956 0.142720997333527
-0.579999983310699 0.141287997364998
-0.569999992847443 0.139706999063492
-0.560000002384186 0.13798400759697
-0.550000011920929 0.13612499833107
-0.540000021457672 0.134136006236076
-0.529999971389771 0.132022991776466
-0.519999980926514 0.129792004823685
-0.509999990463257 0.127448990941048
-0.5 0.125
-0.490000009536743 0.1224509999156
-0.479999989271164 0.119808003306389
-0.469999998807907 0.117077000439167
-0.46000000834465 0.114264003932476
-0.449999988079071 0.111374996602535
-0.439999997615814 0.108415998518467
-0.430000007152557 0.105392999947071
-0.419999986886978 0.102311998605728
-0.409999996423721 0.0991789996623993
-0.400000005960464 0.0960000082850456
-0.389999985694885 0.0927809923887253
-0.379999995231628 0.089528001844883
-0.370000004768372 0.0862470045685768
-0.360000014305115 0.0829440057277679
-0.349999994039536 0.0796249955892563
-0.340000003576279 0.0762960016727448
-0.330000013113022 0.0729630067944527
-0.319999992847443 0.0696319937705994
-0.310000002384186 0.066309005022049
-0.300000011920929 0.063000001013279
-0.28999999165535 0.0597109943628311
-0.280000001192093 0.0564480014145374
-0.270000010728836 0.0532170012593269
-0.259999990463257 0.0500239990651608
-0.25 0.046875
-0.239999994635582 0.0437759980559349
-0.230000004172325 0.0407330021262169
-0.219999998807907 0.0377519987523556
-0.209999993443489 0.0348389968276024
-0.200000002980232 0.0320000015199184
-0.189999997615814 0.0292409993708134
-0.180000007152557 0.0265680011361837
-0.170000001788139 0.0239870008081198
-0.159999996423721 0.0215039998292923
-0.150000005960464 0.0191249996423721
-0.140000000596046 0.0168559998273849
-0.129999995231628 0.0147029999643564
-0.119999997317791 0.0126719996333122
-0.109999999403954 0.0107690002769232
-0.100000001490116 0.00900000054389238
-0.0900000035762787 0.00737100001424551
-0.0799999982118607 0.00588799966499209
-0.0700000002980232 0.00455700000748038
-0.0599999986588955 0.00338399992324412
-0.0500000007450581 0.00237500015646219
-0.0399999991059303 0.00153599993791431
-0.0299999993294477 0.000873000011779368
-0.0199999995529652 0.000391999987186864
-0.00999999977648258 9.89999971352518e-05
1.11022302462516e-16 1.23259516440783e-32
0.00999999977648258 9.89999971352518e-05
0.0199999995529652 0.000391999987186864
0.0299999993294477 0.000873000011779368
0.0399999991059303 0.00153599993791431
0.0500000007450581 0.00237500015646219
0.0599999986588955 0.00338399992324412
0.0700000002980232 0.00455700000748038
0.0799999982118607 0.00588799966499209
0.0900000035762787 0.00737100001424551
0.100000001490116 0.00900000054389238
0.109999999403954 0.0107690002769232
0.119999997317791 0.0126719996333122
0.129999995231628 0.0147029999643564
0.140000000596046 0.0168559998273849
0.150000005960464 0.0191249996423721
0.159999996423721 0.0215039998292923
0.170000001788139 0.0239870008081198
0.180000007152557 0.0265680011361837
0.189999997615814 0.0292409993708134
0.200000002980232 0.0320000015199184
0.209999993443489 0.0348389968276024
0.219999998807907 0.0377519987523556
0.230000004172325 0.0407330021262169
0.239999994635582 0.0437759980559349
0.25 0.046875
0.259999990463257 0.0500239990651608
0.270000010728836 0.0532170012593269
0.280000001192093 0.0564480014145374
0.28999999165535 0.0597109943628311
0.300000011920929 0.063000001013279
0.310000002384186 0.066309005022049
0.319999992847443 0.0696319937705994
0.330000013113022 0.0729630067944527
0.340000003576279 0.0762960016727448
0.349999994039536 0.0796249955892563
0.360000014305115 0.0829440057277679
0.370000004768372 0.0862470045685768
0.379999995231628 0.089528001844883
0.389999985694885 0.0927809923887253
0.400000005960464 0.0960000082850456
0.409999996423721 0.0991789996623993
0.419999986886978 0.102311998605728
0.430000007152557 0.105392999947071
0.439999997615814 0.108415998518467
0.449999988079071 0.111374996602535
0.46000000834465 0.114264003932476
0.469999998807907 0.117077000439167
0.479999989271164 0.119808003306389
0.490000009536743 0.1224509999156
0.5 0.125
0.509999990463257 0.127448990941048
0.519999980926514 0.129792004823685
0.529999971389771 0.132022991776466
0.540000021457672 0.134136006236076
0.550000011920929 0.13612499833107
0.560000002384186 0.13798400759697
0.569999992847443 0.139706999063492
0.579999983310699 0.141287997364998
0.589999973773956 0.142720997333527
0.600000023841858 0.143999993801117
0.610000014305115 0.14511901140213
0.620000004768372 0.146072000265121
0.629999995231628 0.146852999925613
0.639999985694885 0.147455990314484
0.649999976158142 0.147874981164932
0.660000026226044 0.148104012012482
0.670000016689301 0.148137003183365
0.680000007152557 0.147967994213104
0.689999997615814 0.1475909948349
};
\addlegendentry{$f: w \mapsto w^2 - |w|^3$}
\addplot [line width=0.5pt, color0]
table {%
-0.699999988079071 0.15599998831749
-0.689999997615814 0.154799997806549
-0.680000007152557 0.153599992394447
-0.670000016689301 0.152399986982346
-0.660000026226044 0.151199996471405
-0.649999976158142 0.149999991059303
-0.639999985694885 0.148799985647202
-0.629999995231628 0.147599995136261
-0.620000004768372 0.146399989724159
-0.610000014305115 0.145199999213219
-0.600000023841858 0.143999993801117
-0.589999973773956 0.142799988389015
-0.579999983310699 0.141599982976913
-0.569999992847443 0.140399992465973
-0.560000002384186 0.139199987053871
-0.550000011920929 0.137999996542931
-0.540000021457672 0.136799991130829
-0.529999971389771 0.135599985718727
-0.519999980926514 0.134399995207787
-0.509999990463257 0.133199989795685
-0.5 0.131999984383583
-0.490000009536743 0.130799993872643
-0.479999989271164 0.129599988460541
-0.469999998807907 0.1283999979496
-0.46000000834465 0.127199992537498
-0.449999988079071 0.125999987125397
-0.439999997615814 0.124799989163876
-0.430000007152557 0.123599991202354
-0.419999986886978 0.122399985790253
-0.409999996423721 0.121199987828732
-0.400000005960464 0.11999998986721
-0.389999985694885 0.118799984455109
-0.379999995231628 0.117599993944168
-0.370000004768372 0.116399988532066
-0.360000014305115 0.115199990570545
-0.349999994039536 0.113999992609024
-0.340000003576279 0.112799987196922
-0.330000013113022 0.111599996685982
-0.319999992847443 0.11039999127388
-0.310000002384186 0.109199985861778
-0.300000011920929 0.107999995350838
-0.28999999165535 0.106799989938736
-0.280000001192093 0.105599984526634
-0.270000010728836 0.104399994015694
-0.259999990463257 0.103199988603592
-0.25 0.101999990642071
-0.239999994635582 0.10079999268055
-0.230000004172325 0.0995999872684479
-0.219999998807907 0.0983999893069267
-0.209999993443489 0.0971999913454056
-0.200000002980232 0.0959999859333038
-0.189999997615814 0.0947999879717827
-0.180000007152557 0.0935999900102615
-0.170000001788139 0.0923999920487404
-0.159999996423721 0.0911999866366386
-0.150000005960464 0.0899999886751175
-0.140000000596046 0.0887999832630157
-0.129999995231628 0.0875999927520752
-0.119999997317791 0.0863999873399734
-0.109999999403954 0.0851999893784523
-0.100000001490116 0.0839999914169312
-0.0900000035762787 0.08279999345541
-0.0799999982118607 0.0815999880433083
-0.0700000002980232 0.0803999900817871
-0.0599999986588955 0.079199992120266
-0.0500000007450581 0.0779999867081642
-0.0399999991059303 0.0767999887466431
-0.0299999993294477 0.0755999833345413
-0.0199999995529652 0.0743999853730202
-0.00999999977648258 0.073199987411499
1.11022302462516e-16 0.0719999894499779
0.00999999977648258 0.0707999914884567
0.0199999995529652 0.0695999935269356
0.0299999993294477 0.0683999881148338
0.0399999991059303 0.0671999827027321
0.0500000007450581 0.0659999847412109
0.0599999986588955 0.0647999867796898
0.0700000002980232 0.0635999888181686
0.0799999982118607 0.0623999908566475
0.0900000035762787 0.0611999854445457
0.100000001490116 0.0599999874830246
0.109999999403954 0.0587999895215034
0.119999997317791 0.0575999841094017
0.129999995231628 0.0563999861478806
0.140000000596046 0.0551999881863594
0.150000005960464 0.0539999902248383
0.159999996423721 0.0527999922633171
0.170000001788139 0.0515999868512154
0.180000007152557 0.0503999888896942
0.189999997615814 0.0491999909281731
0.200000002980232 0.0479999855160713
0.209999993443489 0.0467999875545502
0.219999998807907 0.0455999821424484
0.230000004172325 0.0443999841809273
0.239999994635582 0.0431999862194061
0.25 0.041999988257885
0.259999990463257 0.0407999902963638
0.270000010728836 0.0395999923348427
0.280000001192093 0.0383999869227409
0.28999999165535 0.0371999889612198
0.300000011920929 0.035999983549118
0.310000002384186 0.0347999855875969
0.319999992847443 0.0335999876260757
0.330000013113022 0.032399982213974
0.340000003576279 0.0311999842524529
0.349999994039536 0.0299999862909317
0.360000014305115 0.0287999883294106
0.370000004768372 0.0275999829173088
0.379999995231628 0.0263999849557877
0.389999985694885 0.0251999869942665
0.400000005960464 0.0239999890327454
0.409999996423721 0.0227999910712242
0.419999986886978 0.0215999931097031
0.430000007152557 0.0203999951481819
0.439999997615814 0.0191999897360802
0.449999988079071 0.017999991774559
0.46000000834465 0.0167999863624573
0.469999998807907 0.0155999809503555
0.479999989271164 0.014399990439415
0.490000009536743 0.0131999850273132
0.5 0.0119999796152115
0.509999990463257 0.0107999891042709
0.519999980926514 0.00959998369216919
0.529999971389771 0.00839999318122864
0.540000021457672 0.0071999728679657
0.550000011920929 0.00599998235702515
0.560000002384186 0.0047999769449234
0.569999992847443 0.00359998643398285
0.579999983310699 0.0023999810218811
0.589999973773956 0.00119997560977936
0.600000023841858 -1.49011611938477e-08
0.610000014305115 -0.00120002031326294
0.620000004768372 -0.00240001082420349
0.629999995231628 -0.00360001623630524
0.639999985694885 -0.00480000674724579
0.649999976158142 -0.00600001215934753
0.660000026226044 -0.00720001757144928
0.670000016689301 -0.00840000808238983
0.680000007152557 -0.00960001349449158
0.689999997615814 -0.0108000040054321
};
\addlegendentry{Linearizations of $f$}
\addplot [line width=0.5pt, color0, forget plot]
table {%
-0.699999988079071 -0.0120000094175339
-0.689999997615814 -0.0108000040054321
-0.680000007152557 -0.00960001349449158
-0.670000016689301 -0.00840000808238983
-0.660000026226044 -0.00720001757144928
-0.649999976158142 -0.00600001215934753
-0.639999985694885 -0.00480000674724579
-0.629999995231628 -0.00360001623630524
-0.620000004768372 -0.00240001082420349
-0.610000014305115 -0.00120002031326294
-0.600000023841858 -1.49011611938477e-08
-0.589999973773956 0.00119997560977936
-0.579999983310699 0.0023999810218811
-0.569999992847443 0.00359998643398285
-0.560000002384186 0.0047999769449234
-0.550000011920929 0.00599998235702515
-0.540000021457672 0.0071999728679657
-0.529999971389771 0.00839999318122864
-0.519999980926514 0.00959998369216919
-0.509999990463257 0.0107999891042709
-0.5 0.0119999796152115
-0.490000009536743 0.0131999850273132
-0.479999989271164 0.014399990439415
-0.469999998807907 0.0155999809503555
-0.46000000834465 0.0167999863624573
-0.449999988079071 0.017999991774559
-0.439999997615814 0.0191999897360802
-0.430000007152557 0.0203999951481819
-0.419999986886978 0.0215999931097031
-0.409999996423721 0.0227999910712242
-0.400000005960464 0.0239999890327454
-0.389999985694885 0.0251999869942665
-0.379999995231628 0.0263999849557877
-0.370000004768372 0.0275999829173088
-0.360000014305115 0.0287999883294106
-0.349999994039536 0.0299999862909317
-0.340000003576279 0.0311999842524529
-0.330000013113022 0.032399982213974
-0.319999992847443 0.0335999876260757
-0.310000002384186 0.0347999855875969
-0.300000011920929 0.035999983549118
-0.28999999165535 0.0371999889612198
-0.280000001192093 0.0383999869227409
-0.270000010728836 0.0395999923348427
-0.259999990463257 0.0407999902963638
-0.25 0.041999988257885
-0.239999994635582 0.0431999862194061
-0.230000004172325 0.0443999841809273
-0.219999998807907 0.0455999821424484
-0.209999993443489 0.0467999875545502
-0.200000002980232 0.0479999855160713
-0.189999997615814 0.0491999909281731
-0.180000007152557 0.0503999888896942
-0.170000001788139 0.0515999868512154
-0.159999996423721 0.0527999922633171
-0.150000005960464 0.0539999902248383
-0.140000000596046 0.0551999881863594
-0.129999995231628 0.0563999861478806
-0.119999997317791 0.0575999841094017
-0.109999999403954 0.0587999895215034
-0.100000001490116 0.0599999874830246
-0.0900000035762787 0.0611999854445457
-0.0799999982118607 0.0623999908566475
-0.0700000002980232 0.0635999888181686
-0.0599999986588955 0.0647999867796898
-0.0500000007450581 0.0659999847412109
-0.0399999991059303 0.0671999827027321
-0.0299999993294477 0.0683999881148338
-0.0199999995529652 0.0695999935269356
-0.00999999977648258 0.0707999914884567
1.11022302462516e-16 0.0719999894499779
0.00999999977648258 0.073199987411499
0.0199999995529652 0.0743999853730202
0.0299999993294477 0.0755999833345413
0.0399999991059303 0.0767999887466431
0.0500000007450581 0.0779999867081642
0.0599999986588955 0.079199992120266
0.0700000002980232 0.0803999900817871
0.0799999982118607 0.0815999880433083
0.0900000035762787 0.08279999345541
0.100000001490116 0.0839999914169312
0.109999999403954 0.0851999893784523
0.119999997317791 0.0863999873399734
0.129999995231628 0.0875999927520752
0.140000000596046 0.0887999832630157
0.150000005960464 0.0899999886751175
0.159999996423721 0.0911999866366386
0.170000001788139 0.0923999920487404
0.180000007152557 0.0935999900102615
0.189999997615814 0.0947999879717827
0.200000002980232 0.0959999859333038
0.209999993443489 0.0971999913454056
0.219999998807907 0.0983999893069267
0.230000004172325 0.0995999872684479
0.239999994635582 0.10079999268055
0.25 0.101999990642071
0.259999990463257 0.103199988603592
0.270000010728836 0.104399994015694
0.280000001192093 0.105599984526634
0.28999999165535 0.106799989938736
0.300000011920929 0.107999995350838
0.310000002384186 0.109199985861778
0.319999992847443 0.11039999127388
0.330000013113022 0.111599996685982
0.340000003576279 0.112799987196922
0.349999994039536 0.113999992609024
0.360000014305115 0.115199990570545
0.370000004768372 0.116399988532066
0.379999995231628 0.117599993944168
0.389999985694885 0.118799984455109
0.400000005960464 0.11999998986721
0.409999996423721 0.121199987828732
0.419999986886978 0.122399985790253
0.430000007152557 0.123599991202354
0.439999997615814 0.124799989163876
0.449999988079071 0.125999987125397
0.46000000834465 0.127199992537498
0.469999998807907 0.1283999979496
0.479999989271164 0.129599988460541
0.490000009536743 0.130799993872643
0.5 0.131999984383583
0.509999990463257 0.133199989795685
0.519999980926514 0.134399995207787
0.529999971389771 0.135599985718727
0.540000021457672 0.136799991130829
0.550000011920929 0.137999996542931
0.560000002384186 0.139199987053871
0.569999992847443 0.140399992465973
0.579999983310699 0.141599982976913
0.589999973773956 0.142799988389015
0.600000023841858 0.143999993801117
0.610000014305115 0.145199999213219
0.620000004768372 0.146399989724159
0.629999995231628 0.147599995136261
0.639999985694885 0.148799985647202
0.649999976158142 0.149999991059303
0.660000026226044 0.151199996471405
0.670000016689301 0.152399986982346
0.680000007152557 0.153599992394447
0.689999997615814 0.154799997806549
};
\end{axis}

\end{tikzpicture}

\caption{\em A simple example where the ALI-G step-size oscillates due to non-convexity. For this problem, ALI-G only converges ifs its maximal learning-rate $\eta$ is less than $10$. By contrast for the same example BORAT with $N>2$ converges for all values of $\eta$. Additionally for $\eta\geq10$ it converges to the optimum in a single update.}
\label{fig:rsi}
\end{figure}
\noindent
Figure \ref{fig:rsi} provides a non-convex motivating example for use of larger values of $N$. Here we demonstrate a 1D symmetric function where ALI-G does not converge for large $\eta$. Instead it oscillates between the two values $w=-3/5$ and $w=3/5$. However, if we were to use a bundle with a $N\geq3$ our model of the loss would include both the linear approximations at $w= -3/5$ and $w= 3/5$ simultaneously and hence when minimising over this model we converge to the optimum for any value of $\eta$. While this is a carefully constructed synthetic example it highlights why we would expect a more accurate model of the loss can help to reduce the dependence on the step size.

\subsection{Selecting Additional Linear Approximations or the Bundle}\label{Bundle_Construction}
When constructing bundles of size larger than two, we are faced with two design decisions regarding how to select additional linear approximations to add to the bundle. First, where in parameter space should we construct the additional linear approximations $\hat{\w}_{t}^n$? And second, should we use the same mini-batch of data when constructing the stochastic linear approximations, or should we sample a new batch to evaluate each linear approximation? 

\paragraph{Selecting $\hat{\w}_{t}^n$.}

Ideally we would select the location of the linear approximations $\hat{\w}_{t}^n$ for $ n\in \{2,...,N-1\}$ in order to maximise the progress made towards $\wstar$ at each step. However, without the knowledge of $\wstar$ a priori, due to the high dimensional and non-convex nature of problem (\ref{eq:main_problem}) this is infeasible. Instead we make use of a a heuristic. Inspired by the work of previous bundle methods for convex problems \citep{Smola2007,asi2019} we select $\hat{\w}_{t}^n$ using the following method:
\begin{align}\label{eq:bundle}
\hat{\w}_{t}^{n} = \argmin_{\w \in \mathbb{R}^d}\left\{ \frac{1}{2\eta}\|\w - \w_{t}\|^2 + \max_{k\in[n-1]}
\left\{\nabla \ell_{z_{t}^k}(\hat{\w}_{t}^k)^\top(\w - \w_{t}) + b_{t}^k\right\} \right\}.
\end{align}
In other words we construct the bundle by recursively adding linear approximations centred at the current optimum. This method of selecting additional linear approximations is appealing as it requires no extra hyperparameters and helps refine the approximation in the region of parameter space that would be explored by the next update.

\paragraph{Re-sampling $z$ for additional linear approximations.}
When constructing additional linear approximations we choose to re-sample $z$. Concretely, we use a new mini-batch of data to construct each stochastic linear approximation. While it is possible to construct all $N-1$ non-zero linear approximations using the same batch of data we find this does not work well in practice. Indeed, such a method behaves similarly to taking multiple consecutive steps of SGD on the same mini-batch, which tends to produce poor optimisation performance.

\paragraph{Summary.}
We now summarise the bundle construction procedure for $N>2$. We construct a bundle around $\w_t$ by first using two linear approximations, one centred at $\w_t$ and the second given by a known lower bound on the loss. We then sequentially add linear approximations until we have $N$. These extra linear approximations are constructed one at a time using new batches of data and centred around the point that is the current minimizer of the bundle. We note that each parameter update of BORAT uses $N-1$ batches of data. Therefore BORAT updates the parameters $N-1$ fewer times than SGD in an epoch (given the same batch-size). Once the bundle is fully constructed we update $\w_t$. At this stage we apply momentum (if enabled) and project back on the feasible set $\Omega$. \\
The construction of the bundle requires solving a minimization problem for each newly added piece when $N\geq2$. Therefore it is critical that such a problem gets solved very efficiently. We next detail how we do this by solving the  corresponding dual problem for $N\geq2$. 

\subsection{Efficient Dual Algorithm to Compute $N\geq2$ Linear Pieces}\label{dual_solution}
In the general case $(N>2)$, the dual has more than 2 degrees of freedom and can not be written as a 1D minimisation. Thus the method derived for the case $N=2$ is no longer applicable. This means it is not possible to obtain a simple closed form update. We must instead run an inner optimisation to solve (\ref{dual}) at each step. Many methods exist for maximising a concave quadratic objective over a simplex. Two algorithms particularly well suited to problems of the form (\ref{dual}) are \citep{Frank1956} or Homotopy Methods \citep{Bach2011}. However, we propose a novel algorithm that exploits the fact that $N$ is small to find the maximum directly. This method decomposes the problem of solving (\ref{dual}) into several sub-problems which provides two computational conveniences. First, the BORAT algorithm repeatedly searches for solutions to a bundle with only one newly added linear piece since the last search. As one might expect, this task shares a great number of sub-problems with the previous solution and allows for much of the computation to be reused. Second, our dual algorithm allows for a parallel implementation, which makes it very fast on the hardware commonly used for deep learning. To illustrate the efficiency of our dual method, the run time of the dual solution, that is finding $\bm{\alpha}_t$ once we have constructed $(\ref{dual})$, takes less than 5\% of the time spent in the call of the optimiser. Note due to the large size of the the networks and the small size of $N$ the majority of the call time is dominated computing $A_{t}^\top A_{t}$ and $A_{t}\bm{\alpha}_t$. \\
We now formally introduce our dual solution algorithm. Our method uses the observation that at the optimal solution in a simplex the partial derivatives will be equal for all nonzero dimensions. This observation can be formally stated as:

\begin{proposition2}[Simplex Optimally Conditions] \label{th:simplex}
Let $F: \mathbb{R}^N \to \mathbb{R}$ be a concave function. 
Let us define $\bm{\alpha}_* = \argmax_{\bm{\alpha}\in\Delta}F(\bm{\alpha})$. 
Then there exists $c \in \mathbb{R}$ such that:
\begin{equation}\label{observation}
\forall n \in [N] \text{ such that } \alpha^{n}_* > 0, \text{ we have: }
   \frac{\partial F(\bm{\alpha})}{\partial \alpha^n}\bigg\rvert_{\bm{\alpha} = \bm{\alpha}_*} =  c.
\end{equation}
In other words, the value of the partial derivative is shared among all coordinates of $\bm{\alpha}_*$ that are non-zero.
\end{proposition2}
This proposition can be easily proved by contradiction. If the partial derivatives are not equal, then moving in the direction of the largest would result in an increase in function value. Likewise moving in the negative direction would produce a decrease. Hence the current point cannot be optimal. Please see Appendix \ref{App:Proof_of_Proposition_1} for a formal proof of Proposition \ref{th:simplex}.\\
In the following paragraphs we explain how this proposition can be used to break up the task of solving problem (\ref{dual}) into $2^N-1$ subproblems. Given a unique subset $I$ of non-zero dimensions of $\bm{\alpha}$ Each of these subproblems involves finding the unconstrained optimum and checking if this point lies within the simplex. We now give an example of a single subproblem.
To simplify notation, let $Q \triangleq \eta A_t^\top A_t$.  
We note that:
\begin{align}
\frac{\partial D(\bm{\alpha})}{\partial \bm{\alpha}}\biggr\rvert_{\bm{\alpha} = \bm{\alpha}_*} =  -Q \bm{\alpha}_* + \bm{b}_t.
\end{align}
If we knew that $\bm{\alpha}_*$ had exclusively non-zero coordinates, then by applying Proposition (\ref{th:simplex}) to the dual objective $D$ we can recover a solution $\bm{\alpha}_*$  and by solving the following linear system:
\begin{equation}
    \begin{bmatrix}
        \bm{\alpha}_*\\
        -c
    \end{bmatrix}
    = \texttt{solve}_{\bm{x} \geq 0} \left(
    \begin{bmatrix}
        Q & \bm{1}\\
        \bm{1}^\top & 0
    \end{bmatrix}
    \bm{x}
    =
    \begin{bmatrix}
    \bm{b}_t
    \\1
    \end{bmatrix} \right),
\end{equation}
The first $N$ rows of the system would enforce that $\bm{\alpha}_*$ satisfies the condition given by Proposition \ref{th:simplex}, and the last row of the linear system would ensure that the coordinates of $\bm{\alpha}_*$ sum to one. In the general case, we do not know which coordinates of $\bm{\alpha}_*$ are non-zero.
However, since typical problems are in low-dimension $N$, we can enumerate all possibilities of subsets of non-zero coordinates for $\bm{\alpha}_*$.\\
We detail this further. 
We consider a non-empty subset $I \subseteq [N]$, for which we define:
\begin{equation}\label{eq:defqibi}
    Q_{[I \times I]} \triangleq \left(Q_{i,j} \right)_{i\in I, j\in I}, \quad
    \bm{b}_{[I]} \triangleq \left(b_{t, i} \right)_{i\in I}.
\end{equation}
We then solve the corresponding linear subsystem:
\begin{equation}\label{eq:phi_eq}
    \begin{bmatrix}\bm{\phi}^{(I)}\\-c\end{bmatrix} \triangleq 
    \texttt{solve}_{\bm{x} \geq 0} \left(
        \begin{bmatrix}
            Q_{[I \times I]} & \bm{1}\\
            \bm{1}^\top & 0
        \end{bmatrix}{\bm{x}} = \begin{bmatrix}\bm{b}_{[I]}\\1\end{bmatrix} \right).
\end{equation}
This $\bm{\phi}^{(I)} \in \mathbb{R}^{|I|}$ can then be lifted to $\mathbb{R}^N$ by setting zeros at coordinates non-contained in $I$.
Formally, we define $\bm{\psi}^{(I)} \in \mathbb{R}^N$ such that:
\begin{equation}\label{eq:psiphi}
    \forall i \in [N], \: \psi^{(I)}_i = \begin{cases}
        \phi^{(I)}_{i} \text{ if } i \in I, \\
        0 \text{ otherwise.}
    \end{cases}
\end{equation}
Therefore, given $I \subseteq [N]$, we can generate a candidate solution $\bm{\psi}^{(I)}$ for problem (\ref{dual}) by solving a linear system in dimension $|I|$.
In the following proposition, we establish that doing so for all possibilities of $I$ guarantees to find the correct solution:


\begin{proposition2}[Problem Equivalence]
We define the set of feasible solutions reached by the different candidates $\bm{\psi}^{(I)}$:
\begin{equation}
\Psi = \left\{ \bm{\psi}^{(I)}: I \subseteq [N], I \neq \varnothing \right\} \cap \Delta_{N}.
\end{equation}
Then we have that:
\begin{equation}
    \argmax_{\bm{\alpha} \in \Delta_N} D(\bm{\alpha}) = \argmax_{\bm{\psi} \in \Psi} D(\bm{\psi}).
\end{equation}
In other words, we can find the optimal solution of (\ref{dual}) by enumerating the members of $\Psi$ and picking the one with highest objective value.
\end{proposition2}
This proposition is trivially true because $\Psi$ is simply the intersection between (i) the original feasible set $\Delta_N$ and (ii) the set of vectors that satisfy the necessary condition of Optimality given by Proposition \ref{th:simplex}.
This insight results in Algorithm \ref{alg:dualsol} which returns a optimal solution to the dual problem. We characterise this claim in the following proposition. 
\begin{proposition2}[Sets of Solutions]
Algorithm \ref{alg:dualsol} returns a solution $\bm{\alpha}^*$  that satisfies $\bm{\alpha}^* \in \argmax_{\bm{\alpha}\in \Delta_N} D(\bm{\alpha})$. This is true even when a the dual does not have a unique solution. Proof given in Appendix \ref{App:Proof_of_Proposition_2}.
\end{proposition2}
Procedurally Algorithm \ref{alg:dualsol} starts by computing $Q=-\eta A_{t}^\top A_{t}$ and $\bm{b}_{t}$ to form a ``master'' system $Q\bm{x}=\bm{b}$. We consider each of the $2^N-1$ subsystems of $Q\bm{x}=\bm{b}_t$ defined by an element of the set set $I$ (lines 1-3), where $I$ represents the set of none-zero dimensions of each subsystem. For each subsystem we get a independent subproblem. To ensure the solution to each subproblem will satisfy $\sum_{n=1}^{N} x^n = 1$ and all partial derivatives have equal value an extra row-column is introduced to each system (line 2). We then compute the point which satisfies the optimality conditions detailed in Proposition \ref{th:simplex} for each subsystem, by solving for $\bm{x}$. We then check if each of these points is feasible, that is, belongs to $\Delta_N$ by examining signs $x^n\geq 0, \forall n \in \{1,...,n\}$ (line 4). Note, we have by construction $\sum_{n=1}^{N} x^n = 1$. Finally, we select $\bm{\alpha}^*$ as the feasible point with maximum dual value (lines 7-8). The optimal $\bm{\alpha}^*$ is then used to define the weight update (\ref{update}). For example, if $\bm{\alpha}^*=[1,0,...,0]^\top$ an SGD step $\vw_{t+1} = \vw_t -\eta\nabla \ell_{z_t}(\vw_t)$ will be taken. \\
The BORAT algorithm can be viewed as automatically picking the best step out of a maximum of $(2^N-1)$ possible options, where each option has a closed form solution. Although the computational complexity of this method is exponential in $N$, this algorithm is still very efficient in practice for two reasons. First, we only consider small $N$. Second, as sub problems can be solved independently, it permits a parallel implementation. With a fully parallel implementation the time complexity of this algorithm reduces to $\mathcal{O}\left(N^3\right)$. Empirically with such an implementation, for $N\leq10$ we observe approximately a linear relationship between $N$ and time taken per training epoch. See Appendix \ref{App:Additional_Results} for a comparison of training epoch time between BORAT and SGD. 


\begin{algorithm}
\caption{Dual Optimisation Algorithm}\label{alg:dualsol}
\begin{algorithmic}[1]
\Require $\eta$, $N$, $\mathcal{P}=\{ S: S\subseteq\{1,2,...,N\}, S\neq  \varnothing\}$, $ Q = \eta A_{t}^\top A_{t}$, $\bm{b}_t$, $d_{max} = 0$.  
\For{$I \in P$} 
\State $\hat{Q} = \begin{bmatrix}
Q_{[I \times I]} & \bm{1}\\
\bm{1}^\top & 0
\end{bmatrix}$, $\hat{\bm{b}} = \begin{bmatrix}\bm{b}_{[I]}\\1\end{bmatrix}$ \Comment{see Equation (\ref{eq:defqibi}) for definitions}
\State $\bm{\phi}^{[I]} = \texttt{solve}_{\bm{x}}\left(\hat{Q}\bm{x}=\hat{\bm{b}}\right)$ \Comment{solve the subsystem, see Equation (\ref{eq:phi_eq})}
\If{$\bm{\phi}^{[I]}_i \geq 0, \hspace{0.1cm} \forall i \in \{1,2,..., |I|\}$} \Comment{check for non negativity of solution}
\State $\bm{\psi}^{(I)} = \texttt{select}(\bm{\phi}^{[I]}, I)$ \Comment{select elements according to Equation (\ref{eq:psiphi})}
\State $d = -\frac{1}{2}\bm{\psi}^{(I)\top} Q\bm{\psi}^{(I)} + \bm{\psi}^{(I)\top} \bm{b}_{t}$ \Comment{compute the dual value}
\If{$d \geq d_{max}$} \Comment{save maximum value}
\State $ d_{max} = d, \bm{\alpha}^* = \bm{\psi}^{(I)}$
\EndIf
\EndIf
\EndFor
\State \textbf{Return} $\bm{\alpha}^*$ \Comment{return optimal value}
\end{algorithmic}
\end{algorithm}

\subsection{Computation considerations}
While the method of adding linear approximations detailed in (\ref{eq:bundle}) requires running Algorithm \ref{alg:dualsol} at each inner loop iteration, when adding an additional element $\alpha^n$ if we keep track of the best dual value we need only compute the $2^{N-1}$ new subproblems that include non-zero $\alpha^n$. Thus we only need to run Algorithm \ref{alg:dualsol} once for each $\w_t$ update, that is once per $N-1$ batches.

\subsection{Summary of Algorithm}
The full BORAT method is outlined in Algorithm 2. The bundle is constructed in lines 4-6. The update is obtained in lines 7-9 using Algorithm 1. Finally, the updated parameters are projected to the feasible region $\Omega$ in line 9.\\
In the majority of our experiments, we accelerate BORAT with Nesterov momentum. We use Nesterov momentum as we find this helps produce strong generalise performance. The update step at line 9 of Algorithm \ref{alg:borat} is then replaced by (i) a velocity update $\vv_{t} = \mu \vv_{t-1} - \eta A_{t}\bm{\alpha}$ and (ii) a parameter update $\vw_{t+1} = \Pi_{\Omega}\left(\vw_t + \mu \vv_{t} \right)$.

\begin{algorithm}[ht]
\caption{The BORAT Algorithm}\label{alg:borat}
\begin{algorithmic}[1]
\Require maximal learning-rate $\eta$, maximum bundle size $N\geq 2$, initial feasible $\w_0 \in \Omega$ .
\State $t = 0$
\While{not converged}
\State Get $\ell_{z_t}(\hat{\w}_t^1)$, $\nabla \ell_{z_t} (\hat{\w}_t^1)$ with $z_t$ drawn i.i.d. 
\For{$n=2,...,N-1$}\Comment{sample additional points}
\State Sample $z_{t,n} \in \Z$, $\ell_{z_{t}^n}(\hat{\w}_{t}^n)$, $\nabla \ell_{z_{t}^n}(\hat{\w}_{t}^n)$ 
\State compute $\hat{\w}_{t}^{n+1}$ according to (\ref{eq:bundle})
\EndFor
\State compute $-\eta A_{t}^\top A_{t}$ and $\bm{b}_{t}$ 
\State $\bm{\alpha}^* = \argmax_{\bm{\alpha} \in \Delta_N}\left\{ -\frac{\eta}{2}\bm{\alpha}^\top A_{t}^\top A_{t}\bm{\alpha} + \bm{\alpha}^\top \bm{b}_{t} \right\}$ \Comment{see Algorithm \ref{alg:dualsol} for and details}
\State $\w_{t+1} = \Pi_\Omega\left(\w_t - \eta A_{t}\bm{\alpha}^*\right)$ \Comment{here $\Pi_\Omega$ is the projection onto $\Omega$}
\State $t=t+1$
\EndWhile\label{euclidendwhile}
\State \textbf{end while} 
\end{algorithmic}
\end{algorithm}

\section{Justification and Analysis}
The interpolation setting gives by definition, $\fstar = 0$. However, more subtly, it also allows the updates to rely on the stochastic estimate $\ell_{z_t}(\vw_t)$ rather than the exact but expensive $f(\vw_t)$. Intuitively, this is possible because in the interpolation setting, we know the global minimum is achieved for each loss function $\ell_{z_t}(\vw_t)$ simultaneously. The following results formalise the convergence guarantee of BORAT in the stochastic setting. Note here we prove these result for BORAT with the minor modification, that is, all linear approximations are formed using the same mini-batch of data, $\ell_{z_{t}^n}$ = $\ell_{z_{t}}$ for all $n\in\{2,...,N-1\}$. First, we consider the standard convex setting, where we additionally assume the interpolation assumption is satisfied and that each $\ell_z$ is Lipschitz continuous. Next we consider an important class of non-convex functions used for analysis in previous works related to interpolation \citep{vaswani2019}.

\begin{utheorem}[Convex and Lipschitz] \label{th:convex}
Let $\Omega$ be a convex set. We assume that for every $z \in \Z$, $\ell_z$ is convex and $C$-Lipschitz. Let $\wstar$ be a solution of (\ref{eq:main_problem}), and assume that we have perfect interpolation: $\forall z \in \Z, \: \ell_z(\wstar) = 0$. Then BORAT for $N\geq2$ applied to $f$ satisfies:
\begin{equation}
f\left(\tfrac{1}{T+ 1} \sum\limits_{t=0}^T \vw_t \right) - \fstar
   \leq  C\sqrt{\frac{\|\vw_0 - \wstar\|^2 }{(T+1)}} + \frac{\|\vw_0 - \wstar\|^2 }{\eta(T+1)}.
\end{equation}
\end{utheorem}
Hence BORAT recovers the same asymptotic rate as SGD without the need to reduce the learning rate $\eta$. In the Appendix \ref{App:Convex_Results} we show that for convex and $\beta$-smooth and the $\alpha$-strongly convex settings BORAT recovers rates of $O(1/T)$ and $O(\exp(\alpha T))$ respectively. Note the earlier version of this work provides convergence results for ALI-G without perfect interpolation.\\
We follow earlier work \citep{vaswani2019} and provide a convergence rate for BORAT applied to non-convex functions that satisfy the Restricted Secant Inequality (RSI). A function is said to satisfy the RSI condition with constant $\mu$ over a the set $\Omega$ if the following holds:
\begin{align}
\forall \w \in \Omega, \langle\nabla \ell_z(\w),\w-\wstar\rangle \leq \mu||\w-\wstar||.
\end{align}
\begin{utheorem}[RSI]\label{th:rsi}
We consider problems of type (\ref{eq:main_problem}). We assume 
$\ell_z$ satisfies the RSI  with constant $\mu$, smoothness constant $\beta$ and perfect interpolation e.g. $l_z(\w^*)=0, \hspace{0.1cm}\forall z \in\mathcal{Z}$. Then if set  $\eta \leq \hat{\eta} = \min\{\frac{1}{4\beta},\frac{1}{4\mu}, \frac{\mu}{2\beta^2}\}$ then in the worst case we have:
\begin{align}
f(\w_{T+1}) - f^*&\leq \text{exp}\left( \left(-\frac{3}{8}\hat{\eta}\mu \right)T\right)||\w_0 - \w^*||^2. 
\end{align}
\end{utheorem}
In this setting BORAT recovers the same asymptotic rate as SGD.

\section{Related work}\label{Related_work}

\paragraph{Bundle Methods.}
Bundle methods have been primarily proposed for the optimisation of non-smooth convex functions \citep{Lemarechal1995,Smola2007,Auslender2009}. However, these works do not treat the stochastic case, as they consider small problems where the full gradient can be cheaply evaluated. To our knowledge our work is the first to propose a bundle method for the optimisation of stochastic non-convex problems. 

\paragraph{Interpolation in Deep Learning.}
The interpolation property of DNNs was utilised by early efforts to analyse the convergence speed of SGD. These works demonstrate that SGD achieves the convergence rates of full-batch gradient descent in the interpolation setting \citep{Ma2018a,Vaswani2019a,Zhou2019}. Such works are complementary to ours in the sense that they provide a convergence analysis of an existing algorithm for deep learning. In a different line of work, \cite{Liu2019} propose to exploit interpolation to prove convergence of a new acceleration method for deep learning. However, their experiments suggest that the method still requires the use of a hand-designed learning-rate schedule.

\paragraph{Adaptive Gradient Methods.}
Similarly to BORAT, most adaptive gradient methods also rely on tuning a single hyperparameter, thereby providing a more pragmatic alternative to SGD that needs a specification of the full learning-rate schedule. While the most popular ones are Adagrad \citep{duchi2011}, RMSPROP \citep{Tieleman2012}, Adam \citep{kingma2014} and AMSGrad \citep{Reddi2018}, there have been many other variants \citep{Zeiler2012, Orabona2015, Defossez2017, Levy2017, Mukkamala2017, Zheng2017, Bernstein2018, Chen2018, Shazeer2018, Zaheer2018, Chen2019, Loshchilov2019, Luo2019}. However, as pointed out in \citep{Wilson2017}, adaptive gradient methods tend to give poor generalization in supervised learning. In our experiments, the results provided by BORAT are significantly better than those obtained by the most popular adaptive gradient methods. Recently, \citet{Liu2019a} have proposed to ``rectify'' Adam with a learning-rate warm up, which partly bridges the gap in generalization performance between Adam and SGD. However, their method still requires a learning-rate schedule, and thus remains difficult to tune on new tasks.

\paragraph{Adaptive Learning-Rate Algorithms.}
\citet{Vaswani2019a} show that one can use line search in a stochastic setting for interpolating models while guaranteeing convergence. This work is complementary to ours, as it provides convergence results with weaker assumptions on the loss function, but is less practically useful as it requires up to four hyperparameters, instead of one for BORAT. Less closely related methods, included second-order ones, adaptively compute the learning-rate without using the minimum \citep{Schaul2013,Martens2015,Tan2016,Zhang2017a,Baydin2018,Wu2018,Li2019,Henriques2019}, but do not demonstrate competitive generalization performance against SGD with a well-tuned hand-designed schedule.

\paragraph{$L_4$ Algorithm.}
The $L_4$ algorithm \citep{Rolinek2019} also uses a modified version of the Polyak step-size. However, the $L_4$ algorithm computes an online estimate of $\fstar$ rather than relying on a fixed value. This requires three hyperparameters, which are in practice sensitive to noise and crucial for empirical convergence of the method. In addition, $L_4$ does not come with convergence guarantees. In contrast, by utilizing the interpolation property and a single learning rate, our method is able to (i) provide reliable and accurate minimization with only a single hyperparameter, and (ii) offer guarantees of convergence in the stochastic convex setting.

\paragraph{Frank-Wolfe Methods.}
The proximal interpretation in Equation (\ref{alig_prox}) allows us to draw additional parallels between ALI-G and existing methods. In particular, the formula of the learning-rate $\alpha^1$ may remind the reader of the Frank-Wolfe algorithm \citep{Frank1956} in some of its variants \citep{Locatello2017}, or other dual methods \citep{Lacoste-Julien2013,Shalev-Shwartz2016}. This is because such methods solve in closed form the dual of problem (\ref{alig_prox}), and problems in the form of (\ref{alig_prox}) naturally appear in dual coordinate ascent methods \citep{Shalev-Shwartz2016}.\\
When no regularization is used, ALI-G and Deep Frank-Wolfe (DFW) \citep{Berrada2019} are procedurally identical algorithms. This is because in such a setting, one iteration of DFW also amounts to solving (\ref{alig_prox}) in closed-form -- more generally, DFW is designed to train deep neural networks by solving proximal linear support vector machine problems approximately. However, we point out the two fundamental advantages of BORAT over DFW: (i) BORAT can handle arbitrary (lower-bounded) loss functions, while DFW can only use convex piece-wise linear loss functions; and (ii) as seen previously, BORAT provides convergence guarantees in the convex setting.

\paragraph{SGD with Polyak's Learning-Rate.}
\citet{Oberman2019} extend the Polyak step-size to rely on a stochastic estimation of the gradient $\nabla \ell_{z_t}(\vw_t)$ only, instead of the expensive deterministic gradient $\nabla f(\vw_t)$. However, they still require evaluating $f(\vw_t)$, the objective function over the entire training data set, in order to compute its learning-rate, which makes the method impractical. In addition, since they do not do exploit the interpolation setting nor the fact that regularization can be expressed as a constraint, they also require the knowledge of the optimal objective function value $\fstar$. We also refer the interested reader to the recent analysis of \cite{Loizou2020}, which provides a set of improved theoretical results.

\paragraph{{\sc aProx} Algorithm.}
Most similar to this work \citep{asi2019} have recently introduced the {\sc aProx} algorithm, a family of proximal stochastic optimisation algorithms for convex problems. The {\sc aProx} ``truncated model''  and the {\sc aProx} ``relatively accurate models'' share aspects to ALI-G and BORAT respectively. However, there are four clear advantages of this work over \citep{asi2019} in the interpolation setting, in particular for training neural networks. First, our work is the first to empirically demonstrate the applicability and usefulness of the algorithm on varied modern non-convex deep learning tasks -- most of our experiments use several orders of magnitude more data and model parameters than the small-scale convex problems of \citep{asi2019}. Second, our analysis and insights allow us to make more aggressive choices of learning rate than \citep{asi2019}. Indeed, \citep{asi2019} assume that the maximal learning-rate is exponentially decaying, even in the interpolating convex setting. In contrast, by avoiding the need for an exponential decay, the learning-rate of BORAT requires only one hyperparameters instead of two for {\sc aProx}. Third, our analysis proves fast convergence in function space rather than iterate space. Fourth, unlike BORAT \cite{asi2019} do not include the global lower bound in their {\sc aProx} ``relatively accurate models'' and does not provide details on how to solve each proximal problem efficiently.

\section{Experiments}\label{sec:experiments}

We split the experimental results into two sections. The first section demonstrates the strong generalisation performance of ALI-G and BORAT on a wide range of tasks. Here we compare ALI-G and BORAT to other single hyperparameter optimisation algorithms. The second section shows for tasks where SGD and ALI-G are sensitive to the learning rate, using larger $N$ increases both the robustness to the learning rate and task regularisation hyperparameter. \\
For experiments we chosen to investigate BORAT with $(N=3)$ and $(N=5)$ and refer to the resulting algorithms as BORAT3 and BORAT5 respectively. For $N\geq 2$ BORAT uses more than one batch of data for each update. In order to give a fair comparison we keep the number of passes through the data constant for all experiments. This has the effect that BORAT3 and BORAT5 respectively make a half and a quarter of the weight updates of SGD and ALIG. The time per epoch of BORAT is very similar to that of all other methods, see Appendix \ref{App:Additional_Results} for more details. Hence all methods approximately have the same time cost per epoch. Consequently, faster convergence in terms of number of epochs translates into faster convergence in terms of wall-clock time. \\
The code to reproduce our results is publicly available\footnote{\url{https://github.com/oval-group/borat}}. For baselines we use the official implementation where available in PyTorch \citep{Paszke2017}. We use our implementation of $L_4$, which we unit-test against the official TensorFlow implementation \citep{Abadi2015}. we employ the official implementation of DFW\footnote{\url{https://github.com/oval-group/dfw}} and we re-use their code for the experiments on SNLI and CIFAR.

\subsection{Effectiveness of ALI-G and BORAT}

We empirically compare ALI-G and BORAT to the optimisation algorithms most commonly used in deep learning using standard loss functions. Where the hyperparameters of each algorithm are cross validated to select a best performing model. We consider a wide range of problems: training a wide residual network on SVHN, training a Bi-LSTM on the Stanford Natural Language Inference data set, training both wide residual networks and densely connected networks on the CIFAR data sets and lastly training a wide residual network on the Tiny Imagenet data set. For these problems, we demonstrate that ALI-G $(N=2)$ and BORAT for $N\in\{3,5\}$ obtains comparable performance to SGD with a hand-tuned learning rate schedule, and typically outperforms adaptive gradient methods. Finally, we empirically assess the performance of BORAT and its competitors in terms of training objective on CIFAR-100 and ImageNet, in order to demonstrate the scalability of BORAT to large-scale settings. Note that the tasks of training wide residual networks on SVHN and CIFAR-100 are part of the DeepOBS benchmark \citep{Schneider2019}, which aims at standardising baselines for deep learning optimisers. In particular, these tasks are among the most difficult ones of the benchmark because the SGD baseline benefits from a manual schedule for the learning rate where as ALI-G and BORAT uses a single fixed value. Despite this, our set of experiments demonstrate that ALI-G and BORAT obtains competitive performance in relation to SGD. In addition, our methods significantly outperforms adaptive gradient methods. All Experiments were performed on a single GPU (SVHN, SNLI, CIFAR) or on up to 4 GPUs (ImageNet). 

\subsubsection{Wide Residual Networks on SVHN}\label{sec:SVHN}

\paragraph{Setting.}
The SVHN data set contains 73k training samples, 26k testing samples and 531k additional easier samples. From the 73k difficult training examples, we select 6k samples for validation; we use all remaining (both difficult and easy) examples for training, for a total of 598k samples. We train a wide residual network 16-4 following \citep{Zagoruyko2016}.

\paragraph{Method.}
For SGD, we use the manual schedule for the learning rate of \citet{Zagoruyko2016}. For $L_4$Adam and $L_4$Mom, we cross-validate the main learning-rate hyperparameter $\alpha$ to be in $\{0.0015, 0.015, 0.15\}$ ($0.15$ is the value recommended by \citep{Rolinek2019}). For other methods, the learning rate hyperparameter is tuned as a power of 10. The $\ell_2$ regularization is cross-validated in $\{0.0001, 0.0005\}$ for all methods but BORAT. For BORAT, the regularization is expressed as a constraint on the $\ell_2$-norm of the parameters, and its maximal value is set to $100$. SGD, BORAT and BPGrad use a Nesterov momentum of 0.9. All methods use a dropout rate of 0.4 and a fixed budget of 160 epochs, following \citep{Zagoruyko2016}. A batch size of 128 is used for all experiments.

\begin{table}[h]
\centering
\begin{tabular}{lcc|lcc}
    \toprule
    \multicolumn{4}{c}{Test Accuracy on SVHN (\%)} \\
    \midrule
    Adagrad & 98.0  & BPGrad & 98.1 \\
    AMSGrad & 97.9  &$L_4$Adam &{\bf 98.2} \\
    DFW & 98.1      &ALI-G & 98.1 \\
    $L_4$Mom & 19.6 &BORAT3 & 98.1  \\
    Adam & 97.9     &  BORAT5 & 98.1 \\
    \cmidrule(lr){1-2} \cmidrule(lr){3-4}
    {\color{red} SGD} &98.3 &{\color{red} SGD$^\dagger$} & 98.4 \\
    \bottomrule
    \end{tabular}
    \caption{\em
    In red, SGD benefits from a hand-designed schedule for its learning-rate.
    In black, adaptive methods, including ALI-G, have a single hyperparameter for their learning-rate.
    $SGD^\dagger$ refers to the performance reported by \citep{Zagoruyko2016}.
}
\label{tab:svhn}
\end{table}

\paragraph{Results.}
A comparison to other methods is presented in Table \ref{tab:svhn}. On this relatively easy task, most methods achieve about 98\% test accuracy. Despite the cross-validation, $L_4$Mom does not converge on this task. However, note that $L_4$Adam achieves accurate results. Even though SGD benefits from a hand-designed schedule, BORAT and other adaptive methods obtain comparable performance on this task.

\subsubsection{Bi-LSTM on SNLI}

\paragraph{Setting.}
We train a Bi-LSTM of 47M parameters on the Stanford Natural Language Inference (SNLI) data set \citep{Bowman2015}. The SNLI data set consists of 570k pairs of sentences, with each pair labeled as entailment, neutral or contradiction. This large scale data set is commonly used as a pre-training corpus for transfer learning to many other natural language tasks where labeled data is scarcer \citep{Conneau2017} -- much like ImageNet is used for pre-training in computer vision. We follow the protocol of \citep{Berrada2019} and we re-use their results for the baselines.

\paragraph{Method.}
For $L_4$Adam and $L_4$Mom, the main hyperparameter $\alpha$ is cross-validated in $\{0.015, 0.15\}$ -- compared to the recommended value of 0.15, this helped convergence and considerably improved performance. The SGD algorithm benefits from a hand-designed schedule, where the learning-rate is decreased by 5 when the validation accuracy does not improve. Other methods use adaptive learning-rates and do not require such a schedule. The value of the main hyperparameter $\eta$ is cross-validated as a power of ten for the BORAT algorithm and for previously reported adaptive methods. Following the implementation by \citep{Conneau2017}, no $\ell_2$ regularization is used. The algorithms are evaluated with the Cross-Entropy (CE) loss and the multi-class hinge loss (SVM), except for DFW which is designed for use with an SVM loss only. For all optimisation algorithms, the model is trained for 20 epochs,  using a batch size of 64, following \citep{Conneau2017}.

\begin{table}[h]
\centering
\begin{tabular}{lcc|lcc}
    \toprule
    \multicolumn{6}{c}{Test Accuracy on SNLI (\%)} \\
    \midrule
    & CE & SVM & & CE & SVM \\
    \cmidrule(lr){2-2} \cmidrule(lr){3-3} \cmidrule(lr){5-5} \cmidrule(lr){6-6}
    Adagrad$^*$ &83.8 &84.6 &Adam$^*$ &84.5 &85.0 \\
    AMSGrad$^*$ &84.2 &85.1 &BPGrad$^*$ &83.6 &84.2 \\
    DFW$^*$ & - &{\bf 85.2} & $L_4$Adam &83.3 &82.5 \\
    $L_4$Mom &83.7 &83.2 & BORAT3 &84.4 &{\bf85.2} \\
    ALI-G & {\bf 84.8} & {\bf 85.2} & BORAT5 & 84.5& {\bf85.2}\\
    \cmidrule(lr){1-3} \cmidrule(lr){4-6}
    {\color{red} SGD$^*$} &84.7 &85.2 &{\color{red} SGD$^\dagger$} &84.5 & - \\
    \bottomrule
    \end{tabular}
\caption{\em
In red, SGD benefits from a hand-designed schedule for its learning-rate. In black, adaptive methods have a single hyperparameter for their learning-rate. With an SVM loss, DFW and ALI-G are procedurally identical algorithms
    -- but in contrast to DFW, ALI-G can also employ the CE loss.
    Methods in the format $X^*$ re-use results from \citep{Berrada2019}.
    $SGD^\dagger$ is the result from \citep{Conneau2017}.
    }
\label{tab:snli}
\end{table}

\paragraph{Results.} Table \ref{tab:snli} compares ALI-G and BORAT against the other optimises. Moreover, ALI-G, which requires a single hyperparameter for the learning-rate, outperforms all other methods for both the SVM and the CE loss functions. BORAT3 and BORAT5 achieve the same results to ALI-G for both losses.

\subsubsection{Wide Residual Networks and Densely Connected Networks on CIFAR}

\paragraph{Setting.}\label{cifar_setting}
We follow the methodology of our previous work~\citep{Berrada2019}. We test two architectures: a Wide Residual Network (WRN) 40-4 \citep{Zagoruyko2016} and a bottleneck DenseNet (DN) 40-40 \citep{Huang2017a}. We use 45k samples for training and 5k for validation. The images are centred and normalized per channel. We apply standard data augmentation with random horizontal flipping and random crops. 

\paragraph{Method.}
We compare ALI-G and BORAT to other common single hyperparameter optimisers. Here we  cross validate the hyperparameters in order to find a best performing network for each method. AMSGrad was selected in \citep{Berrada2019} because it was the best adaptive method on similar tasks, outperforming in particular Adam and Adagrad. In addition to these baselines, we also provide the performance of $L_4$Adam, $L_4$Mom, AdamW \citep{Loshchilov2019} and Yogi \citep{Zaheer2018}. We follow the cross validation scheme of \citep{Berrada2019} restating it here for completeness, All methods. employ the CE loss only, except for the DFW algorithm, which is designed to use an SVM loss.  The batch-size is cross-validated in $\{64, 128, 256\}$ for the DN architecture, and $\{128, 256, 512\}$ for the WRN architecture. For $L_4$Adam and $L_4$Mom, the learning-rate hyperparameter $\alpha$ is cross-validated in $\{0.015, 0.15\}$. For AMSGrad, AdamW, Yogi, DFW, ALI-G and BORAT  the learning-rate hyperparameter $\eta$ is cross-validated as a power of 10 (in practice $\eta \in \{0.1, 1\}$ for BORAT). SGD, DFW and BORAT use a Nesterov momentum of 0.9. For all methods excluding ALI-G, BORAT and AdamW, the $\ell_2$ regularization is cross-validated in $\{0.0001, 0.0005\}$ on the WRN architecture, and is set to $0.0001$ for the DN architecture. For AdamW, the weight-decay is cross-validated as a power of 10. For ALI-G, BORAT, $\ell_2$ regularization is expressed as a constraint on the norm of the vector of parameters; and its value is cross validated in $\{50, 75, 100\}$. For all optimisation algorithms, the WRN model is trained for 200 epochs and the DN model for 300 epochs, following \citep{Zagoruyko2016} and \citep{Huang2017a} respectively.

\paragraph{Results.}
Table \ref{tab:cifar} details the results of the comparison of ALI-G and BORAT against other single hyperparameter optimisers for the CE loss only. In this setting, ALI-G and BORAT outperforms AMSGrad, AdamW, Yogi, $L_4$Adam and $L_4$ Mom and constant step size SGD by a large margin. This is true for all values of $N$. The second best method is DFW which also has the restriction that it can only be used in conjunction with the hinge loss. ALI-G and BORAT produce the test accuracy achievable using SGD with the manual learning rate schedules from \citep{Huang2017a} and \citep{Zagoruyko2016} for half of the model data set combinations considered here. For these tasks BORAT provides state of the art results without the need for the learning rate to be manually adapted through out training. For the remaining two combinations; training a DN on Cifar10 and training a WRN on Cifar100, BORAT lags in performance by approximately $0.2\%$ and $2\%$ respectively. With minor variation depending on which version of BORAT is used. Note SGD with a hand tuned learning rate schedule provides a reasonable upper limit of the generalisation performance achievable due to the amount of time that has been put into improving the schedule.

\begin{table}[h]
\centering
\begin{tabular}{lcccc}
    \toprule
    \multicolumn{5}{c}{Test Accuracy on CIFAR (\%)} \\
    \midrule
    &\multicolumn{2}{c}{CIFAR-10} &\multicolumn{2}{c}{CIFAR-100} \\
    \cmidrule(lr){2-3} \cmidrule(lr){4-5}
    & WRN & DN & WRN & DN \\
    \cmidrule(lr){2-2} \cmidrule(lr){3-3} \cmidrule(lr){4-4} \cmidrule(lr){5-5}
    SGD & 91.2 & 91.5 &  67.8 &  67.5 \\
    AMSGrad & 90.8 & 91.7 &  68.7 &  69.4 \\
    AdamW & 92.1 & 92.6 &  69.6 &  69.5 \\
    Yogi & 91.2 & 92.1 &  68.7 &  69.6 \\
    DFW & 94.2 & 94.6 & 76.0 &  73.2 \\
    $L_4$Adam & 90.5 & 90.8 &  61.7 &  60.5 \\
    $L_4$Mom & 91.6 & 91.9 &  61.4 &  62.6 \\
    ALI-G & {\bf 95.4} & 94.5 & {\bf 76.1} &  76.2 \\
    BORAT3 & {\bf 95.4} & {\bf 94.9} &  76.0 &  {\bf 76.5} \\
    BORAT5 & 95.0 & {\bf 94.9} &  75.8 &   75.7 \\
    \bottomrule
    \end{tabular}
\caption{\em
    Test accuracy of single hyperparameter optimisation methods. Each reported result is an average over three independent runs; the standard deviations and optimal hyperparameters are reported in Appendix \ref{App:Additional_Results} (the standard deviations are at most $0.4$ for ALI-G and BORAT).}
\label{tab:cifar}
\end{table}

\subsubsection{Wide Residual Networks on Tiny ImageNet}

\paragraph{Setting.}\label{tiny_setting}
Tiny ImageNet contains 200 classes for training where each class has 500 images. The validation set contains 10,000 images. All images are RGB with 64x64 pixels. The test-set contains 10,000 images however ground truth labels are not freely available. We again use the Wide Residual Network (WRN) detailed in Section \ref{sec:SVHN}. The images are centred and normalized per channel. We apply standard data augmentation with random horizontal flipping and random crops. 

\paragraph{Method.}
We investigate using SGD (with a constant step-size), ALI-G and BORAT to train a WRN on Tiny ImageNet. We use Adam \citep{Kingma2015} as an indicator of what can be expected from a popular adaptive method applied to the same task. We make use of both the cross entropy (CE) and multi-class hinge (SVM) losses. The learning-rate hyperparameter $\eta$ is cross validated in powers of 10, a batch-size of $128$ and a training time of 200 epochs was used for all experiments. The $\ell_2$ regularisation is cross validated in powers of 10 for ADAM. For constant step-size SGD, ALI-G and BORAT we make use of the constrained base regularisation detailed in Section \ref{sec:preliminaries}, cross validating $r \in \{50, 100, 150, 200, 250\}$. Additionally all methods excluding Adam use a Nesterov momentum of 0.9.

\paragraph{Results.}
The best results for SGD, ALI-G, BORAT and Adam are shown in table \ref{tab:tiny}. For both losses Adam performs worst, only achieving $2.4\%$ validation accuracy when optimising the SVM loss. The best results are achieved by BORAT with $N=5$ and $N=3$ for the CE and SVM losses respectively. These results suggest the added gain of BORAT is more pronounced for larger more challenging data sets.

\begin{table}[ht]
\centering
\begin{tabular}{lcc}
    \toprule
    \multicolumn{3}{c}{Validation Accuracy on Tiny ImageNet (\%)} \\
    \midrule
    &CE Loss & Hinge Loss \\
    \cmidrule(lr){2-2} \cmidrule(lr){3-3}
    Adam & 55.0 & 2.4 \\
    SGD & 59.4 & 23.2 \\
    ALI-G & 61.1 & 24.9 \\
    BORAT3 & 61.1 & {\bf44.1} \\
    BORAT5 & {\bf62.1} & 39.4 \\
    \bottomrule
    \end{tabular}
\caption{\em
    Validation accuracy of single hyperparameter optimisation methods on the Tiny ImageNet data set for cross entropy and hinge loss.
    }
\label{tab:tiny}
\end{table}

\subsubsection{Comparing Training Performance on CIFAR-100}

In this section, we empirically assess the performance of ALI-G and its competitors in terms of training objective on CIFAR-100. In order to have comparable objective functions, the $\ell_2$ regularization is deactivated. The learning-rate is selected as a power of ten for best final objective value, and the batch-size is set to its default value. For clarity, we only display the performance of SGD, Adam, Adagrad and BORAT (DFW does not support the cross-entropy loss). The $L_4$ methods diverge in this setting. Here SGD uses a constant learning-rate to demonstrate the need for adaptivity. Therefore all methods use one hyperparameter for their learning-rate. All methods use a fixed budget of 200 epochs for WRN-CIFAR-100 and 300 epochs for DN-CIFAR-100. As can be seen, ALI-G and BORAT provides better training performance than the baseline algorithms on both tasks. Here BORAT3 and BORAT5 are slightly slower than ALI-G due to the fewer number of parameter updates.

\begin{figure}[h]
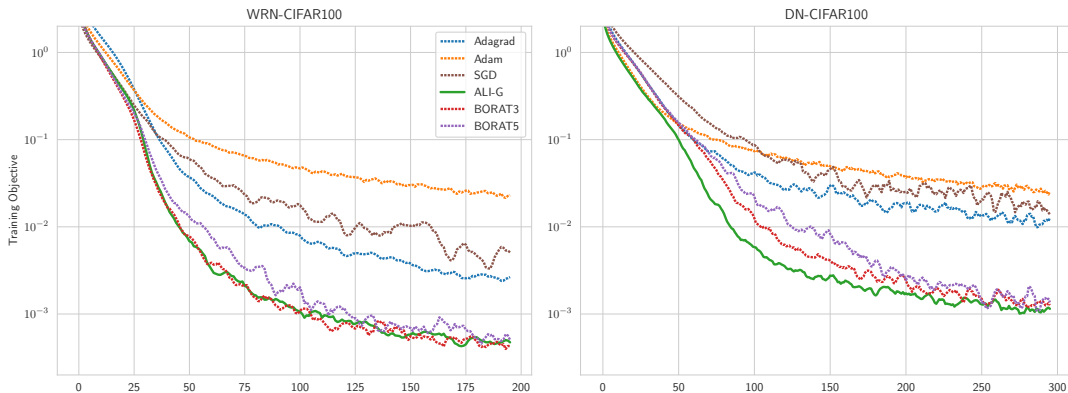

\centering
\footnotesize
\scalebox{0.45}{\input{include/cifar_train_wrn_cifar100.pgf}}
\hspace{-10pt}
\scalebox{0.45}{\input{include/cifar_train_dn_cifar100.pgf}}
\caption{\em
Objective function over the epochs on CIFAR-100 (smoothed with a moving average over 5 epochs).
ALI-G and BORAT reaches a value that is an order of magnitude better than the baselines.
}
\label{fig:cifar100_training}
\end{figure}

\subsubsection{Training at Large Scale}

We demonstrate the scalability of BORAT by training a ResNet18 \citep{He2016} on the ImageNet data set. In order to satisfy the interpolation assumption, we employ a loss function tailored for top-5 classification \citep{Lapin2016}, and we do not use data augmentation. Our focus here is on the training objective and accuracy. ALI-G and BORAT uses the following training setup: a batch-size of 1024 split over 4 GPUs, a $\ell_2$ maximal norm of 400 for $\w$, a maximal learning-rate of 10 and no momentum. As can be seen in figure 5, ALI-G and BORAT reaches 99\% top-5 accuracy in 12 epochs, and accurately minimises the objective function to $2\cdot10^{-4}$ within $90$ epochs.

\vspace{10pt}
\begin{figure}[H]
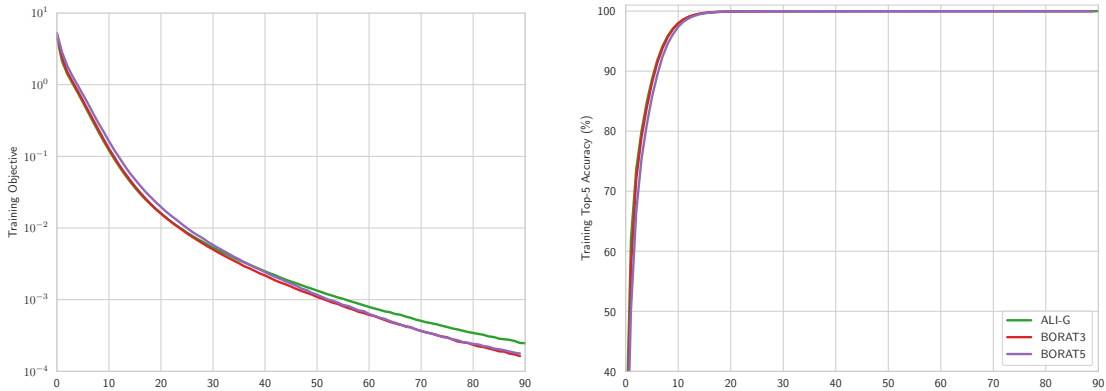

\centering
\footnotesize
\hfill
\scalebox{0.45}{\input{include/imagenet_obj.pgf}}
\hfill
\scalebox{0.45}{\input{include/imagenet_acc.pgf}}
\hfill
\caption{\em
Training performance of a ResNet-18 learning with different versions of BORAT on ImageNet. Note, that all versions converge at similar rates even though for larger values of $N$, BORAT3 and BORAT5 makes far fewer updates.}
\label{fig:imagenet_training}
\end{figure}

\subsection{Robustness of BORAT}

In this section we show that using additional linear approximations increases the robustness of BORAT to its learning rate $\eta$ and the problem regularisation amount $r$. BORAT produces high accuracy models for a wider range of values than SGD and ALI-G. Hence on problems where SGD and ALI-G are sensitive to their learning rate BORAT with $N\geq2$ produces the best performance. These results demonstrate the advantage of using a larger bundle to model the loss in each proximal update.\\
In order to illustrate this increased robustness we assess the stability of constant step size SGD, ALI-G, BORAT3 and BORAT5 to their hyperparameters. This is done by completing a grid search over $\eta$ and $r$ for a number of tasks while holding the batch size and epoch budget constant. This allows us to assess the range of values where these algorithms produce high accuracy models. We perform this grid search for six tasks split over the CIFAR100 and Tiny ImageNet data sets. We choose these data sets as they are challenging yet a model capable of interpolation can still fit on a single GPU. We compare against SGD with a constant learning rate for two reasons. First, SGD effectively uses a bundle of size 1, composed of only the linearization around the current iterate. Second, constant step size SGD has a single learning rate hyperparameter with the same scale as BORAT and also permits the constraint based regularisation described in Section \ref{sec:preliminaries}. 

\subsubsection{Wide Residual Networks on CIFAR100 and Tiny ImageNet}

\paragraph{Setting.}
For a description of the CIFAR100 and Tiny Imagenet data sets please refer to sections \ref{cifar_setting} and \ref{tiny_setting} respectively. 

\paragraph{Method.}
We run four separate experiments on the CIFAR100 data set. The first two of these experiments examine the robustness to hyperparameters of SGD, ALI-G and BORAT when in combination with the cross entropy (CE) and multi-class hinge (SVM) losses. For the second two experiments on CIFAR100 we investigate the same algorithms in the presents of label noise. We limit our self to the CE loss for these two experiments. Label noise is applied by switching the label of the images in the training set with probability $p$ to a random class label in the same super class. We repeat the experiments without label noise on the more challenging Tiny ImageNet data set, resulting in six different settings. For all algorithms we train a wide residual network (WRN) detailed in Section \ref{cifar_setting} and make use of a Nesterov momentum of 0.9 as we found its use produced superior generalisation performance. For all experiments we use a batch size of 128 and perform a grid search over the 20 hyper parameter combinations given by $r \in \{50, 100, 150, 200, 250\}$ and $\eta \in \{0.01, 0.1, 1.0, 10.0\}$. 
\begin{figure}[H]
  \centering
  \includegraphics[width=\textwidth]{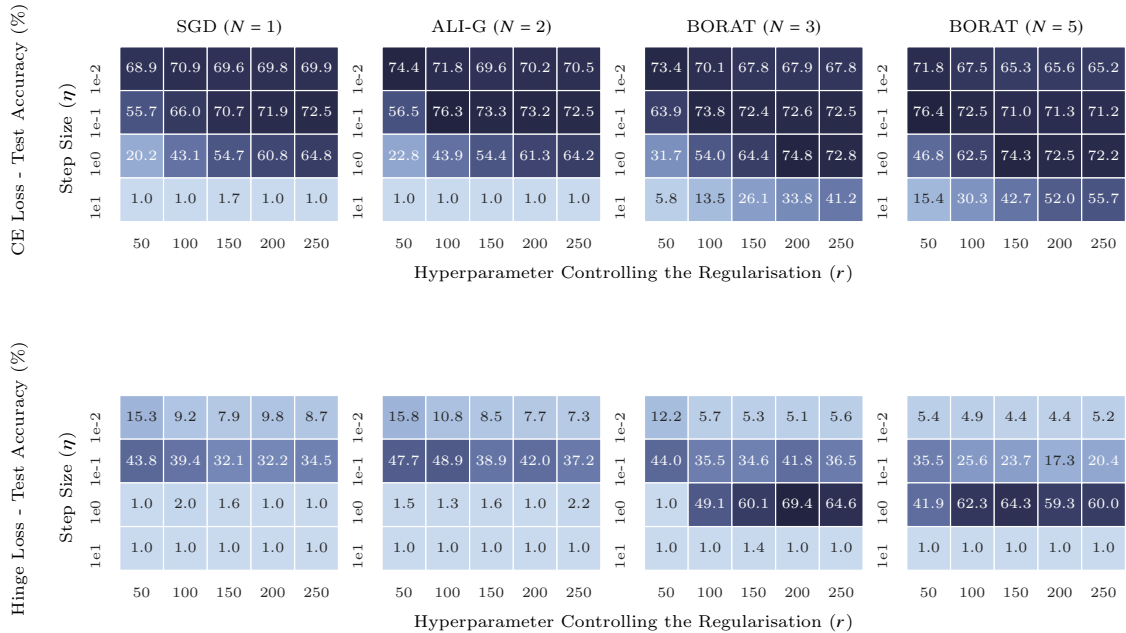}
  \caption{Comparison of SGD and BORAT's robustness to hyperparameters when increasing $N$ for the CE and SVM losses. Experiments are performed on the CIFAR100 data set. Colour represents test performance, where darker colours correspond to higher values. For both losses, increasing $N$ allows for higher learning rates to be used while still producing convergent behaviour. Additionally for the CE loss and $\eta\in\{1.0,10.0\}$ larger $N$ allows for a smaller $r$ or greater levels of regularisation to be used. Consequently increasing $N$ improves the over all robustness to hyperparameters, while not sacrificing generalisation performance.}
  \label{fig:cifar_robust}
\end{figure}

\paragraph{Results.}
Figure \ref{fig:cifar_robust} details the robustness of SGD, ALI-G and BORAT for the CE and SVM losses without label noise. In these two settings ALI-G exhibits similar robustness to SGD, however it outperforms SGD in terms of the performance of the best model trained. On the CE loss SGD and ALI-G are reasonably robust, providing good results for the majority of hyperparameter combinations with $\eta\leq1$. For the SVM loss all algorithms are sensitive to their learning rate $\eta$. SGD and ALI-G only produce an increase in accuracy for small values of $\eta$, hence no model achieves high accuracy in the 200 epoch budget. For both losses, increasing $N$ improves the robustness and produces convergent behaviour for larger values of $\eta$ and smaller values of $r$. This is particularly pronounced for the SVM loss. Here permitting larger $\eta$ allows for a high accuracy network to be trained within the 200 epoch budget. For both losses and $\eta=0.01$ BORAT3 and BORAT5 slightly under perform compared to ALI-G and SGD. This is simply a consequence of these methods making significantly fewer updates. In spite of this the resultant effect of increasing the bundle size is a larger range of hyperparameters that produce good generalisation performance.\\
\noindent
The results from the label noise experiments are shown in Figure \ref{fig:cifar_noisy}. For $p=0.1$ the results closely mirror those where no label noise is used, shown in Figure \ref{fig:cifar_robust}, however here all models achieve roughly $0-5\%$ worse test performance.  When the noise is increased to $p=0.5$ the test error increases drastically by $0-25\%$. In both cases $p=0.1$ and $p=0.5$, increasing $N$ permits to obtain good results for a larger range of values of $\eta$ and $r$. Interestingly, using $N=3,5$ results in the best performance when $p=0.5$. This is somewhat expected since using a larger bundle size means more samples are used to calculate each parameter update and hence helps reduce the effect of the label noise. \\
\begin{figure}[H]
  \centering
  \includegraphics[width=\textwidth]{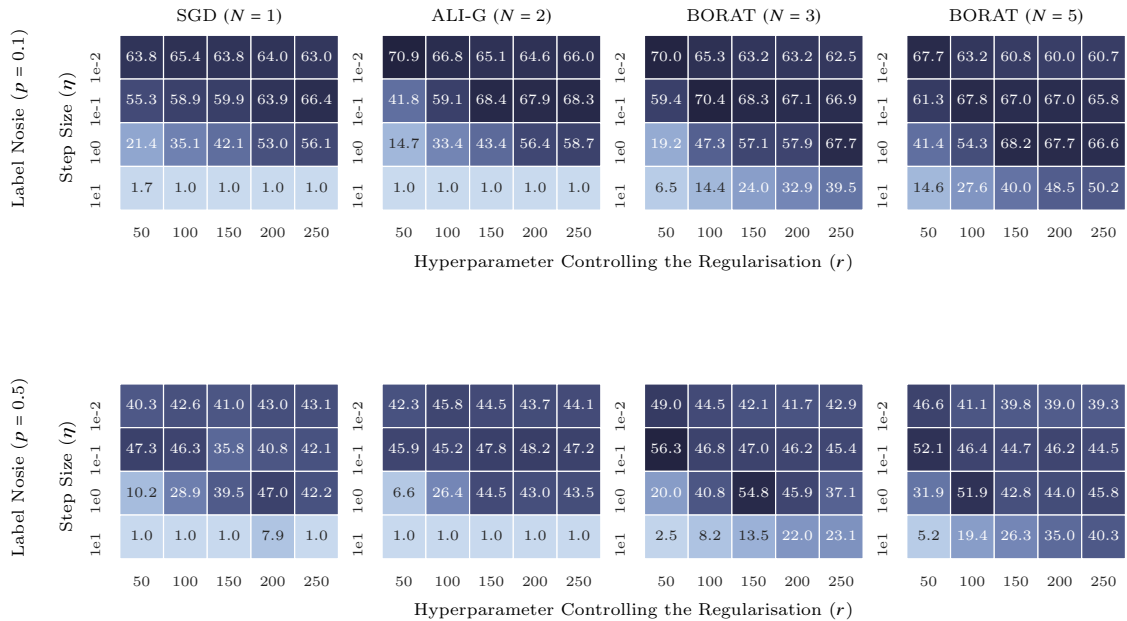}
  \caption{Test accuracy of SGD, ALI-G, BORAT3 and BORAT5 robustness to hyperparameters when trained on CIFAR100 with noisy labels. Here we give results with two different levels of noise $p=0.1$ (upper row) and $p=0.5$ (lower row). For readability purposes, the colour encodes test accuracy, where darker colours correspond to higher values. Increasing $N$ allows for higher learning rates and greater levels of regularisation to be used. Additionally, when the level of noise is high ($p=0.5$, lower row), BORAT3 and BORAT5 significantly outperforms SGD $(N=1)$ and ALI-G $(N=2)$.}
  \label{fig:cifar_noisy}
   \hspace{-0.8cm}
\end{figure}
\noindent
Figure \ref{fig:tiny_robust} details the robustness of SGD, ALI-G and BORAT for the Tiny Imagenet experiments. These results show BORAT offers improved robustness of on multiple data sets as we recover similar performance to CIFAR100. For the CE loss increasing $N$ allows for slightly higher learning rates and greater levels of regularisation to be used for the CE loss. For the SVM loss BORAT produces models with reasonable accuracy with $\eta \in \{0.1,1.0\}$ opposed to SGD and ALI-G that produce poor results for all learning rates excluding $\eta=1.0$. Consequently, when training with the SVM loss for 200 epochs, BORAT produces a drastically better models than SGD and ALI-G. This happens despite BORAT making significantly fewer updates of $\w_t$ in the 200 epochs. 
\begin{figure}[H]
  \centering
  \includegraphics[width=\textwidth]{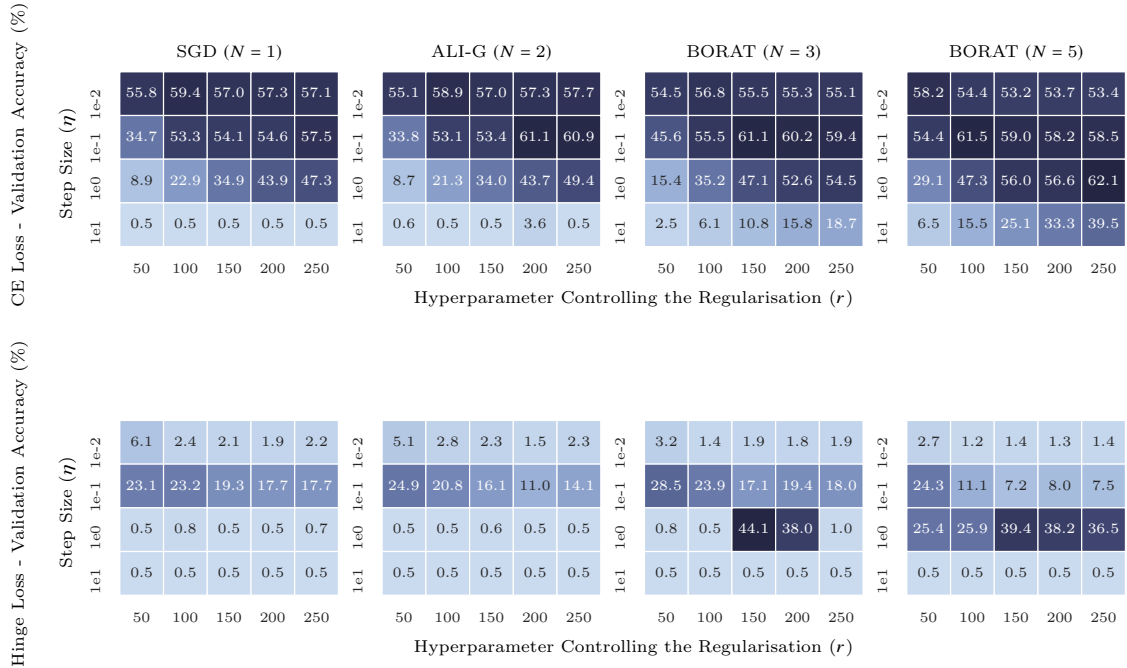}
  \caption{Comparison of SGD, ALI-G and BORAT's robustness to hyperparameters on the Tiny ImageNet data set. Here we investigate the affect of increasing the bundle size when in use with the CE and Hinge losses. Colour represents validation performance, where darker colours correspond to higher values. The the CE loss is used BORAT produces an increase in the range of hyperparameters that result in high accuracy models. For the SVM loss BORAT is the only optimiser to produce accurate results}
  \label{fig:tiny_robust}
\end{figure}
\noindent
Across all six experiments BORAT consistently produces high accuracy models for a larger range of hyperparameter combinations than SGD or ALI-G. This is particularly pronounced for the SVM loss experiments where the decrease sensitivity of BORAT makes all the difference between finding hyperparameters that result in a high accuracy model and not. Consequently, in comparison to its competitors, BORAT is systematically more robust to the choice of learning-rate and regularization hyper-parameter, and also offers better generalization more often than not. Therefore we particularly recommend using BORAT for tasks where hyper-parameter tuning is a highly time consuming task.

\section{Discussion}
In this work have introduced BORAT a bundle method designed for the optimisation of DNNs capable of interpolation. We detailed a special case of BORAT with a minimal bundle size that we name ALI-G. ALI-G produces a algorithm that automatically decays the steps size as the loss of the iterate approaches its minimal value. By only needing to find a good maximal learning rate that is automatically decayed we remove the labour and compute intensive task of finding a good learning rate schedule. Using standard publicly available data sets, we have shown that a ALI-G is highly effective in a number of settings. When tuning a single hyperparameter for the learning-rate, ALI-G achieves state of the art performance while being simple to implement and having minor additional computational cost over SGD. \\
For problems where ALI-G is sensitive to its learning rate its robustness and performance can be enhanced by using BORAT which makes use of a greater bundle size. BORAT produces high accuracy models for a larger range of hyperparameters and is particularly effective when using the multi-class hinge loss or in the presence of label noise. However, BORAT also achieves similar generalisation performance to ALI-G in all settings considered. Our results suggest that BORAT is presently the most robust single hyperparameter optimisation method for DNNs. \\
It may be possible to modify BORAT so that it can make a parameter update after each gradient evaluation by cleverly selecting how linear approximations are added and removed from the bundle. If done correctly this may lead to a further increase in the speed of convergence. However we leave this to future work. When keeping the total number of gradient evaluations constant, a downside of using a larger bundle size for BORAT is that the number of parameter updates decreases.  Despite this, in our experiments, BORAT is able to obtain good results within the same budget of passes through the data. Notably, this also means that BORAT has a time per epoch comparable to that of adaptive gradient methods. Another potential criticism of BORAT is that its memory footprint grows linearly with the bundle size. However in practice, our results show that using a bundle size of three, which corresponds to the same memory cost as ADAM, is often sufficient to obtain improved robustness. Finally, the applicability of BORAT can be limited by the required assumption of interpolation. While we argue that this interpolation property is satisfied in many interesting use cases, it may also not hold true for one or more of the following confounding factors: (i) limited size of the neural network, such as those used in embedded devices; (ii) large size of the training data set, which is becoming increasingly common in machine learning; and (iii) complexity of the loss function, such as in adversarial training. Furthermore, the concept of interpolation itself is ill-defined for unsupervised tasks. Thus, another interesting direction of future work would be to generalise BORAT to non-interpolating settings. This could be achieved by adapting an estimate of the minimum value of the objective function online whilst retaining the desirable quality of minimal hyperparameter tuning.

\newpage
\appendix

\section{Dual Derivation}\label{App:Dual_Derivation}
\begin{lemma}
The dual of the primal problem (\ref{primal}), can be written as follows:
\begin{align*}
\sup_{\bm{\alpha} \in \Delta_N}\left\{ -\frac{\eta}{2}\bm{\alpha}^\top A_{t}^\top A_{t}\bm{\alpha} + \bm{\alpha}^\top \bm{b}_{t}^n \right\}.
\end{align*}
Where $A_{t}$ is defined as a $d \times N$ matrix, whose $n^{th}$ row is $\nabla \ell_{z_t^n}(\hat{\w}_{t}^n)$, $\bm{b}_{t}^n = [b_{t}^1,...,b_{t}^N]^\top$ and $\bm{\alpha} = [\alpha^{1}, \alpha^{2}, \dots, \alpha^{N}]^\top$.
\end{lemma}
\begin{proof}
We start from the primal problem:
\begin{align}\label{primal_}
\argmin_{\w \in \mathbb{R}^d}\left\{ \frac{1}{2\eta}||\w - \w_{t}||^2 + \max_{n\in[N]}
\left\{\nabla \ell_{z_t^n}(\hat{\w}_{t}^n)^\top(\w - \w_{t}) + b_{t}^n\right\} \right\},
\end{align}
We define $\tilde{\w} = \w - \w_{t}$.
\begin{align*}
&\min_{\w \in \mathbb{R}^d}\left\{ \frac{1}{2\eta}||\tilde{\w}||^2 + \max_{n\in[N]}\left\{
\nabla \ell_{z_t^n}(\hat{\w}_{t}^n)^\top\tilde{\w} + b_{t}^n \right\}\right\},
\\
&\min_{\w \in \mathbb{R}^d, \zeta}\left\{ \frac{1}{2\eta}||\tilde{\w}||^2 + \hspace{0.1cm}\zeta  \right\} \hspace{0.2cm}\text{subject to: }\zeta \geq \nabla \ell_{z_t^n}(\hat{\w}_{t}^n)^\top\tilde{\w} + b_{t}^n,\hspace{0.2cm} \forall n\in[N],
\\\hspace{0.2cm}
&\min_{\tilde{\w} \in \mathbb{R}^d, \zeta}\sup_{\alpha^n\geq 0 ,\forall n}\left\{ \frac{1}{2\eta}||\tilde{\w}||^2 + \hspace{0.1cm}\zeta - \sum_{n=1}^N \alpha^n( \zeta - \nabla \ell_{z_t^n}(\hat{\w}_{t}^n)^\top\tilde{\w} - b_{t}^n) \right\},
\\\hspace{0.2cm}
&\sup_{\alpha^n\geq 0 ,\forall n}\min_{\tilde{\w} \in \mathbb{R}^d, \zeta}\left\{ \frac{1}{2\eta}||\tilde{\w}||^2 + \hspace{0.1cm}\zeta - \sum_{n=1}^N\alpha^n( \zeta - \nabla \ell_{z_t^n}(\hat{\w}_{t}^n)^\top\tilde{\w} - b_{t}^n) \right\},\hspace{0.2cm}\text{(strong duality)}\\
&\sup_{\alpha^n\geq 0 ,\forall n}\min_{\tilde{\w} \in \mathbb{R}^d, \zeta}\left\{ \frac{1}{2\eta}||\tilde{\w}||^2 + \hspace{0.1cm}\zeta + \sum_{n=1}^N\alpha^n(  \nabla \ell_{z_t^n}(\hat{\w}_{t}^n)^\top\tilde{\w} + b_{t}^n -\zeta) \right\}.
\end{align*}
The inner problem is now smooth in $\tilde{\w}$ and $\zeta$. We write the KKT conditions:
\begin{align*}
\frac{\partial \cdot}{\partial \tilde{\w}} &= \frac{\tilde{\w}}{\eta} + \sum_{n=1}^N\alpha^n \nabla \ell_{z_t^n}(\hat{\w}_{t}^n) = 0 \\
\frac{\partial \cdot}{\partial \zeta} &= 1 - \sum_{n=1}^N\alpha^n = 0 
\end{align*}
We define $\Delta_N \triangleq \{ \bm{\alpha}\in \mathbb{R}^N:\sum_{n=1}^N\alpha^n = 1,\alpha^n\geq0, n= 1,...,N \}$ as probability simplex over the elements of $\bm{\alpha}$. Thus when we plug in the KKT conditions we obtain:
\begin{align*}
\hspace{0.2cm}
&\sup_{\bm{\alpha} \in \Delta_N}\left\{ \frac{1}{2\eta}||\eta\sum_{n=1}^N\alpha^n \nabla \ell_{z_t^n}(\hat{\w}_{t}^n)||^2 + \sum_{n=1}^N\alpha^n\left(  -\nabla \ell_{z_t^n}(\hat{\w}_{t}^n)^\top(\eta\sum_{m=1}^N\alpha^m \nabla \ell_{z_t^m}(\hat{\w}_{t}^m)) + b_{t}^n\right) \right\},\\
&\sup_{\bm{\alpha} \in \Delta_N}\left\{ \frac{\eta}{2}||\sum_{n=1}^N\alpha^n \nabla \ell_{z_t^n}(\hat{\w}_{t}^n)||^2 -\eta\sum_{n=1}^N\alpha^n(  \nabla \ell_{z_t^n}(\hat{\w}_{t}^n)^\top\sum_{m=1}^N \alpha^m \nabla \ell_{z_t^m}(\hat{\w}_{t}^m)) + \sum_{n=1}^N\alpha^nb_{t}^n \right\},\\
&\sup_{\bm{\alpha} \in \Delta_N}\left\{ \frac{\eta}{2}||\sum_{n=1}^N\alpha^n \nabla \ell_{z_t^n}(\hat{\w}_{t}^n)||^2 -\eta||\sum_{n=1}^N\alpha^n \nabla \ell_{z_t^n}(\hat{\w}_{t}^n)||^2 + \sum_{n=1}^N\alpha^n b_{t}^n \right\},\\
&\sup_{\bm{\alpha} \in \Delta_N}\left\{ -\frac{\eta}{2}||\sum_{n=1}^N\alpha^n \nabla \ell_{z_t^n}(\hat{\w}_{t}^n)||^2 + \sum_{n=1}^N\alpha^n b_{t}^n \right\}.
\end{align*}
Using the definitions of $A_{t}$, $\bm{b}_{t}^n$ and $\bm{\alpha}$, we can recover the following compact form for the dual problem:
\begin{align*}
\sup_{\bm{\alpha} \in \Delta_N}\left\{ -\frac{\eta}{2}\bm{\alpha}^\top A_{t}^\top A_{t}\bm{\alpha} + \bm{\alpha}^\top \bm{b}_{t}^n \right\}.
\end{align*}
\end{proof}
\setcounter{proposition2}{0}
\section{Proof of Proposition 1}\label{App:Proof_of_Proposition_1}
\begin{proposition2}
Let $F: \mathbb{R}^N \to \mathbb{R}$ be a concave function. 
Let us define $\bm{\alpha}_* = \argmax_{\bm{\alpha}\in\Delta}F(\bm{\alpha})$. 
Then there exists $c \in \mathbb{R}$ such that:
\begin{equation}\label{observation_2}
\forall n \in [N] \text{ such that } \alpha^{n}_* > 0, \text{ we have: }
   \frac{\partial F(\bm{\alpha})}{\partial \alpha^n}\bigg\rvert_{\bm{\alpha} = \bm{\alpha}_*} =  c.
\end{equation}
In other words, the value of the partial derivative is shared among all coordinates of $\bm{\alpha}_*$ that are non-zero.
\end{proposition2}
\begin{proof}
We start form the optimally condition for constrained problems:
\begin{align}\label{eq:opt_con_cov}
\langle\nabla F(\bm{\alpha}_*),\bm{\alpha}_*- \bm{\alpha} \rangle\geq 0, \hspace{0.5cm} \forall \bm{\alpha} \in \Delta
\end{align}
We consider the point $\hat{\bm{\alpha}}$. Without loss of generality we assume $\hat{\bm{\alpha}}$ is equal to $\bm{\alpha}_*$ for all but two dimensions $i$ and $j$. We let $\hat{\alpha}_i = 0$ and $\hat{\alpha}_j = \alpha^*_i + \alpha^*_j$. Hence we know $\hat{\bm{\alpha}} \in \Delta$ as we have $\sum_i\hat{\alpha}_i =\sum_i\alpha^*_i = 1$ and $\hat{\alpha}_i\geq 0, \hspace{0.1cm} \forall i$. Letting $\bm{\alpha}=\hat{\bm{\alpha}}$ in (\ref{eq:opt_con_cov}) give us:
\begin{align}
-\frac{\partial F(\bm{\alpha}_*)}{\partial \alpha_i}\alpha^*_i 
+ \frac{\partial F(\bm{\alpha}_*)}{\partial \alpha_j}\alpha^*_i \geq 0.
\end{align}
Rearranging we have:
\begin{align}
\frac{\partial F(\bm{\alpha}_*)}{\partial \alpha_j}\alpha^*_i \geq \frac{\partial F(\bm{\alpha}_*)}{\partial \alpha_i}\alpha^*_i.
\end{align}
Hence for any $\alpha^*_i\neq0$,  we have:
\begin{align}
\frac{\partial F(\bm{\alpha}_*)}{\partial \alpha_j} \geq \frac{\partial F(\bm{\alpha}_*)}{\partial \alpha_i}.
\end{align}
Noticing, this results holds for any $i$ and $j$ gives us the following and proves the result.
\begin{align}
\frac{\partial F(\bm{\alpha}_*)}{\partial \alpha_j} = \frac{\partial F(\bm{\alpha}_*)}{\partial \alpha_i} = c
\end{align}
\end{proof}

\newpage
\section{Proof of Proposition 3}\label{App:Proof_of_Proposition_2}
\setcounter{proposition2}{2}
\begin{proposition2}\label{eight}
Algorithm \ref{alg:dualsol} returns a solution $\bm{\alpha}^*$  that satisfies $\bm{\alpha}^* \in \argmax_{\bm{\alpha}\in \Delta_N} D(\bm{\alpha})$. This is true even when a the dual does not have a unique solution.
\end{proposition2}
\begin{proof}
let $\bm{x}^* \in \argmax_{\bm{\alpha}\in \Delta_N} D(\bm{\alpha})$ such that $\bm{x}^*$ has a maximal number of zero coordinates. $x^*$ exists as we know the solution set is empty. Let $I$ be the set of non-zero coordinates of $x^*$. We denote by $S^{(I)}$ the set of soultions to the linear system associated with $I$:
\begin{equation}\label{eq:s_def}
    \mathcal{S}^{(I)}\triangleq \left\{ \bm{x} \in \mathbb{R}^{|I|}:
        \tilde{Q}{\bm{x}} = \tilde{\bm{b}}, \bm{x} \geq 0\right\}, \text{where}, \tilde{Q} \triangleq \begin{bmatrix}
            Q_{[I \times I]} & \bm{1}\\
            \bm{1}^\top & 0
        \end{bmatrix} \text{and},  \tilde{\bm{b}}= \begin{bmatrix}\bm{b}_{[I]}\\1\end{bmatrix}.
\end{equation}
$\mathcal{S}^{(I)}$ is a polytope as the intersection intersection between the probability simplex and a linear sub-space. There for it admits and vertex representation:
\begin{equation}\label{eq:v_def}
    \mathcal{S}^{(I)} = \text{Conv}(\mathcal{V}^{(I)})
\end{equation}
such that Conv$(\cdot)$ denotes the convex hull operation, and $\mathcal{V}^{(I)}$ is a finite set. Since $\mathcal{S}^{(I)}$ contains $\bm{x}^*[I]$, $\mathcal{S}^{(I)}$ is non-empty and neither is $\mathcal{V}^{(I)}$. let $v$ be and element of $\mathcal{V}^{(I)}$. Note, that if $v$ had one or more zero coordinates, it would be a solution after lifting to $\mathbb{R}^N$ it would also be a optimal solution to the dual as good as $\bm{x}^*$ while having more zero coordinates than x*, this is impossible by the definition of $\bm{x}^*$. Thus we can conclude $v$ has exclusively non-zero coordinates.\\
Since $v$ is a external point of $S^{(I)}$, see section 2.1 of \citet{Bertsekas2009} for a definition. It follows from proposition 2.1.4b of \citep{Bertsekas2009} the columns of $\tilde{Q}$  are independent. Therefore the linear system admin a unique solution $\bm{x}^*$, which is found by Algorithm \ref{alg:dualsol}.
\end{proof}

\vspace{-0.3cm}
\section{Convex Results}\label{App:Convex_Results}

\begin{lemma}\label{lemma:dual_values}
Adding additional linear approximations to the bundle of BORAT can never result in a lower maximal dual value. Formally:
\begin{align}
 D^m(\bm{\alpha}_*) \geq D^l(\bm{\alpha}_*) \hspace{0.5cm} \forall m > l. 
\end{align}
\end{lemma}
\begin{proof}
Any vector of $\Delta_l$ can be lifted to $\Delta_{m}$ by appending $m-l$ zeros to it, which does not impact the value of the objective function.
The lifted set $\Delta_l$ is then a subset of $\Delta_{m}$, hence the result (maximisation performed performed over a larger space).
\end{proof}

\begin{utheorem}[Convex] \label{th:convex_lipschitz}
We assume that $\Omega$ is a convex set, and that for every $z \in \Z$, $\ell_z$ is convex and. Let $\wstar$ be a solution of (\ref{eq:main_problem}), and assume that we have perfect interpolation: $\forall z \in \Z, \: \ell_z(\wstar) = 0$. Then BORAT for $N\geq2$ applied to $f$ satisfies:
\vspace{-0.2cm}
\begin{equation}
\|\w_{t+1} - \w^*\|^2   \leq \|\w_{t} - \w^*\|^2-2\eta \max_{\bm{\alpha} \in \Delta_N}D(\bm{\alpha}),
\end{equation}
where $D$ is defined in (\ref{dual}).
\end{utheorem}
\begin{proof}
\begin{align}
\|\w_{t+1} - \w^*\|^2
&\leq \|\Pi_\Omega(\w_{t} - \eta A_{t}\bm{\alpha}_t )- \w^*\|^2,&\\
&\leq \|\w_{t} - \eta A_{t}\bm{\alpha}_t - \w^*\|^2, &( \Pi_\Omega \text{ projection})  \\
&\leq \|\w_{t} - \w^*\|^2 + \eta^2 \|A_{t}\bm{\alpha}_t\|^2 - 2\eta \langle A_{t}\bm{\alpha}_t,\w_{t} - \w^*\rangle, &(\text{expanding})
\end{align}
\begin{align}
\|\w_{t+1} - \w^*\|^2 - \|\w_{t} - \w^*\|^2 \leq  \eta^2 \|A_{t}\bm{\alpha}_t\|^2 - 2\eta \langle A_{t}\bm{\alpha}_t,\w_{t} - \w^*\rangle, \hspace{0.1cm}(\text{rearranging})
\end{align}
\begin{align}
&= \eta^2 \|A_{t}\bm{\alpha}_t\|^2 - 2\eta \langle A_{t}\bm{\alpha}_t,\w_{t} - \w^*\rangle\\
&=  \eta^2 \|A_{t}\bm{\alpha}_t \|^2
- 2\eta \langle A_{t}\bm{\alpha}_t,\w_{t} - \hat{\w}_{t}^n\rangle - 2\eta \langle A_{t}\bm{\alpha}_t,\hat{\w}_{t}^n - \w^*\rangle , \\
&=  \eta^2 \|A_{t}\bm{\alpha}_t \|^2
- 2\eta \langle A_{t}\bm{\alpha}_t,\w_{t} - \hat{\w}_{t}^n\rangle- 2\eta \sum_{n=1}^{N-1}\alpha^n \nabla \ell_{z_t}(\hat{\w}_{t}^n)^\top (\hat{\w}_{t}^n- \w^*),(A_{t}\alpha_t^N=0)\\
&\leq  \eta^2 \|A_{t}\bm{\alpha}_t \|^2
- 2\eta \langle A_{t}\bm{\alpha}_t,\w_{t} - \hat{\w}_{t}^n\rangle - 2\eta \sum_{n=1}^{N-1} \alpha_{t}^n(\ell_{z_t}(\hat{\w}_{t}^n) - \ell_{z_t}(\w^*)), \hspace{0.5cm}(\text{convexity})\\
&\leq \eta^2 \|A_{t}\bm{\alpha}_t \|^2
- 2\eta \sum_{n=1}^{N-1}\alpha^n \nabla \ell_{z_t}(\hat{\w}_{t}^n)^\top (\w_t -\hat{\w}_{t}^n)- 2\eta \sum_{n=1}^{N-1} \alpha_{t}^n(\ell_{z_t}(\hat{\w}_{t}^n) - \ell_{z_t}(\w^*)),\\
&\leq  \eta^2 \|A_{t}\bm{\alpha}_t \|^2
- 2\eta \sum_{n=1}^{N-1}\alpha^n [\ell_{z_t}(\hat{\w}_{t}^n) - \nabla \ell_{z_t}(\hat{\w}_{t}^n)^\top (\hat{\w}_{t}^n -\w_t)] + 2\eta \sum_{n=1}^{N-1} \alpha_{t}^n \ell_{z_t}(\w^*),\\
&\leq \eta^2 \|A_{t}\bm{\alpha}_t \|^2 - 2\eta \bm{\alpha}_{t} \bm{b}_{t} - 2\eta \sum_{n=1}^{N-1} \alpha_{t}^n\ell_{z_t}(\w^*),\hspace{1cm}(\bm{b}_{t}\text{ definition})\\
&\leq -2\eta D(\bm{\alpha}_t)  + 2\eta \sum_{n=1}^{N-1} \alpha_{t}^n \ell_{z_t}(\w^*),\hspace{1cm}(D\text{ definition})\\
&\leq -2\eta D(\bm{\alpha}_t)  + 2\eta(1 - \alpha_{t}^N) \ell_{z_t}(\w^*),\hspace{1cm}(\sum_{n=1}^{N} \alpha_{t}^n = 1)\\
&\leq -2\eta D(\bm{\alpha}_t), \hspace{1cm}(\ell_{z_t}(\w^*)=0,  \text{ perfect interpolation} )\\
&\leq -2\eta \max_{\bm{\alpha} \in \Delta_N}D(\bm{\alpha})\hspace{1cm}(\bm{\alpha}_t \text{definition}) \label{prog_dual}
\end{align}
\end{proof}
A consequence of Lemma \ref{lemma:dual_values} is that the convergence rate given by Theorem 5 improves as $N$ increases.

\subsection{Convex and C-Lipshits}
\setcounter{utheorem}{0}
\begin{utheorem}[Convex and Lipschitz] \label{th:convex_lipschitz_2}
We assume that $\Omega$ is a convex set, and that for every $z \in \Z$, $\ell_z$ is convex and $C$-Lipschitz. Let $\wstar$ be a solution of (\ref{eq:main_problem}), and assume that we have perfect interpolation: $\forall z \in \Z, \: \ell_z(\wstar) = 0$. Then BORAT for $N\geq2$ applied to $f$ satisfies:
\begin{equation}
f\left(\tfrac{1}{T+ 1} \sum\limits_{t=0}^T \vw_t \right) - \fstar
   \leq  C\sqrt{\frac{\|\vw_0 - \wstar\|^2 }{(T+1)}} + \frac{\|\vw_0 - \wstar\|^2 }{\eta(T+1)}.
\end{equation}
\end{utheorem}
\begin{proof}
We start from (\ref{prog_dual}), hence we have:
\begin{align}
\| \w_{t+1} - \wstar \|^2 &\leq  \| \w_{t} - \wstar \|^2 -2\eta D(\bm{\alpha}_t)\label{start}
\end{align}
From Lemma \ref{lemma:dual_values} we additionally have that $D^2(\bm{\alpha}_*)\leq D^{N>2}(\bm{\alpha}_*)$ hence consider the $N=2$ as this provides the worse rate.
\begin{align}
\| \w_{t+1} - \wstar \|^2 &\leq  \| \w_{t} - \wstar \|^2 -2\eta D^2(\bm{\alpha}_t)\label{start}
\end{align}
For $N=2$ we have exactly two sub problems, and hence can write the dual in following compact form:
\begin{equation}\label{dualvaluen2}
 D^2(\bm{\alpha}_t) = \begin{cases}
-\frac{\eta}{2}\|\bm{g}_{z_t}\|^2 + \ell_{z_t}(\w_t),         &\text{if $ \eta\|\bm{g}_{z_t}\|^2 \leq  \ell_{z_t}(\w_t)$} \\
 \frac{1}{2\eta}\frac{\ell_{z_t}(\w_t)^2}{\|\bm{g}_{z_t}\|^2} &\text{if $ \eta\|\bm{g}_{z_t}\|^2 \geq  \ell_{z_t}(\w_t)$}
 \end{cases}
\end{equation}
We now introduce $\mathcal{I}_T$ and $\mathcal{J}_T$ as follows:
\begin{equation*}
\begin{split}
    \mathcal{I}_T &\triangleq \left\{ t \in \{0, ..., T\} : \eta\|\bm{g}_{z_t}\|^2 \geq  \ell_{z_t}(\w_t) \right\} \\
    \mathcal{J}_T &\triangleq \{0, ..., T\} \ \backslash \ \mathcal{I}_T
\end{split}
\end{equation*}
Defining $\mathds{1}$ to be the indicator function we can write:
\begin{align}
\| \w_{t+1} - \wstar \|^2 &\leq  \| \w_{t} - \wstar \|^2  +\mathds{1}{(t \in \mathcal{I}_T)} \eta \left(\eta\|\bm{g}_{z_t}\|^2 - 2\ell_{z_t}(\w_t)\right) -\mathds{1}{(t \in \mathcal{J}_T)} \left(  \frac{\ell_{z_t}(\w_t)^2}{\|\bm{g}_{z_t}\|^2} \right).
\end{align}
From our definition of $\mathcal{I}_T$ for all $t \in \mathcal{I}_T$ we have $\eta\|\bm{g}_{z_t}\|^2 \geq  \ell_{z_t}(\w_t)$ hence we can write:
\begin{align}
\| \w_{t+1} - \wstar \|^2 &\leq  \| \w_{t} - \wstar \|^2  -\mathds{1}{(t \in \mathcal{I}_T)} \eta \ell_{z_t}(\w_t) -\mathds{1}{(t \in \mathcal{J}_T)} \left(  \frac{\ell_{z_t}(\w_t)^2}{\|\bm{g}_{z_t}\|^2} \right).\label{start_10}
\end{align}
Summing over $T$:
\begin{align}
\| \w_{T+1} - \wstar \|^2
    &\leq \| \w_{0} - \wstar \|^2 - \eta\sum\limits_{t \in \mathcal{I}_T}  \ell_{z_t}(\w_t) - \sum\limits_{t \in \mathcal{J}_T} \left(  \frac{\ell_{z_t}(\w_t)^2}{\|\bm{g}_{z_t}\|^2} \right)\label{staring_point_2}
\end{align}
Using $\| \w_{T+1} - \wstar \|^2\geq 0$, we obtain:
\begin{align}\label{staring_point}
 \eta\sum\limits_{t \in \mathcal{I}_T} \ell_{z_t}(\w_t) + \sum\limits_{t \in \mathcal{J}_T} \left(  \frac{\ell_{z_t}(\w_t)^2}{\|\bm{g}_{z_t}\|^2} \right) &\leq \| \w_{0} - \wstar \|^2.
\end{align}
From $\left(\frac{\ell_{z_t}(\w_t)^2}{\|\bm{g}_{z_t}\|^2}\right)\geq0$, we get:
\begin{align}\label{convex_lip_sgd_steps}
\sum\limits_{t \in \mathcal{I}_T} \ell_{z_t}(\w_t) &\leq \frac{1}{\eta}\| \w_{0} - \wstar \|^2.
\end{align}
Likewise, using the observation that $\ell_{z}\geq 0$, we can write: 
\begin{align}
\sum\limits_{t \in \mathcal{J}_T}  \frac{\ell_{z_t}(\w_t)^2}{C^2}  \leq \sum\limits_{t \in \mathcal{J}_T}  \frac{\ell_{z_t}(\w_t)^2}{\|\bm{g}_{z_t}\|^2}  &\leq \| \w_{0} - \wstar \|^2.
\end{align}
Diving by $C^2$:
\begin{align}
\sum\limits_{t \in \mathcal{J}_T}  \ell_{z_t}(\w_t)^2  \leq C^2\| \w_{0} - \wstar \|^2.
\end{align}
Using the Cauchy-Schwarz inequality, we can further write:
\begin{align}
\left(\sum\limits_{t \in \mathcal{J}_T}  \ell_{z_t}(\w_t)\right)^2 \leq |\mathcal{J}_T| \sum\limits_{t \in \mathcal{J}_T}  \ell_{z_t}(\w_t)^2.
\end{align}
Therefore we have:
\begin{align}\label{convex_lip_alig_steps}
\sum\limits_{t \in \mathcal{J}_T}  \ell_{z_t}(\w_t) \leq C\sqrt{|\mathcal{J}_T| \| \w_{0} - \wstar \|^2}.
\end{align}
We can now put together inequalities (\ref{convex_lip_sgd_steps}) and (\ref{convex_lip_alig_steps}) by writing:
\begin{align}
\sum\limits_{t = 0}^T  \ell_{z_t}(\w_t) &= \sum\limits_{t \in \mathcal{J}_T}  \ell_{z_t}(\w_t) + \sum\limits_{t \in \mathcal{I}_T} \ell_{z_t}(\w_t) \\
\sum\limits_{t = 0}^T  \ell_{z_t}(\w_t) &\leq \frac{1}{\eta}\| \w_{0} - \wstar \|^2 + C\sqrt{|\mathcal{J}_T| \| \w_{0} - \wstar \|^2} \\
\sum\limits_{t = 0}^T  \ell_{z_t}(\w_t) &\leq \frac{1}{\eta}\| \w_{0} - \wstar \|^2 + C\sqrt{(T+1) \| \w_{0} - \wstar \|^2}
\end{align}
Dividing by $T + 1$ and taking the expectation over $z_t$, we finally get:
\begin{align}
 f\left(\frac{1}{T+1}\sum_{t=0}^T\w_t\right) - f^*&\leq \frac{1}{T+1}\sum_{t=0}^T f(\w_t) - f^*, \hspace{0.5cm}(f \text{ is convex}) \\
 f\left(\frac{1}{T+1}\sum_{t=0}^T\w_t\right) - f^*&\leq  C\sqrt{\frac{\|\vw_0 - \wstar\|^2 }{(T+1)}} + \frac{\|\vw_0 - \wstar\|^2 }{\eta(T+1)}. 
\end{align}
\end{proof}

\subsection{Convex and Smooth}

\begin{lemma}\label{lemma:smooth_bound} Let $z \in \mathcal{Z}$. Assume that $\ell_{z}$ is $\beta$-smooth and non-negative on $\mathbb{R}^d$. Then we have:
\begin{align*}
    \forall (\w) \in \mathbb{R}^d,\hspace{0.2cm} \ell_{z}(\w) \geq \frac{1}{2\beta}\|\nabla \ell_{z}(\w)\|^2 
\end{align*}
Note that we do not assume that $\ell_{z}$ is convex.
\end{lemma}
\begin{proof}
Let $\w\in\mathbb{R}^d$. By Lemma 3.4 of \citet{Bubeck2015}, we have:
\begin{align*}
    \forall \hspace{0.1cm} \bm{u} \in\mathbb{R}^d,\hspace{0.2cm} |\ell_{z}(\bm{u}) - \ell_{z}(\w) - \nabla \ell_{z}(\w)^\intercal(\bm{u}-\w)| \leq \frac{\beta}{2}\|\bm{u}-\w\|^2.
\end{align*}
Therefore we can write: 
\begin{align*}
    \forall \hspace{0.1cm} \bm{u} \in\mathbb{R}^d,\hspace{0.2cm}  \ell_{z}(\bm{u}) \leq \ell_{z}(\w) + \nabla \ell_{z}(\w)^\intercal(\bm{u}-\w)|  +\frac{\beta}{2}\|\bm{u}-\w\|^2.
\end{align*}
And since $\forall \hspace{0.1cm} \bm{u}, l_z(\bm{u})\geq 0$, we have:
\begin{align*}
    \forall \hspace{0.1cm} \bm{u} \in\mathbb{R}^d,\hspace{0.2cm} 0 \leq \ell_{z}(\w) + \nabla \ell_{z}(\w)^\intercal(\bm{u}-\w)|  +\frac{\beta}{2}\|\bm{u}-\w\|^2.
\end{align*}
We now choose $\bm{u} = -\frac{1}{\beta}\nabla l_z(\w)$, which yeilds:
\begin{align*}
    \forall \hspace{0.1cm} \bm{u} \in\mathbb{R}^d,\hspace{0.2cm} 0 \leq \ell_{z}(\w) - \frac{1}{\beta}\|\nabla \ell_{z}(\w)\|^2  +\frac{\beta}{2}\|\nabla \ell_{z}(\w)\|^2,
\end{align*}
which gives the desired result.
\end{proof}

\setcounter{utheorem}{2}
\begin{utheorem}[Convex and Smooth - Large Eta] \label{th:convex_smooth_large_eta}
We assume that $\Omega$ is a convex set, and that for every $z \in \Z$, $\ell_z$ is convex and $\beta$-smooth. Let $\wstar$ be a solution of (\ref{eq:main_problem}), and assume that we have prefect interpolation: $\forall z \in \Z, \: \ell_z(\wstar) = 0$. Then BORAT for $N\geq2$ applied to $f$ with $\eta\geq\frac{1}{2\beta}$ satisfies:
\begin{equation}
f\left(\tfrac{1}{T+ 1} \sum\limits_{t=0}^T \vw_t \right) - \fstar
   \leq  2\beta\frac{\| \w_{0} - \wstar \|^2}{T+1}.
\end{equation}
\end{utheorem}
\begin{proof}
We start from Equation (\ref{staring_point_2}) we have:
\begin{align}
\| \w_{T+1} - \wstar \|^2
    &\leq \| \w_{0} - \wstar \|^2 -\eta\sum\limits_{t \in \mathcal{I}_T}   \ell_{z_t}(\w_t) - \sum\limits_{t \in \mathcal{J}_T} \left(  \frac{\ell_{z_t}(\w_t)^2}{\|\bm{g}_{z_t}\|^2} \right).
\end{align}
Now using Lemma \ref{lemma:smooth_bound} we can write:
\begin{align}
\| \w_{T+1} - \wstar \|^2
    &\leq \| \w_{0} - \wstar \|^2  -\eta\sum\limits_{t \in \mathcal{I}_T}   \ell_{z_t}(\w_t)- \frac{1}{2\beta}\sum\limits_{t \in \mathcal{J}_T}  \ell_{z_t}(\w_t).\label{start_3}
\end{align}
From our assumption on $\eta \geq \frac{1}{2\beta}$ we can write:
\begin{align}
    \| \w_{T+1} - \wstar \|^2
    &\leq \| \w_{0} - \wstar \|^2  -\frac{1}{2\beta}\sum\limits_{t \in \mathcal{I}_T}   \ell_{z_t}(\w_t)- \frac{1}{2\beta}\sum\limits_{t \in \mathcal{J}_T}  \ell_{z_t}(\w_t),\\
        \| \w_{T+1} - \wstar \|^2
    &\leq \| \w_{0} - \wstar \|^2  -\frac{1}{2\beta}\sum\limits_{t = 0}^T   \ell_{z_t}(\w_t).\label{start_large_eta}
\end{align}
using $\| \w_{T+1} - \wstar \|^2\geq 0$, we obtain:
\begin{align}
  \sum\limits_{t = 0}^T   \ell_{z_t}(\w_t)  &\leq 2\beta\| \w_{0} - \wstar \|^2.
\end{align}
Dividing by $T + 1$ and taking the expectation over $z_t$, we finally get:
\begin{align*}
 f\left(\frac{1}{T+1}\sum_{t=0}^T\w_t\right) - f^*&\leq \frac{1}{T+1}\sum_{t=0}^T f(\w_t) - f^*, \hspace{0.5cm}(f \text{ is convex}) \\
 f\left(\frac{1}{T+1}\sum_{t=0}^T\w_t\right) - f^*&\leq 2\beta\frac{\| \w_{0} - \wstar \|^2}{T+1}. \\
\end{align*}
\end{proof}

\begin{utheorem}[Convex and Smooth - Small Eta] \label{th:convex_smooth_small_eta}
We assume that $\Omega$ is a convex set, and that for every $z \in \Z$, $\ell_z$ is convex and $\beta$-smooth. Let $\wstar$ be a solution of (\ref{eq:main_problem}), and assume that we have prefect interpolation: $\forall z \in \Z, \: \ell_z(\wstar) = 0$. Then BORAT for $N\geq2$ applied to $f$ with $\eta\leq\frac{1}{2\beta}$ satisfies:
\begin{equation}
f\left(\tfrac{1}{T+ 1} \sum\limits_{t=0}^T \vw_t \right) - \fstar
   \leq  2\beta\frac{\| \w_{0} - \wstar \|^2}{T+1}.
\end{equation}
\end{utheorem}
\begin{proof}
now considering the $\eta \leq \frac{1}{2\beta}$ starting from (\ref{start_3})
\begin{align}
\| \w_{T+1} - \wstar \|^2
    &\leq \| \w_{0} - \wstar \|^2  -\eta\sum\limits_{t \in \mathcal{I}_T}   \ell_{z_t}(\w_t)- \frac{1}{2\beta}\sum\limits_{t \in \mathcal{J}_T}  \ell_{z_t}(\w_t),\label{start_3}\\
    \| \w_{T+1} - \wstar \|^2
    &\leq \| \w_{0} - \wstar \|^2  -\eta\sum\limits_{t \in \mathcal{I}_T}   \ell_{z_t}(\w_t)- \eta\sum\limits_{t \in \mathcal{J}_T}  \ell_{z_t}(\w_t),\\
        \| \w_{T+1} - \wstar \|^2
    &\leq \| \w_{0} - \wstar \|^2  -\eta\sum\limits_{t = 0}^T   \ell_{z_t}(\w_t).
\end{align}
using $\| \w_{T+1} - \wstar \|^2\geq 0$, we obtain:
\begin{align}
  \sum\limits_{t = 0}^T   \ell_{z_t}(\w_t)  &\leq \frac{1}{\eta}\| \w_{0} - \wstar \|^2.
\end{align}
Dividing by $T + 1$ and taking the expectation over $z_t$, we finally get:
\begin{align*}
 f\left(\frac{1}{T+1}\sum_{t=0}^T\w_t\right) - f^*&\leq \frac{1}{T+1}\sum_{t=0}^T f(\w_t) - f^*, \hspace{0.5cm}(f \text{ is convex}) \\
 f\left(\frac{1}{T+1}\sum_{t=0}^T\w_t\right) - f^*&\leq \frac{\| \w_{0} - \wstar \|^2}{\eta(T+1)}. \\
\end{align*}
\end{proof}

\subsection{Strongly Convex}
Finally, we consider the $\alpha$-strongly convex and $\beta$-smooth case.

\begin{lemma} \label{lemma:strglycvx_gamma_bound}
    Let $z \in \Z$.
    Assume that $\ell_{z}$ is $\alpha$-strongly convex, non-negative on $\mathbb{R}^d$, and such that $ \inf \ell_{z} = 0$.
    Then we have:
    \begin{equation}
        \forall \: \w \in \mathbb{R}^d, \ \dfrac{\ell_{z}(\w)}{\| \nabla \ell_{z}(\w) \|^2} \leq \dfrac{1}{2 \alpha}.
    \end{equation}
\end{lemma}

\begin{proof} \
Let \(\w \in \mathbb{R}^d\) and suppose that $\ell_{z}$ reaches its minimum at $\wbar \in \mathbb{R}^d$ (this minimum exists because of strong convexity).
By definition of strong convexity, we have that:
\begin{equation}
    \forall \ \hat{\w} \in \mathbb{R}^d, \ \ell_{z}(\hat{\w}) \geq \ell_{z}(\w) + \nabla \ell_{z}(\w)^\top (\hat{\w} - \w) + \dfrac{\alpha}{2} \| \hat{\w} - \w \|^2
\end{equation}
We minimize the right hand-side over $\hat{\w}$, which gives:
\begin{equation}
\begin{split}
\forall \hat{\w} \in \mathbb{R}^d, \ \ell_{z}(\hat{\w})
    &\geq \ell_{z}(\w) + \nabla \ell_{z}(\w)^\top (\hat{\w} - \w) + \dfrac{\alpha}{2} \| \hat{\w} - \w \|^2 \\
    &\geq \ell_{z}(\w)  - \dfrac{1}{2 \alpha} \| \nabla \ell_{z}(\w) \|^2
\end{split}
\end{equation}
Thus by choosing $\hat{\w} = \wbar$ and re-ordering, we obtain the following result (a.k.a. the Polyak-Lojasiewicz inequality):
\begin{equation}
    \ell_{z}(\w) - \ell_{z}(\wbar) \leq \dfrac{1}{2 \alpha} \| \nabla \ell_{z}(\w) \|^2
\end{equation}
However we have $\ell_{z}(\wbar)$ = 0 which concludes the proof.
\end{proof}

\begin{lemma} \label{lemma:parallelogram_inequality}
For any $a, b \in \mathbb{R}^d$, we have that:
\begin{equation}
    \|a \|^2 + \|b \|^2 \geq \dfrac{1}{2} \| a - b\|^2.
\end{equation}
\end{lemma}
\vspace{0.2cm}
\begin{proof}
This is a simple application of the parallelogram law, but we give the proof here for completeness.
\begin{align*}
\|a \|^2 + \|b \|^2 - \dfrac{1}{2} \| a - b\|^2
    &= \|a \|^2 + \|b \|^2 - \dfrac{1}{2} \| a\|^2 -\dfrac{1}{2} \| b\|^2 + a^\top b \\
    &= \dfrac{1}{2} \| a\|^2 + \dfrac{1}{2} \| b\|^2 + a^\top b \\
    &= \dfrac{1}{2} \| a + b \|^2 \\
    &\geq 0 \\
\end{align*}
\end{proof}

\begin{lemma} \label{lemma:strglycvx_fun_bound}
    Let $z \in \Z$.
    Assume that $\ell_{z}$ is $\alpha$-strongly convex and achieves its (possibly constrained) minimum at $\wstar \in \Omega$.
    Then we have:
    \begin{equation}
        \forall \: \w \in \Omega, \ \ell_{z}(\w)  - \ell_{z}(\wstar) \geq \dfrac{\alpha}{2} \| \w - \wstar \|^2
    \end{equation}
\end{lemma}

\begin{proof}
By definition of strong-convexity \cite{Bubeck2015}, we have:
\begin{equation}
\forall \: \w \in \Omega, \: \ell_z(\w) - \ell_z(\wstar) - \nabla \ell_z(\wstar)^\top (\w - \wstar) \geq \dfrac{\alpha}{2} \| \w - \wstar \|^2.
\end{equation}
In addition, since $\wstar$ minimizes $\ell_z$, then necessarily:
\begin{equation}
    \forall \: \w \in \Omega, \: \nabla \ell_z(\wstar)^\top (\w - \wstar) \geq 0.
\end{equation}
Combining the two equations gives the desired result.
\end{proof}
Finally, we consider the $\alpha$-strongly convex and $\beta$-smooth case. Again, our proof yields a natural separation between$\eta\geq\frac{1}{2\beta}$ and $\eta\leq\frac{1}{2\beta}$.

\begin{utheorem}[Strongly Convex - Large Eta] \label{th:alig_cvx}
We assume that $\Omega$ is a convex set, and that for every $z \in \Z$, $\ell_z$ is $\alpha$-strongly convex and $\beta$-smooth. Let $\wstar$ be a solution of (\ref{eq:main_problem}), and assume that we have prefect interpolation: $\forall z \in \Z, \: \ell_z(\wstar) = 0$. Then BORAT for $N\geq2$ and applied to $f$ with a $\eta\geq\frac{1}{2\beta}$satisfies:
\begin{equation}
\mathbb{E}[f(\w_{t+1})] - f(\wstar) \leq \frac{\beta}{2} \exp\left(-\frac{\alpha t}{4\beta}\right)\| \w_{0} - \wstar \|^2.
\end{equation}
\end{utheorem}
\begin{proof}
We start from (\ref{start_10}):
\begin{align}
\| \w_{t+1} - \wstar \|^2 &\leq  \| \w_{t} - \wstar \|^2  -\mathds{1}{(t \in \mathcal{I}_T)} \eta \ell_{z_t}(\w_t) -\mathds{1}{(t \in \mathcal{J}_T)} \left(  \frac{\ell_{z_t}(\w_t)^2}{\|\bm{g}_{z_t}\|^2} \right).
\end{align}
We next use Lemma \ref{lemma:smooth_bound} write:
\begin{align}
\| \w_{t+1} - \wstar \|^2 &\leq  \| \w_{t} - \wstar \|^2  -\mathds{1}{(t \in \mathcal{I}_T)} \eta \ell_{z_t}(\w_t) -\mathds{1}{(t \in \mathcal{J}_T)} \frac{1}{2\beta}\\ell_{z_t}(\w_t).
\end{align}
Now we use our assumption on $\eta$ to give:
\begin{align}
\| \w_{t+1} - \wstar \|^2 &\leq  \| \w_{t} - \wstar \|^2  - \frac{1}{2\beta} \ell_{z_t}(\w_t) 
\end{align}
taking expectations: 
\begin{align}
\mathbb{E}[\| \w_{t+1} - \wstar \|^2] &\leq  \| \w_{t} - \wstar \|^2  - \frac{1}{2\beta} f(\w_t) 
\end{align}
Using Lemma \ref{th:alig_cvx}
\begin{align}
\mathbb{E}[\| \w_{t+1} - \wstar \|^2] &\leq  \| \w_{t} - \wstar \|^2  - \frac{\alpha}{4\beta} \| \w_{t} - \wstar \|^2
\end{align}
We use a trivial induction over $t$ and write:
\begin{align}
\mathbb{E}[\| \w_{t+1} - \wstar \|^2]  &\leq  \left(1-\frac{\alpha}{4\beta}\right)\| \w_{t} - \wstar \|^2,\\
\mathbb{E}[\| \w_{t+1} - \wstar \|^2]  &\leq  \left(1-\frac{\alpha}{4\beta}\right)^t\| \w_{0} - \wstar \|^2,\\  
\end{align}
Given an arbitrary $\w \in \mathbb{R}^d$, we now wish to relate the distance $\| \w - \wstar \|^2$ to the function values $f(\w) - f(\wstar)$. From smoothness, we have:
\begin{align}
f(\w{t+1}) - f(\wstar) \leq \nabla f(\wstar)^\top (\w{t+1} -\wstar) + \frac{\beta}{2}\|\w{t+1} -\wstar\|^2.
\end{align}
However we know by definition $ \nabla f(\wstar) = 0$ hence: 
\begin{align}
f(\w_{t+1}) - f(\wstar) \leq \frac{\beta}{2}\|\w_{t+1} -\wstar\|^2.
\end{align}
Taking expectations: 
\begin{align}
\mathbb{E}[f(\w_{t+1})] - f(\wstar) \leq \frac{\beta}{2}\mathbb{E}[\|\w_{t+1} -\wstar\|^2].
\end{align}
Hence we recover the final rate:
\begin{align}
\mathbb{E}[f(\w_{t+1})] - f(\wstar) \leq \frac{\beta}{2} \left(1-\frac{\alpha}{4\beta}\right)^t\| \w_{0} - \wstar \|^2,\\
\mathbb{E}[f(\w_{t+1})] - f(\wstar) \leq \frac{\beta}{2} \exp\left(-\frac{\alpha t}{4\beta}\right)\| \w_{0} - \wstar \|^2.
\end{align}
\end{proof}

\begin{utheorem}[Strongly Convex - Small Eta] \label{th:alig_cvx}
We assume that $\Omega$ is a convex set, and that for every $z \in \Z$, $\ell_z$ is $\alpha$-strongly convex and $\beta$-smooth. Let $\wstar$ be a solution of (\ref{eq:main_problem}), and assume that we have prefect interpolation: $\forall z \in \Z, \: \ell_z(\wstar) = 0$. Then BORAT for $N\geq2$ and applied to $f$ with a $\eta\leq\frac{1}{2\beta}$satisfies:
\begin{equation}
\mathbb{E}[f(\w_{t+1})] - f(\wstar) \leq \frac{\beta}{2} \exp\left(-\frac{\alpha \eta t}{2}\right)\| \w_{0} - \wstar \|^2.
\end{equation}
\end{utheorem}
\begin{proof}
We start from (\ref{start_10}):
\begin{align}
\| \w_{t+1} - \wstar \|^2 &\leq  \| \w_{t} - \wstar \|^2  -\mathds{1}{(t \in \mathcal{I}_T)} \eta \ell_{z_t}(\w_t) -\mathds{1}{(t \in \mathcal{J}_T)} \left(  \frac{\ell_{z_t}(\w_t)^2}{\|\bm{g}_{z_t}\|^2} \right).
\end{align}
We next use Lemma \ref{lemma:smooth_bound} write:
\begin{align}
\| \w_{t+1} - \wstar \|^2 &\leq  \| \w_{t} - \wstar \|^2  -\mathds{1}{(t \in \mathcal{I}_T)} \eta \ell_{z_t}(\w_t) -\mathds{1}{(t \in \mathcal{J}_T)} \frac{1}{2\beta}\ell_{z_t}(\w_t).
\end{align}
Now we use our assumption on $\eta$ to give:
\begin{align}
\| \w_{t+1} - \wstar \|^2 &\leq  \| \w_{t} - \wstar \|^2  - \eta \ell_{z_t}(\w_t) 
\end{align}
taking expectations: 
\begin{align}
\mathbb{E}[\| \w_{t+1} - \wstar \|^2] &\leq  \| \w_{t} - \wstar \|^2  - \eta \ell_{z_t}(\w_t) 
\end{align}
Using Lemma \ref{th:alig_cvx}
\begin{align}
\mathbb{E}[\| \w_{t+1} - \wstar \|^2] &\leq  \| \w_{t} - \wstar \|^2  - \frac{\alpha\eta}{2} \| \w_{t} - \wstar \|^2
\end{align}
We use a trivial induction over $t$ and write:
\begin{align}
\mathbb{E}[\| \w_{t+1} - \wstar \|^2]  &\leq  \left(1-\frac{\alpha\eta}{2}\right)\| \w_{t} - \wstar \|^2,\\
\mathbb{E}[\| \w_{t+1} - \wstar \|^2]  &\leq  \left(1-\frac{\alpha\eta}{2}\right)^t\| \w_{0} - \wstar \|^2,\\  
\end{align}
Given an arbitrary $\w \in \mathbb{R}^d$, we now wish to relate the distance $\| \w - \wstar \|^2$ to the function values $f(\w) - f(\wstar)$. From smoothness, we have:
\begin{align}
f(\w{t+1}) - f(\wstar) \leq \nabla f(\wstar)^\top (\w{t+1} -\wstar) + \frac{\beta}{2}\|\w{t+1} -\wstar\|^2.
\end{align}
However we know by definition $ \nabla f(\wstar) = 0$ hence: 
\begin{align}
f(\w_{t+1}) - f(\wstar) \leq \frac{\beta}{2}\|\w_{t+1} -\wstar\|^2.
\end{align}
Taking expectations: 
\begin{align}
\mathbb{E}[f(\w_{t+1})] - f(\wstar) \leq \frac{\beta}{2}\mathbb{E}[\|\w_{t+1} -\wstar\|^2].
\end{align}
Hence we recover the final rate:
\begin{align}
\mathbb{E}[f(\w_{t+1})] - f(\wstar) \leq \frac{\beta}{2} \left(1-\frac{\alpha\eta}{2}\right)^t\| \w_{0} - \wstar \|^2,\\
\mathbb{E}[f(\w_{t+1})] - f(\wstar) \leq \frac{\beta}{2} \text{exp}\left(\frac{-\alpha \eta t}{2}\right)\| \w_{0} - \wstar \|^2.
\end{align}
\end{proof}

\section{None Convex Results}\label{App:None_Convex_Results}

Here we provide the proof of theorem \ref{th:rsi}, which we restate for clarity. To simplify our analysis, we consider the BORAT algorithm with $N=3$. We prove these result for BORAT with the minor modification, that is, all linear approximations are formed using the same mini-batch of data, $\ell_{z_{t}^n}$ = $\ell_{z_{t}}$ for all $n\in\{2,...,N-1\}$.
\setcounter{utheorem}{1}
\begin{utheorem}[RSI]\label{th:rsi}
We consider problems of type (1), see main paper. We assume $l_z$ satisfies RSI with constant $\mu$, smoothness constant $\beta$ and perfect interpolation e.g. $l_z(\w^*)=0, \hspace{0.1cm}\forall z \in\mathcal{Z}$. Then if set  $\eta \leq \hat{\eta} = \min\{\frac{1}{4\beta},\frac{1}{4\mu}, \frac{\mu}{2\beta^2}\}$ then in the worst case we have:
\begin{align}
f(\w_{T+1}) - f^*&\leq \text{exp}\left( \left(-\frac{3}{8}\hat{\eta}\mu \right)T\right)||\w_0 - \w^*||^2. 
\end{align}
\end{utheorem}

In order to derive the proof for Theorem \ref{th:rsi} we first give a brief overview of BORAT with $N=3$. We detail the $(2^N-1)$ possible subproblems (7 in this case), and the resulting values of $\bm{\alpha}_t$ for each. We show for $\eta \leq \frac{1}{2\beta}$, only a sub-set of the subproblems result in valid solutions with optimal points within the simplex. Finally we derive a rate assuming that for the each of the remaining subproblems is optimal for all $t$. Lastly by taking the minimum of these rates we construct the worst case rate.

\subsubsection{BORAT $(N=3)$}\label{N=3}
With $(N=3)$ the bundle is made of three linear approximations. These are, the lower bound on the loss and linear approximations made at the current point and at the optimum of the bundle of size two. Hence this third linear approximation is made at the location one would reach after taking a ALI-G step. Note, here we use $\gamma_t$ to denote the optimal value of $\alpha^1$ as given by \ref{alig_alpha}. As some of this proofs get quite involved, we make use of the following abbreviations to save space:
\begin{align*}
\hat{\w}_{t}^1 &= \w_t, \hspace{1cm}\bm{g}_{z_t} = \nabla \ell_{z_t^n}(\w_t), \hspace{1cm}b_{t}^1 = \ell_{z_t^n}(\w_t), \\
\hat{\w}_{t}^2 &= \w_t', \hspace{1cm}\bm{g}_{z_t}' = \nabla \ell_{z_t^n}(\w_t'), \hspace{1cm}b_{t}^2 = \ell_{z_t^n}(\w_t') + \eta \gamma_t \langle \bm{g}_{z_t}, \bm{g}_{z_t}' \rangle.
\end{align*}
Where $\w_{t}'$ is defined as $\w_{t}' = \w_t - \eta \gamma_t \nabla \ell_{z_t^n}(\w_t)$, where $\gamma_t = \min\{1, \frac{\ell_{z_t^n}(\w_t)}{\eta\|\nabla \ell_{z_t^n}(\w_t)\|^2}\eta$. With this notation defined, the BORAT primal problem for this special case can be simplified to:
\begin{align}\label{n=3primal}
\w_{t+1} = \argmin_{\w \in \mathbb{R}^d}\left\{ \frac{1}{2\eta}||\w - \w_{t}||^2 + \max\left\{
\langle\bm{g}_{z_t},\w - \w_{t}\rangle + b_{t}^1,   
\langle \bm{g}_{z_t}',\w - \w_{t} \rangle + b_{t}^2, 
0  \right\}\right\}.
\end{align}
The dual of (\ref{n=3primal}) be written an:
\begin{align}\label{n=3dual}
\bm{\alpha}_{t} = \argmax_{\bm{\alpha} \in \Delta_3}\left\{ -\frac{\eta}{2} ||\alpha^1 \bm{g}_{z_t} + \alpha^{2t} \bm{g}_{z_t}' ||^2 + \alpha^1 \ell_{z_t^n}(\w_t) + \alpha^{2}\ell_{z_t^n}(\w_t') + \alpha^{2} \eta \gamma_t\langle \bm{g}_{z_t},\bm{g}_{z_t}'\rangle\right\}.
\end{align}
For $N=3$ we have the following parameter update:
\begin{align}\label{n=3update}
 \w_{t+1} = \w_t - \alpha^{1t} \eta  \nabla \ell_{z_t^n}(\w_t) - \alpha^{2t} \eta  \nabla \ell_{z_t^n}(\w_t'),
\end{align}
\subsubsection{subproblems}
Algorithm \ref{alg:dualsol} solves $(2^N-1)$ sub problem. Each of these linear systems corresponds to a loci of the simplex, defining the feasible set. We refer to each of these $(2^N-1)$ options as subproblems. However, each subproblem can also be interpreted as a sub-system of $Q\bm{\alpha}=\bm{b}$, (see line 2 of Algorithm \ref{alg:dualsol} for definitions). For ($N=3$) the form of $Q\bm{\alpha}=\bm{b}$ is detailed in (\ref{Qx=b}). 
\begin{gather}\label{Qx=b}
\begin{bmatrix}
   \eta||\bm{g}_{z_t}||^2 &  \eta\langle\bm{g}_{z_t},\bm{g}_{z_t}'\rangle & 0 & 1\\
    \eta\langle\bm{g}_{z_t},\bm{g}_{z_t}'\rangle &  \eta||\bm{g}_{z_t}'||^2 &  0 & 1\\
   0     & 0   &  0 &  1     \\
   1          & 1          &  1 &  0          \\ 
   \end{bmatrix}
 \begin{bmatrix} \alpha^{1} \\ \alpha^{2} \\ \alpha^{3} \\ c \end{bmatrix}
 =
 \begin{bmatrix} b_1 \\ b_2 \\ 0 \\ 1 \end{bmatrix},
\end{gather}
where $c$ is defined in Theorem \ref{th:simplex}. Each subproblem is defined by setting a unique subset of the dual variables $\alpha^n$ to zero, before solving for the remaining variables. For ($N=3$) we have exactly seven subproblems, which we each give a unique label, see table \ref{steptypes3}. For clarity, we detail a few specific subproblems. The SGD subproblem corresponds to setting $\alpha_2,\alpha_3=0$, and recovers the SGD update. Likewise, the ESGD step corresponds to setting $\alpha^1,\alpha_3=0$ and recovers an update similar to the extra gradient method. Finally, by setting $\alpha_2=0$ before solving the system we recover the a ALI-G like step, hence we label this subproblem as the ALI-G step. If a subproblem results in a $\bm{\alpha}\in\Delta_3$ we refer to that subproblem as feasible, conversely, if it results in a $\bm{\alpha}\notin\Delta_3$ we refer to that subproblem as infeasible. Algorithm 2 computes the dual value for all feasible subproblems and selects the largest. This subproblem's $\bm{\alpha}$ is then used in (\ref{n=3update}) to update the parameters. The closed form solutions for $\bm{\alpha}$ for each of the 7 subproblems are listed in table \ref{steptypes3}. We use a subscript to show that $\bm{\alpha}$ belongs to a certain subproblem. For example $\bm{\alpha}_{SGD} = [1,0,0]$.
\begin{table}[H]
  \caption{subproblems for $N=3$.}
  \label{steptypes3}
  \centering
 \begin{tabular}{c c c c c} 
 
 $\alpha^1$ & $\alpha^{2}$ & $\alpha^3$& label\\ [0.5ex] 
 \hline
  \hline
 1 & 0 & 0 &  SGD\\ 
 \hline
 0 & 1 & 0 &  SEGD\\
 \hline
 0 & 0 & 1  &ZERO\\
 \hline
  0 & $\frac{b_2}{\eta||\bm{g}_{z_t}'||^2}$  & $1- \frac{b_2}{\eta||\bm{g}_{z_t}'||^2}$  &EALIG \\
 \hline
 $\frac{b_1}{\eta||\bm{g}_{z_t}||^2}$ & 0 & $1- \frac{b_1}{\eta||\bm{g}_{z_t}||^2}$ & ALIG  \\
 \hline
 $\frac{\eta||\bm{g}_{z_t}'||^2 - 2\eta \langle\bm{g}_{z_t}, \bm{g}_{z_t}'\rangle + b_1 - b_2}{\eta||\bm{g}_{z_t} - \bm{g}_{z_t}'||^2}$ & $\frac{\eta||\bm{g}_{z_t}||^2 + b_2 - b_1}{\eta||\bm{g}_{z_t}-\bm{g}_{z_t}'||^2}$ & 0 &MAX2 \\
 \hline
  $\frac{b_1||\bm{g}_{z_t}'||^2 - b_2 \bm{g}_{z_t}^\top \bm{g}_{z_t}'}{\eta||\bm{g}_{z_t}||^2||\bm{g}_{z_t}'||^2-\eta||\bm{g}_{z_t}\bm{g}_{z_t}'||^2}$ & $\frac{b_2||\bm{g}_{z_t}||^2 - b_1 \bm{g}_{z_t}^\top a_2}{\eta||\bm{g}_{z_t}||^2||\bm{g}_{z_t}'||^2-\eta||\bm{g}_{z_t}\bm{g}_{z_t}'||^2}$ &$ 1 - \alpha^1 - \alpha^{2} $& MAX3 \\
 \hline
\end{tabular}
\end{table}

\subsubsection{Dual Values}\label{dual_values}
The corresponding expressions for the dual values for the seven different subproblems are detailed below:
\begin{align*}
D_{ZERO}(\bm{\alpha}) &= 0,\\
D_{SGD}(\bm{\alpha}) &= -\frac{\eta}{2}||\bm{g}_{z_t}||^2 + \ell_{z_t^n}(\w_t),\\
D_{ESGD}(\bm{\alpha}) &= -\frac{\eta}{2}||\bm{g}_{z_t}'||^2 + \ell_{z_t^n}(\w_t')+ \eta\gamma_t\langle \bm{g}_{z_t}, \bm{g}_{z_t}' \rangle,\\
D_{ALIG}(\bm{\alpha}) &= \frac{1}{2\eta}\frac{\ell_{z_t^n}(\w_t)^2}{||\bm{g}_{z_t}||^2},\\
D_{EALIG}(\bm{\alpha}) &= \frac{1}{2\eta}\frac{\left(\ell_{z_t^n}(\w_t')+ \eta\gamma_t\langle \bm{g}_{z_t}, \bm{g}_{z_t}' \rangle\right)^2}{||\bm{g}_{z_t}'||^2} ,\\
D_{MAX2}(\bm{\alpha}) &= \frac{1}{2\eta||\bm{g}_{z_t}-\bm{g}_{z_t}'||^2}
\bigg((\ell_{z_t^n}(\w_t')-\ell_{z_t^n}(\w_t))^2+2\eta\left( \ell_{z_t^n}(\w_t')|| \bm{g}_{z_t}||^2 
+ \ell_{z_t^n}(\w_t)|| \bm{g}_{z_t}'||^2\right)
\\
&- 4\eta \ell_{z_t^n}(\w_t)\langle \bm{g}_{z_t}, \bm{g}_{z_t}' \rangle 
+ 2\eta^2|| \bm{g}_{z_t}||^2\langle \bm{g}_{z_t}, \bm{g}_{z_t}' \rangle
-\eta^2|| \bm{g}_{z_t}||^2|| \bm{g}_{z_t}'||^2\bigg)
,\\
D_{MAX3}(\bm{\alpha}) &= \frac{1}{2}\frac{ \left(\ell_{z_t^n}(\w_t')\bm{g}_{z_t} + \eta\gamma_t\langle \bm{g}_{z_t}, \bm{g}_{z_t}'\rangle \bm{g}_{z_t}  - \ell_{z_t^n}(\w_t)\bm{g}_{z_t}'\right)^2}{\eta|| \bm{g}_{z_t}||^2|| \bm{g}_{z_t}'||^2 - \eta\langle \bm{g}_{z_t}, \bm{g}_{z_t}'\rangle^2}.\\
\end{align*}

Due to spatial constraints we state these results without proof. The dual value for each subproblem is simply derived by inserting the closed form expression for $\bm{\alpha}$ for each subproblem detailed in table \ref{steptypes3} into (\ref{dual}). These dual values observe a tree like hierarchy,
\begin{align*}
    D_{MAX3} &\geq D_{MAX2}, D_{ALIG}, D_{EALIG},\\
    D_{MAX2} &\geq D_{SGD}, D_{ESGD},\\
    D_{ALIG} &\geq D_{SGD}, D_{ZERO},\\
    D_{EALIG}&\geq D_{ESGD}, D_{ZERO}.
\end{align*}
This hierarchy is a consequence of the subproblems maximising the dual over progressively larger spaces, $\Delta^{n-1} \subset \Delta^n$.

\subsubsection{Feasible Subproblems}
To give a worst case rate on the convergence of BORAT with $N=3$, we first prove for smooth functions and small $\eta$, only certain subproblems will be feasible. We start with a useful lemma, before proving this result.
\begin{lemma}\label{lemma:ggprimegeq0} Let $z \in \mathcal{Z}$. Assume that $\ell_{z}$ is $\beta$-smooth. If we define $\w' = \w - \eta\nabla \ell_{z}(\w)$ and $\eta \leq \frac{1}{\beta}$ then we have:
\begin{align*}
    \forall (\w) \in \mathbb{R}^d,\hspace{0.2cm} \langle \nabla \ell_{z}(\w),\nabla \ell_{z}(\w') \rangle \geq 0.
\end{align*}
Note that we do not assume that $\ell_{z}$ is convex.
\end{lemma}
\begin{proof}
\begin{align*}
  2\langle  \nabla \ell_{z}(\w), \nabla \ell_{z}(\w')\rangle &= - ||\nabla \ell_{z}(\w) - \nabla \ell_{z}(\w')||^2 + ||\nabla \ell_{z}(\w)||^2 + ||\nabla \ell_{z}(\w')||^2 \\  2\langle  \nabla \ell_{z}(\w), \nabla \ell_{z}(\w')\rangle &\geq -\beta^2||\w - \w'||^2 + ||\nabla \ell_{z}(\w)||^2 + ||\nabla \ell_{z}(\w')||^2, \hspace{0.5cm}(\text{smoothness})\\
  2\langle  \nabla \ell_{z}(\w), \nabla \ell_{z}(\w')\rangle &\geq -\beta^2\eta^2||\nabla \ell_{z}(\w)||^2 + ||\nabla \ell_{z}(\w)||^2 + ||\nabla \ell_{z}(\w')||^2,\hspace{0.5cm} (\w'\text{ definition})\\
   \langle  \nabla \ell_{z}(\w), \nabla \ell_{z}(\w')\rangle &\geq \frac{1}{2}(1-\beta^2\eta^2)||\nabla \ell_{z}(\w)||^2 + \frac{1}{2}||\nabla \ell_{z}(\w')||^2,\\
   \langle  \nabla \ell_{z}(\w), \nabla \ell_{z}(\w')\rangle &\geq 0. \hspace{0.5cm}(\frac{1}{\beta} \geq \eta)
\end{align*}
\end{proof}
\begin{lemma}[Feasible Subproblems]\label{max_3_step}
If $\ell_{z_t^n}$ has smoothness constant $\beta$ and we set $\eta \leq \frac{1}{2\beta}$ for the BORAT algorithm with ($N=3$) detailed in Section \ref{N=3}, the ALIG, EALIG and MAX3 subproblems will always be infeasible.
\end{lemma}
\begin{proof}
We start by showing the ALIG step is infeasible. 
From Lemma \ref{lemma:smooth_bound} we have:
\begin{align*}
\frac{\ell_{z}(\w)}{||\bm{g}_{z_t}||^2 } \geq \frac{1}{2\beta}.
\end{align*}
For the ALI-G step to be feasible we require $\alpha^{3}_{ALIG} > 0$ or $1 - \frac{\ell_{z_t^n}(\w_t)}{\eta||\bm{g}_{z_t}||^2} > 0$. Rearranging, we have:
\begin{align*}
    \eta > \frac{\ell_{z_t^n}(\w_t)}{||\bm{g}_{z_t}||^2}.
\end{align*}
Combining these two inequalities gives:
\begin{align*}
\eta > \frac{\ell_{z}(\w)}{||\bm{g}_{z_t}||^2 } \geq \frac{1}{2\beta}.
\end{align*}
Hence for any $\frac{1}{2\beta} \geq \eta$, $\eta > \frac{\ell_{z_t^n}(\w_t)}{||\bm{g}_{z_t}||^2}$ cannot hold.
We now use a similar argument to show that the EALIG subproblem is infeasible. For EALIG to be feasible we require $\alpha^{3t}_{EALIG} > 0$, plugging in the closed form solution for $\alpha^{3}_{EALIG}$ gives:
\begin{align*}
    \eta > \frac{\ell_{z_t^n}(\w_t')+\eta\gamma_t\langle \bm{g}_{z_t},  \bm{g}_{z_t}'\rangle}{||\bm{g}_{z_t}'||^2} =\frac{\ell_{z_t^n}(\w_t')+\eta\langle \bm{g}_{z_t},  \bm{g}_{z_t}'\rangle}{||\bm{g}_{z_t}'||^2} = \frac{\ell_{z_t^n}(\w_t')}{||\bm{g}_{z_t}'||^2} + \frac{\eta\langle \bm{g}_{z_t},  \bm{g}_{z_t}'\rangle}{||\bm{g}_{z_t}'||^2} \geq \frac{\ell_{z_t^n}(\w_t')}{||\bm{g}_{z_t}'||^2} \geq \frac{1}{2\beta},
\end{align*}
where the penultimate inequality makes use of $\langle \bm{g}_{z_t}, \bm{g}_{z_t}'\rangle\geq0$, which is a direct application of Lemma \ref{lemma:ggprimegeq0}. The final inequality is a direct application of Lemma \ref{lemma:smooth_bound}. Again, it is clear that the condition $\eta >\frac{\ell_{z_t^n}(\w_t')}{||\bm{g}_{z_t}'||^2}$ cannot be satisfied for $\eta\leq\frac{1}{2\beta}$.
\\
We show that the $MAX3$ step is never taken for $\eta\leq \frac{1}{2\beta}$. First, we show that the dual value for the $MAX3$ step can be written as $D_{MAX3}(\bm{\alpha}) = \frac{1}{2}\left( \ell_{z_t^n}(\w_t)\alpha^{1t} +   \ell_{z_t^n}(\w_t')\alpha^{2t} + \eta\gamma_t \alpha^{2t}\langle \bm{g}_{z_t}, \bm{g}_{z_t}'\rangle\right)$. We start from the dual value stated in Section \ref{dual_values}.
\begin{align*}
D_{MAX3}(\bm{\alpha}) &= \frac{1}{2}\frac{ \left(\ell_{z_t^n}(\w_t')\bm{g}_{z_t} + \eta\gamma_t \langle \bm{g}_{z_t}, \bm{g}_{z_t}'\rangle \bm{g}_{z_t}  - \ell_{z_t^n}(\w_t)\bm{g}_{z_t}'\right)^2}{\eta|| \bm{g}_{z_t}||^2|| \bm{g}_{z_t}'||^2 - \eta\langle \bm{g}_{z_t}, \bm{g}_{z_t}'\rangle^2},
\end{align*}
expanding,
\begin{align*}
 D_{MAX3}(\bm{\alpha}) &= 
 \frac{1}{2}\ell_{z_t^n}(\w_t')\underbrace{\frac{ \left(\ell_{z_t^n}(\w_t')\bm{g}_{z_t} + \eta\gamma_t \langle \bm{g}_{z_t}, \bm{g}_{z_t}'\rangle \bm{g}_{z_t}  - \ell_{z_t^n}(\w_t)\bm{g}_{z_t}'\right)}{\eta|| \bm{g}_{z_t}||^2|| \bm{g}_{z_t}'||^2 - \eta\langle \bm{g}_{z_t}, \bm{g}_{z_t}'\rangle^2}\bm{g}_{z_t}}_{=\alpha^{2}}\\
 &+\frac{1}{2}\eta\gamma_t \langle \bm{g}_{z_t}, \bm{g}_{z_t}'\rangle \underbrace{\frac{ \left(\ell_{z_t^n}(\w_t')\bm{g}_{z_t} + \eta\gamma_t \langle \bm{g}_{z_t}, \bm{g}_{z_t}'\rangle \bm{g}_{z_t}  - \ell_{z_t^n}(\w_t)\bm{g}_{z_t}'\right)}{\eta|| \bm{g}_{z_t}||^2|| \bm{g}_{z_t}'||^2 - \eta\langle \bm{g}_{z_t}, \bm{g}_{z_t}'\rangle^2}\bm{g}_{z_t}}_{=\alpha^{2}}\\
 & - \frac{1}{2}\ell_{z_t^n}(\w_t)\underbrace{\frac{ \left(\ell_{z_t^n}(\w_t')\bm{g}_{z_t} + \eta\gamma_t \langle \bm{g}_{z_t}, \bm{g}_{z_t}'\rangle \bm{g}_{z_t}  - \ell_{z_t^n}(\w_t)\bm{g}_{z_t}'\right)}{\eta|| \bm{g}_{z_t}||^2|| \bm{g}_{z_t}'||^2 - \eta\langle \bm{g}_{z_t}, \bm{g}_{z_t}'\rangle^2}\bm{g}_{z_t}'.}_{=-\alpha^{1}}
\end{align*}
Using the definitions of $\alpha^{1}_{MAX3}$ and $\alpha^{2}_{MAX3}$ we recover the following expression for the $MAX3$ subproblem's dual function:
\begin{align*}
D_{MAX3}(\bm{\alpha}) &= \frac{1}{2}\left( \ell_{z_t^n}(\w_t)\alpha^{1t} +   \ell_{z_t^n}(\w_t')\alpha^{2t} + \eta \gamma_t\alpha^{2t}\langle \bm{g}_{z_t}, \bm{g}_{z_t}'\rangle\right).
\end{align*}
With the dual function in this form it is easy to upper bound the feasible dual value for the $MAX3$ subproblem as $D_{MAX3} \leq \frac{1}{2}\max \left\{\ell_{z_t^n}(\w_t), \ell_{z_t^n}(\w_t') + \eta\gamma_t\langle \bm{g}_{z_t}, \bm{g}_{z_t}'\rangle\right\}$. This is a consequence of the fact that $\bm{\alpha} \in \Delta$ must hold for feasible steps. However, from the hierarchy of dual values we have the lower bounds $D_{MAX3} \geq D_{SGD}$ and $D_{MAX3} \geq D_{ESGD}$, on the $MAX3$ dual value, see Section \ref{dual_values}. If either of these lower bounds have larger value than the feasible dual value's upper bound, the $MAX3$ step will not be feasible. We now show that this is always the case for $\eta\leq\frac{1}{2\beta}$. In order to do this we consider two cases.

We first assume $\ell_{z_t^n}(\w_t) \geq \ell_{z_t^n}(\w_t') + \eta\gamma_t\langle \bm{g}_{z_t}, \bm{g}_{z_t}'\rangle$. Hence we know the maximum feasible value for $D^{MAX3} = \frac{1}{2}\ell_{z_t^n}(\w_t)$, if either $D^{SGD}$ or $D^{ESGD}$ have larger dual value we can conclude the $MAX3$ step is unfeasible.
\begin{align*}
D_{SGD}(\bm{\alpha}) &= -\frac{\eta}{2}||\bm{g}_{z_t}||^2 + \ell_{z_t^n}(\w_t),\\
\end{align*}
Hence if the following condition holds we know the $MAX3$ step will be unfeasible:
\begin{align*}
\frac{1}{2}\ell_{z_t^n}(\w_t) &\leq -\frac{\eta}{2}||\bm{g}_{z_t}||^2 + \ell_{z_t^n}(\w_t).\\
\end{align*}
Thus, the converse must hold for the $MAX3$ step to be feasible:
\begin{align*}
\frac{1}{2}\ell_{z_t^n}(\w_t) &\geq -\frac{\eta}{2}||\bm{g}_{z_t}||^2 + \ell_{z_t^n}(\w_t),\\
\end{align*}
which is equivalent to,
\begin{align*}
 \eta &\geq  \frac{\ell_{z_t^n}(\w_t)}{||\bm{g}_{z_t}||^2}.\\
\end{align*}
Using the same logic as stated for the ALI-G step we know this condition is never satisfied for $\eta \leq \frac{1}{2\beta}$.

We now assume $\ell_{z_t^n}(\w_t) \leq \ell_{z_t^n}(\w_t') + \eta\gamma_t\langle \bm{g}_{z_t}, \bm{g}_{z_t}'\rangle$ and thus we know the max feasible value of $D^{MAX3} \leq \frac{1}{2}\ell_{z_t^n}(\w_t') + \frac{1}{2}\eta\gamma_t\langle \bm{g}_{z_t}, \bm{g}_{z_t}'\rangle$, again if either $D^{SGD}$ or $D^{ESGD}$ have larger values, we know the $MAX3$ subproblem is unfeasible:
\begin{align*}
D_{ESGD}(\bm{\alpha}) &= -\frac{\eta}{2}||\bm{g}_{z_t}'||^2 + \ell_{z_t^n}(\w_t')+ \eta\gamma_t\langle \bm{g}_{z_t}, \bm{g}_{z_t}' \rangle.
\end{align*}
Hence for the $MAX3$ step to be valid we must have:
\begin{align*}
    \frac{1}{2}\ell_{z_t^n}(\w_t') + \eta\gamma_t\langle \bm{g}_{z_t}, \bm{g}_{z_t}'\rangle \leq -\frac{\eta}{2}||\bm{g}_{z_t}'||^2 + \ell_{z_t^n}(\w_t')+ \eta\gamma_t\langle \bm{g}_{z_t}, \bm{g}_{z_t}' \rangle,
\end{align*}
which is equivalent to,
\begin{align*}
  \eta \leq  \frac{ \ell_{z_t^n}(\w_t')+ \eta\gamma_t\langle \bm{g}_{z_t}, \bm{g}_{z_t}' \rangle}{||\bm{g}_{z_t}'||^2}.
\end{align*}
Again, we have the same condition as for the EALIG step, which we have already proven can never be feasible for $\eta \leq \frac{1}{2\beta}$. Hence the $MAX3$ subproblem is never feasible for $\eta \leq \frac{1}{2\beta}$.
\end{proof}

\begin{lemma}\label{lemma_1}
For any set of vectors $\bm{a},\bm{b},\bm{c}$ then, the following inequality holds:
\begin{align*}
- 2||\bm{a}-\bm{b}||^2 &\leq -||\bm{a}-\bm{c}||^2 + 2||\bm{b}-\bm{c}||^2.
\end{align*}
\end{lemma}
\begin{proof}
First consider two vectors $\bm{x}$ and $\bm{y}$.
\begin{align*}
0&\leq||\bm{x}-\bm{y}||^2,\\
0&\leq||\bm{x}||^2+||\bm{y}||^2-2\langle \bm{x},\bm{y}\rangle\\
2\langle \bm{x},\bm{y}\rangle&\leq||\bm{x}||^2+||\bm{y}||^2,\\
||\bm{x}+\bm{y}||^2&=||\bm{x}||^2+||\bm{y}||^2+2\langle \bm{x},\bm{y}\rangle,\\
||\bm{x}+\bm{y}||^2&\leq 2||\bm{x}||^2+2||\bm{y}||^2,\\
-2||\bm{x}||^2&\leq 2||\bm{y}||^2 - ||\bm{x}+\bm{y}||^2.
\end{align*}
Setting $\bm{x}=\bm{a}-\bm{b}$ and $\bm{y}=\bm{b}-\bm{c}$ gives the desired result.
\end{proof}

\begin{lemma}\label{lemma:smoothness} Let $z \in \mathcal{Z}$. Assume that $\ell_{z}$ is $\beta$-smooth and non-negative on $\mathbb{R}^d$. Then we have:
\begin{align*}
    \forall (\w) \in \mathbb{R}^d,\hspace{0.2cm} \ell_{z}(\w) \geq \frac{1}{2\beta}||\nabla \ell_{z}(\w)||^2 
\end{align*}
Note that we do not assume that $\ell_{z}$ is convex.
\end{lemma}
\begin{proof}
Let $\w\in\mathbb{R}^d$. By Lemma 3.4 of \citet{Bubeck2015}, we have:
\begin{align*}
    \forall \hspace{0.1cm} \bm{u} \in\mathbb{R}^d,\hspace{0.2cm} |\ell_{z}(\bm{u}) - \ell_{z}(\w) - \nabla \ell_{z}(\w)^\intercal(\bm{u}-\w)| \leq \frac{\beta}{2}||\bm{u}-\w||^2.
\end{align*}
Therefore we can write: 
\begin{align*}
    \forall \hspace{0.1cm} \bm{u} \in\mathbb{R}^d,\hspace{0.2cm}  \ell_{z}(\bm{u}) \leq \ell_{z}(\w) + \nabla \ell_{z}(\w)^\intercal(\bm{u}-\w)|  +\frac{\beta}{2}||\bm{u}-\w||^2.
\end{align*}
And since $\forall \hspace{0.1cm} \bm{u}, l_z(\bm{u})\geq 0$, we have:
\begin{align*}
    \forall \hspace{0.1cm} \bm{u} \in\mathbb{R}^d,\hspace{0.2cm} 0 \leq \ell_{z}(\w) + \nabla \ell_{z}(\w)^\intercal(\bm{u}-\w)|  +\frac{\beta}{2}||\bm{u}-\w||^2.
\end{align*}
We now choose $\bm{u} = -\frac{1}{\beta}\nabla l_z(\w)$, which yeilds:
\begin{align*}
    \forall \hspace{0.1cm} \bm{u} \in\mathbb{R}^d,\hspace{0.2cm} 0 \leq \ell_{z}(\w) - \frac{1}{\beta}||\nabla \ell_{z}(\w)||^2  +\frac{\beta}{2}||\nabla \ell_{z}(\w)||^2,
\end{align*}
which gives the desired result.
\end{proof}



In this section we derive the rate for each of the remaining feasible steps, that is, $SGD$, $ESGD$ and $MAX2$.
\subsection{SGD Subproblem}
\begin{lemma}
We assume that $\Omega=\mathbb{R}^d$, for every $z \in \mathcal{Z}$, $\ell_{z}(w)$ is $\beta$ and satisfies the RSI condition with constant $\mu$. Let $\w^*$ be a solution of $f(\w)$. We assume $\forall z \in \mathcal{Z}, \ell_{z_t^n}(\w^*) = 0$. Then, if we apply BORAT with $\eta \leq \hat{\eta} = \min\{\frac{1}{4\beta},\frac{1}{4\mu}, \frac{\mu}{\beta^2}\}$ and we take the step resulting from the SGD subproblem for all $t$ we have:
\begin{align*}
\mathbb{E}[||\w_{t+1} - \w^*||^2]&\leq (1 -\hat{\eta}\mu)||\w_{t} - \w^*||^2.
\end{align*}
\end{lemma}
\begin{proof}
\begin{align}
||\w_{t+1} - \w^*||^2&\leq ||\Pi_\Omega(\w_{t} - \eta \bm{g}_{z_t}') - \w^*||^2,\\
&\leq ||\w_{t} - \eta \bm{g}_{z_t}- \w^*||^2,\\
&= ||\w_{t} - \w^*||^2 + \eta^2|| \bm{g}_{z_t}||^2 -2\eta \langle  \bm{g}_{z_t}, \w_t - \w^* \rangle, \\
&\leq ||\w_{t} - \w^*||^2 + \eta^2|| \bm{g}_{z_t}||^2 -2\eta\mu||\w_{t} - \w^*||^2.  
\end{align}
We have $||\bm{g}_{z_t}||^2 \leq 2\beta \ell_{z_t^n}(\w_t)$ from Lemma \ref{lemma:smooth_bound} and $ \ell_{z_t^n}(\w_t)\leq \frac{\beta}{2} ||\w_{t} - \w^*||^2$ from smoothness giving $||\bm{g}_{z_t}||^2 \leq \beta^2 ||\w_{t} - \w^*||^2$. We can now upper bound the r.h.s producing:
\begin{align}
||\w_{t+1} - \w^*||^2&\leq ||\w_{t} - \w^*||^2 + \eta^2\beta^2 ||\w_{t} - \w^*||^2 -2\eta\mu||\w_{t} - \w^*||^2,\\
||\w_{t+1} - \w^*||^2&\leq (1 -2\eta\mu+ \eta^2\beta^2)||\w_{t} - \w^*||^2,\\
||\w_{t+1} - \w^*||^2&\leq \left(1 -\eta(2\mu- \eta\beta^2)\right)||\w_{t} - \w^*||^2.
\end{align}
Now, if we select $\eta \leq \hat{\eta} = \min\{\frac{1}{2\beta},\frac{1}{4\mu}, \frac{\mu}{\beta^2}\}$ in the worst case we get:
\begin{align}
||\w_{t+1} - \w^*||^2&\leq (1 -\hat{\eta}\mu)||\w_{t} - \w^*||^2.
\end{align}
Taking expectations with respect to $z_t$:
\begin{align}
\mathbb{E}[||\w_{t+1} - \w^*||^2]&\leq \mathbb{E}[(1 -\hat{\eta}\mu)||\w_{t} - \w^*||^2].
\end{align}
Noting that $\w_t$ does not depend on $z_t$, and neither does $\w^*$ due to the interpolation property:
\begin{align}
\mathbb{E}[||\w_{t+1} - \w^*||^2]&\leq (1 -\hat{\eta}\mu)||\w_{t} - \w^*||^2.
\end{align}
\end{proof}

\subsection{ESGD Subproblem}

\begin{lemma}\label{lemma:gamma_t}
Let $z \in \mathcal{Z}$. We assume that $\ell_{z}$ is $\beta$-smooth and non-negative on $\mathbb{R}^d$. Additionally we assume that $\eta\leq\frac{1}{2\beta}$. If we define $\gamma_t \doteq \min\{1,\frac{\ell_{z_t^n}(\w_t)}{\eta||\bm{g}_{zt}||^2}\}$, then we have:
\begin{align*}
    \gamma_t = 1, \hspace{0.1cm} \forall t
\end{align*}
\end{lemma}
\begin{proof}
From our assumption on $\eta$ and Lemma \ref{lemma:smooth_bound} we have:
\begin{align*}
    \eta \leq\frac{1}{2\beta}\leq \frac{\ell_{z_t^n}(\w_t)}{||\bm{g}_{zt}||^2}.
\end{align*}
Rearranging gives:
\begin{align*}
    1 \leq \frac{\ell_{z_t^n}(\w_t)}{\eta||\bm{g}_{zt}||^2}.
\end{align*}
Plugging in to the the definition of $\gamma_t$, gives the desired result.
\end{proof}

\begin{lemma}
We assume that $\Omega=\mathbb{R}^d$, for every $z \in \mathcal{Z}$, $ \ell_{z}(w)$ is $\beta$ and satisfies the RSI condition with constant $\mu$. Let $\w^*$ be a solution of $f(\w)$. We assume $\forall z \in \mathcal{Z},  \ell_{z_t^n}(\w^*) = 0$. Then, if we apply BORAT with $\eta \leq \hat{\eta} = \min\{\frac{1}{4\beta},\frac{1}{4\mu}, \frac{\mu}{\beta^2}\}$ and we take the step resulting from the ESGD subproblem for all $t$ we have:
\begin{align*}
   \mathbb{E}[||\w_{t+1} - \w^*||^2] \leq (1-\hat{\eta} \mu) ||\w_{t} - \w^*||^2.
\end{align*}
\end{lemma}

\begin{proof} 
\textit{This proof loosely follows work by \citet{vaswani2019}}.
\begin{align}
||\w_{t+1} - \w^*||^2&\leq ||\Pi_\Omega(\w_{t} -  \eta \bm{g}_{z_t}') - \w^*||^2,\\
&\leq ||\w_{t} -  \eta \bm{g}_{z_t}'- \w^*||^2,\\
&= ||\w_{t} - \w^*||^2 + \eta^2|| \bm{g}_{z_t}'||^2 -2\eta \langle  \bm{g}_{z_t}', \w_t - \w^* \rangle ,\\
&= ||\w_{t} - \w^*||^2 + \eta^2|| \bm{g}_{z_t}'||^2 -2\eta \langle  \bm{g}_{z_t}', \w_t' + \eta\gamma_t \bm{g}_{z_t} - \w^* \rangle,\\
&= ||\w_{t} - \w^*||^2 + \eta^2|| \bm{g}_{z_t}'||^2 -2\eta \langle  \bm{g}_{z_t}', \w_t' + \eta \bm{g}_{z_t} - \w^* \rangle, \hspace{0.1cm}(\text{lemma } \ref{lemma:gamma_t}) \\
&= ||\w_{t} - \w^*||^2 + \eta^2|| \bm{g}_{z_t}'||^2 -2\eta \langle  \bm{g}_{z_t}', \w_t' - \w^* \rangle -2 \eta^2\langle  \bm{g}_{z_t}',\bm{g}_{z_t}  \rangle.
\end{align}
Using the RSI condition:
\begin{align}
||\w_{t+1} - \w^*||^2& \leq ||\w_{t} - \w^*||^2 + \eta^2|| \bm{g}_{z_t}'||^2 -2\eta \mu||\w_t' - \w^*||^2 -2 \eta^2\langle  \bm{g}_{z_t}',\bm{g}_{z_t}  \rangle.
\end{align}
Using lemma (\ref{lemma_1}) to upper bound $-||\w_t' - \w^*||^2$,
\begin{align}
&= ||\w_{t} - \w^*||^2 + \eta^2|| \bm{g}_{z_t}'||^2 -\eta \mu||\w^* - \w_t||^2 + 2\eta \mu||\w_t - \w_t'||^2 -2
\eta^2\langle  \bm{g}_{z_t}',\bm{g}_{z_t}  \rangle,\\
&= (1-\eta \mu) ||\w_{t} - \w^*||^2 + \eta^2|| \bm{g}_{z_t}'||^2 + 2\eta \mu||\w_t - \w_t'||^2 -2
\eta^2\langle  \bm{g}_{z_t}',\bm{g}_{z_t}  \rangle,\\
&= (1-\eta \mu) ||\w_{t} - \w^*||^2 + \eta^2||  \bm{g}_{z_t}' - 
\bm{g}_{z_t}||^2 - \eta^2
||\bm{g}_{z_t}||^2 + 2\eta \mu||\w_t - \w_t'||^2,\\
&= (1-\eta \mu) ||\w_{t} - \w^*||^2 + \eta^2||  \bm{g}_{z_t}' - \bm{g}_{z_t}||^2 - \eta^2||\bm{g}_{z_t}||^2 + 2\eta^3 \mu||\bm{g}_{z_t}||^2,\\
&\leq (1-\eta \mu) ||\w_{t} - \w^*||^2 + \eta^2\beta^2||  \w_{t}' - \w_{t}||^2 - \eta^2||\bm{g}_{z_t}||^2 + 2\eta^3 \mu||\bm{g}_{z_t}||^2, (\text{smoothness})\\
&= (1-\eta \mu) ||\w_{t} - \w^*||^2 + \eta^4\beta^2||\bm{g}_{z_t}||^2 - \eta^2||\bm{g}_{z_t}||^2 + 2\eta^3 \mu||\bm{g}_{z_t}||^2,\\
&= (1-\eta \mu) ||\w_{t} - \w^*||^2 + \eta^2\left( \eta^2\beta^2- 1 + 2\eta \mu\right)||\bm{g}_{z_t}||^2,
\end{align}
Taking expectations with respect to $z_t$:
\begin{align}
   \mathbb{E}[||\w_{t+1} - \w^*||^2] &= \mathbb{E}[(1-\eta \mu) ||\w_{t} - \w^*||^2 + \eta^2\left( \eta^2\beta^2- 1 + 2\eta \mu\right)||\bm{g}_{z_t}||^2,
\end{align}
Noting that $\w_t$ does not depend on $z_t$, and neither does $\w^*$ due to the interpolation property: 
\begin{align}
   \mathbb{E}[||\w_{t+1} - \w^*||^2] &= (1-\eta \mu) ||\w_{t} - \w^*||^2 + \eta^2\left( \eta^2\beta^2- 1 + 2\eta \mu\right)\mathbb{E}[||\bm{g}_{z_t}||^2],
\end{align}
If we set $\eta \leq \hat{\eta} = \min\{\frac{1}{2\beta},\frac{1}{4\mu}, \frac{\mu}{\beta^2}\}$ then we have  $\eta^2\beta^2 \leq \frac{1}{4}$,  $2\eta\mu \leq \frac{1}{2}$,hence $\left( \eta^2\beta^2- 1 + 2\eta \mu\right)\leq 0$.
\begin{align}
   \mathbb{E}[||\w_{t+1} - \w^*||^2] &\leq (1-\eta \mu) ||\w_{t} - \w^*||^2.
\end{align}
Hence if we insert the chosen value for $\eta$ then we have:
\begin{align}
   \mathbb{E}[||\w_{t+1} - \w^*||^2] &\leq (1-\hat{\eta} \mu) ||\w_{t} - \w^*||^2.
\end{align}

\end{proof}

\subsection{MAX2 Subproblem}
\begin{lemma}
We assume that $\Omega=\mathbb{R}^d$, for every $z \in \mathcal{Z}$, $l_{z}(w)$ is $\beta$ and satisfies the RSI condition with constant $\mu$. Let $\w^*$ be a solution of $f(\w)$. We assume $\forall z \in \mathcal{Z}, l_{z_t}(\w^*) = 0$. Then, if we apply BORAT with $\eta \leq \hat{\eta} = \min\{\frac{1}{4\beta},\frac{1}{4\mu}, \frac{\mu}{\beta^2}\}$ and we take the step resulting from the MAX2 subproblem for all $t$ we have:
\begin{align*}
\mathbb{E}[||\w_{t+1} - \w^*||^2]&\leq  \left(1-\frac{3}{8}\hat{\eta}\mu\right)^t||\w_0 - \w^*||^2.   
\end{align*}
\end{lemma}
\begin{proof}
Note we assume $\gamma_t = 1$ as proved in Lemma \ref{lemma:gamma_t}.
\begin{align}
||\w_{t+1} - \w^*||^2
 &\leq ||\Pi_\Omega(\w_{t} - \eta \alpha^{1t} \bm{g}_{z_t}- \eta \alpha^{2t} \bm{g}_{z_t}')- \w^*||^2,\\
 &\leq ||\w_{t} - \eta \alpha^{1t} \bm{g}_{z_t}- \eta \alpha^{2t} \bm{g}_{z_t}'- \w^*||^2,\\
 \begin{split}
 &= ||\w_{t} - \w^*||^2 + \eta^2|| \alpha^{1t} \bm{g}_{z_t} + \alpha^{2t} \bm{g}_{z_t}'||^2 \\&-2\eta \langle  \alpha^{1t} \bm{g}_{z_t} + \alpha^{2t} \bm{g}_{z_t}', \w_t - \w^* \rangle, 
 \end{split}\\
 \begin{split}
 &= ||\w_{t} - \w^*||^2 + \eta^2|| \alpha^{1t} \bm{g}_{z_t} + \alpha^{2t} \bm{g}_{z_t}'||^2\\& -2\eta \alpha^{1t} \langle   \bm{g}_{z_t} , \w_t - \w^* \rangle  -2\eta \alpha^{2t} \langle \bm{g}_{z_t}', \w_t - \w^* \rangle,
   \end{split}\\
 \begin{split}
 &= ||\w_{t} - \w^*||^2 + \eta^2|| \alpha^{1t} \bm{g}_{z_t} + \alpha^{2t} \bm{g}_{z_t}'||^2 -2\eta \alpha^{1t} \langle   \bm{g}_{z_t} , \w_t - \w^* \rangle\\& -2\eta \alpha^{2t} \langle  \bm{g}_{z_t}', \w_t' +  \eta \bm{g}_{z_t} - \w^* \rangle, 
  \end{split}\\
 \begin{split}
 &= ||\w_{t} - \w^*||^2 + \eta^2|| \alpha^{1t} \bm{g}_{z_t} + \alpha^{2t} \bm{g}_{z_t}'||^2 -2\eta\alpha^{1t} \langle   \bm{g}_{z_t} , \w_t - \w^* \rangle\\& -2\eta \alpha^{2t} \langle \bm{g}_{z_t}', \w_t' - \w^* \rangle -2\eta^2  \alpha^{2t}\langle  \bm{g}_{z_t}',  \bm{g}_{z_t} \rangle,
  \end{split}\\
 \begin{split}
&= ||\w_{t} - \w^*||^2 + \eta^2|| \alpha^{1t} \bm{g}_{z_t} + \alpha^{2t} \bm{g}_{z_t}'||^2 -2\eta^2  \alpha^{2t}\langle  \bm{g}_{z_t}',  \bm{g}_{z_t} \rangle\\& -2\eta\alpha^{1t} \langle   \bm{g}_{z_t} , \w_t - \w^* \rangle -2\eta\alpha^{2t}  \langle \bm{g}_{z_t}', \w_t' - \w^* \rangle.
\end{split}
\end{align}
We now make use of $-\langle \bm{g}_{z_t}, \w_t-\w^* \rangle \leq -\mu ||\w^* - \w_t||^2$ (RSI condition),
\begin{align}
 \begin{split}
||\w_{t+1} - \w^*||^2&\leq ||\w_{t} - \w^*||^2 + \eta^2|| \alpha^{1t} \bm{g}_{z_t} + \alpha^{2t} \bm{g}_{z_t}'||^2 \\&-2\eta^2  \alpha^{2t}\langle  \bm{g}_{z_t}',  \bm{g}_{z_t} \rangle -2\eta \alpha^{1t} \mu||\w_t - \w^*||^2 -2\eta \langle \alpha^{2t} \bm{g}_{z_t}', \w_t' - \w^* \rangle,   \end{split}\\
 \begin{split}
 ||\w_{t+1} - \w^*||^2&\leq (1-2\eta \alpha^{1t} \mu)||\w_{t} - \w^*||^2 + \eta^2|| \alpha^{1t} \bm{g}_{z_t} + \alpha^{2t} \bm{g}_{z_t}'||^2 \\&-2\eta^2  \alpha^{2t}\langle  \bm{g}_{z_t}',  \bm{g}_{z_t} \rangle -2\eta \alpha^{2t}\langle  \bm{g}_{z_t}', \w_t' - \w^* \rangle.
 \end{split}
\end{align}
Similarly using $-\langle \bm{g}_{z_t}', \w_t'-\w^* \rangle \leq -\mu ||\w^* - \w_t'||^2$ (RSI condition),
\begin{align}
 \begin{split}
||\w_{t+1} - \w^*||^2 &\leq (1-2\eta \alpha^{1t} \mu)||\w_{t} - \w^*||^2 + \eta^2|| \alpha^{1t} \bm{g}_{z_t} + \alpha^{2t} \bm{g}_{z_t}'||^2\\& -2\eta^2  \alpha^{2t}\langle  \bm{g}_{z_t}',  \bm{g}_{z_t} \rangle -2\eta \alpha^{2t} \mu||\w_t' - \w^*||^2.
 \end{split}
\end{align}
We now upper bound $-||\w_t' - \w^*||^2$, using lemma (\ref{lemma_1}):
\begin{align}
\begin{split}
||\w_{t+1} - \w^*||^2 &\leq (1-2\eta \alpha^{1t} \mu)||\w_{t} - \w^*||^2 + \eta^2|| \alpha^{1t} \bm{g}_{z_t} + \alpha^{2t} \bm{g}_{z_t}'||^2 \\&-2\eta^2  \alpha^{2t}\langle  \bm{g}_{z_t}',  \bm{g}_{z_t} \rangle - \alpha^{2t} \eta \mu ||\w_{t} - \w^*||^2 + 2\alpha^{2t} \eta \mu||\w_{t} - \w_t'||^2.
 \end{split}
\end{align}
This gives the following general form:
\begin{align}
\begin{split}
||\w_{t+1} - \w^*||^2  &\leq (1-2\eta \alpha^{1t} \mu - \alpha^{2t} \eta \mu)||\w_{t} - \w^*||^2 + \eta^2|| \alpha^{1t} \bm{g}_{z_t} + \alpha^{2t} \bm{g}_{z_t}'||^2 \\&-2\eta^2  \alpha^{2t}\langle  \bm{g}_{z_t}',  \bm{g}_{z_t} \rangle + 2\alpha^{2t} \eta \mu||\w_{t} - \w_t'||^2. 
 \end{split}
\end{align}
We now use the inequality $||\w_{t} - \w_t'|| = \eta^2||\bm{g}_{z_t}|| \leq \eta^2\beta^2 ||\w_{t} - \w^*||$ to upper bound the final term, see SGD proof for derivation of inequality:
\begin{align}
\begin{split}
||\w_{t+1} - \w^*||^2  &\leq (1-2\eta \alpha^{1t} \mu - \alpha^{2t} \eta \mu)||\w_{t} - \w^*||^2 + \eta^2|| \alpha^{1t} \bm{g}_{z_t} + \alpha^{2t} \bm{g}_{z_t}'||^2\\&-2\eta^2  \alpha^{2t}\langle  \bm{g}_{z_t}',  \bm{g}_{z_t} \rangle + 2\eta^3  \beta^2 \mu\alpha^{2t}||\w_{t} - \w^*||,
\end{split}\\
\begin{split}
||\w_{t+1} - \w^*||^2  &\leq (1-2\eta \alpha^{1t} \mu - \alpha^{2t} \eta \mu+ 2\eta^3\alpha^{2t}  \beta^2 \mu)||\w_{t} - \w^*||^2 \\&+ \eta^2|| \alpha^{1t} \bm{g}_{z_t} + \alpha^{2t} \bm{g}_{z_t}'||^2  -2\eta^2  \alpha^{2t}\langle  \bm{g}_{z_t}',  \bm{g}_{z_t} \rangle. 
\end{split}
\end{align}
We now simplify the last two terms, starting with the first:
\begin{align}
 \eta^2|| \alpha^{1t} \bm{g}_{z_t} + \alpha^{2t} \bm{g}_{z_t}'||^2 = \eta^2((\alpha^{1t})^2 ||\bm{g}_{z_t}||^2 + 2\alpha^{1t}\alpha^{2t} \langle \bm{g}_{z_t}, \bm{g}_{z_t}' \rangle + (\alpha^{2t})^2 ||\bm{g}_{z_t}'||^2).
\end{align}
Plugging in the expressions for  $\alpha^{1t}$, $\alpha^{2t}$,$\gamma_t=1$, grouping like terms and simplifying gives the following. Note, we have excluded a few steps due to spatial constraints:
\begin{align}
\begin{split}
 &\eta^2|| \alpha^{1t} \bm{g}_{z_t} + \alpha^{2t} \bm{g}_{z_t}'||^2 \\&= \frac{(l_{z_t}(\w_t')-l_{z_t}(\w_t))^2 + 2\eta(l_{z_t}(\w_t')-l_{z_t}(\w_t))\langle \bm{g}_{z_t}, \bm{g}_{z_t}' \rangle + \eta^2||\bm{g}_{z_t}||^2||\bm{g}_{z_t}'||^2 }{ ||\bm{g}_{z_t}-\bm{g}_{z_t}'||^2}
 \end{split}
\end{align}
Plugging in $\alpha^{2t}$ into the remaining term gives the following expressions:
\begin{align}
  -2\eta^2  \alpha^{2t}\langle  \bm{g}_{z_t}',  \bm{g}_{z_t} \rangle = \frac{-2\eta^2|| \bm{g}_{z_t}||^2\langle  \bm{g}_{z_t}',  \bm{g}_{z_t} \rangle - 2\eta(l_{z_t}(\w_t')-l_{z_t}(\w_t))\langle  \bm{g}_{z_t}',  \bm{g}_{z_t} \rangle}{||\bm{g}_{z_t}-\bm{g}_{z_t}'||^2}.
\end{align}
Putting these  together,  
\begin{align}
\begin{split}
&\eta^2|| \alpha^{1t} \bm{g}_{z_t} + \alpha^{2t} \bm{g}_{z_t}'||^2 -2\eta^2  \alpha^{2t}\langle  \bm{g}_{z_t}',  \bm{g}_{z_t} \rangle  \\&= \frac{(l_{z_t}(\w_t')-l_{z_t}(\w_t))^2 + 2\eta(l_{z_t}(\w_t')-l_{z_t}(\w_t))\langle \bm{g}_{z_t}, \bm{g}_{z_t}' \rangle  }{||\bm{g}_{z_t}-\bm{g}_{z_t}'||^2}\\ &+ \frac{ \eta^2||\bm{g}_{z_t}||^2||\bm{g}_{z_t}'||^2 -2\eta^2||\bm{g}_{z_t}||^2\langle  \bm{g}_{z_t}',  \bm{g}_{z_t} \rangle-2\eta(l_{z_t}(\w_t')-l_{z_t}(\w_t))\langle  \bm{g}_{z_t}',  \bm{g}_{z_t} \rangle}{||\bm{g}_{z_t}-\bm{g}_{z_t}'||^2}.
\end{split}
\end{align}
Cancelling terms gives,
\begin{align}
\begin{split}
&\eta^2|| \alpha^{1t} \bm{g}_{z_t} + \alpha^{2t} \bm{g}_{z_t}'||^2 -2\eta^2 \alpha^{2t}\langle  \bm{g}_{z_t}',  \bm{g}_{z_t} \rangle  \\&= \frac{(l_{z_t}(\w_t')-l_{z_t}(\w_t))^2+ \eta^2||\bm{g}_{z_t}||^2||\bm{g}_{z_t}'||^2 -2\eta^2||\bm{g}_{z_t}||^2\langle  \bm{g}_{z_t}',  \bm{g}_{z_t} \rangle}{||\bm{g}_{z_t}-\bm{g}_{z_t}'||^2}.
\end{split}
\end{align}
From $\alpha^{2t} \geq 0$ we have $\eta ||\bm{g}_{z_t}||^2 \geq l_{z_t}(\w_t)-l_{z_t}(\w_t')$ hence we can upper bound $(l_{z_t}(\w_t')-l_{z_t}(\w_t))^2$ by $\eta^2||\bm{g}_{z_t}'||^4$:
\begin{align}
\begin{split}
&\eta^2|| \alpha^{1t} \bm{g}_{z_t} + \alpha^{2t} \bm{g}_{z_t}'||^2 -2\eta^2  \alpha^{2t}\langle  \bm{g}_{z_t}',  \bm{g}_{z_t} \rangle  \\&\leq \frac{\eta^2||\bm{g}_{z_t}'||^4+ \eta^2||\bm{g}_{z_t}||^2||\bm{g}_{z_t}'||^2 -2\eta^2||\bm{g}_{z_t}||^2\langle  \bm{g}_{z_t}',  \bm{g}_{z_t} \rangle}{||\bm{g}_{z_t}-\bm{g}_{z_t}'||^2} \leq \eta^2||\bm{g}_{z_t}||^2.
\end{split}
\end{align}
Again we use the inequality $||\w_{t} - \w_t'|| = \eta^2||\bm{g}_{z_t}|| \leq \eta^2\beta^2 ||\w_{t} - \w^*||$:
\begin{align}
\eta^2|| \alpha^{1t} \bm{g}_{z_t} + \alpha^{2t} \bm{g}_{z_t}'||^2 -2\eta^2  \alpha^{2t}\langle  \bm{g}_{z_t}',  \bm{g}_{z_t} \rangle  & \leq \eta^2||\bm{g}_{z_t}||^2 \leq \eta^2\beta^2||\w_{t+1} - \w^*||^2
\end{align}
Hence we get the following expression:
\begin{align}
||\w_{t+1} - \w^*||^2  &\leq (1-2\eta \alpha^{1t} \mu - \alpha^{2t} \eta \mu+ 2\eta^3\alpha^{2t}  \beta^2 \mu)||\w_{t} - \w^*||^2 + \eta^2\beta^2||\w_{t+1} - \w^*||^2 \rangle,\\
||\w_{t+1} - \w^*||^2  &\leq (1-2\eta \alpha^{1t} \mu - \alpha^{2t} \eta \mu+ 2\eta^3\alpha^{2t}  \beta^2 \mu+ \eta^2\beta^2)||\w_{t} - \w^*||^2,\\
||\w_{t+1} - \w^*||^2  &\leq (1-\eta \mu - \alpha^{1t} \eta \mu+ 2\eta^3 \beta^2 \mu - 2\eta^3\alpha^{1t}  \beta^2 \mu+ \eta^2\beta^2)||\w_{t} - \w^*||^2.
\end{align}
Upper bounding $-\alpha^{1t}$ by $0$,
\begin{align}
||\w_{t+1} - \w^*||^2  &\leq (1-\eta \mu + 2\eta^3 \beta^2 \mu + \eta^2\beta^2)||\w_{t} - \w^*||^2.
\end{align}
Taking expectations with respect to $z_t$:
\begin{align}
\mathbb{E}[||\w_{t+1} - \w^*||^2] 
 &\leq \mathbb{E}[(1-\eta \mu + 2\eta^3 \beta^2 \mu + \eta^2\beta^2)||\w_{t} - \w^*||^2].
\end{align}
Noting that $\w_k$ does not depend on $z_t$, and neither does $\w^*$ due to the interpolation property: 
\begin{align}
\mathbb{E}[||\w_{t+1} - \w^*||^2] 
 &\leq (1-\eta \mu + 2\eta^3 \beta^2 \mu + \eta^2\beta^2)||\w_{t} - \w^*||^2.
\end{align}
For the this step to be convergent we need the following condition to hold $ 2\eta^3 \beta^2 \mu + \eta^2\beta^2-\eta \mu\leq0$. However, for $\eta \leq \hat{\eta} = \min\{\frac{1}{4\beta},\frac{1}{4\mu}, \frac{\mu}{2\beta^2}\}$ we have:
\begin{align}
      2\eta^3 \beta^2 \mu + \eta^2\beta^2-\eta \mu \leq \frac{1}{8}\eta \mu + \frac{1}{2}\eta \mu-\eta \mu \leq  -\frac{3}{8}\eta \mu.
\end{align}

Hence, we recover the rate:
\begin{align}
\mathbb{E}[||\w_{t+1} - \w^*||^2]  &\leq \left(1-\frac{3}{8}\hat{\eta}\mu\right)^t||\w_0 - \w^*||^2.   
\end{align}

\end{proof}

\subsection{Worst Case Rate}
It is clear by inspection that the worst case rate derived corresponds to the MAX2 subproblem. Hence in the worst case this step is taken for all $t$, and thus a trivial induction gives the result of Theorem \ref{th:rsi}.

\section{Additional Results}\label{App:Additional_Results}

\subsection{Cifar Hyperparameters and Variance}
Here we detail the hyperparameters and variance for the ALI-G and BORAT results reported in table \ref{tab:cifar1}. For other optimisation methods please refer to Appendix E of \cite{Berrada2019a}.
\begin{table}[h]
  \caption{Cifar Hyperparameters (BORAT)}
  \label{tab:cifar1}
  \centering
  \begin{tabular}{ccccccccc}
    \toprule
    Data Set & Model & \multicolumn{4}{c}{ Hyperparameters} & \multicolumn{2}{c}{Test Accuracy}\\
    \cmidrule(r){3-6}
    \cmidrule(r){7-8}
    &  &  $N$& $\hspace{0.2cm}\eta$&$r$& batchzise& Mean & std\\
    \midrule
    \multirow{6}{*}{CIFAR10}&\multirow{3}{*}{WRN}&2&\hspace{0.2cm}0.1	&	50	&	128	&	95.4	&	0.13\\
                            &                    &3&\hspace{0.2cm}1	&	100	&	128	&	95.4	&	0.05\\
                            &                    &5&\hspace{0.2cm}1	&	75	&	128	&	95.0	&	0.08\\
    \cmidrule(r){2-8}
                            &\multirow{3}{*}{DN}&2&\hspace{0.2cm}0.1	&	100	&	64	&	94.5	&	0.09\\
                            &                    &3&\hspace{0.2cm}1	&	75	&	256	&	94.9	&	0.13\\
                            &                    &5&\hspace{0.2cm}1	&	75	&	128	&	94.9	&	0.13\\
    \midrule
    \multirow{6}{*}{CIFAR100}&\multirow{3}{*}{WRN}&2&\hspace{0.2cm}0.1	&	50	&	512	&	76.1	&	0.21\\
                            &                    &3&\hspace{0.2cm}0.1	&	50	&	256	&	76.0	&	0.16\\
                            &                    &5&\hspace{0.2cm}0.1	&	50	&	128	&	75.8	&	0.22\\
    \cmidrule(r){2-8}
                            &\multirow{3}{*}{DN} &2&\hspace{0.2cm}0.1	&	75	&	256	&	76.2	&	0.14\\
                            &                    &3&\hspace{0.2cm}0.1	&	75	&	128	&	76.5	&	0.38\\
                            &                    &5&\hspace{0.2cm}0.1	&	75	&	64	&	75.7	&	0.03\\
    \bottomrule
  \end{tabular}
\end{table}


\subsection{Empirical Run Time}

Here we detail the effect of increasing $N$ on the run time of BORAT. Due to each update requiring $N-1$ gradient evaluations BORAT with $N\geq2$ take significantly longer between updates than other methods. However, BORAT achieves good empirical convergence rates taking $N-1$ fewer parameter updates than other methods, as shown in the results section. Hence we consider the epoch time and show for each pass through the data BORAT has a similar run time to SGD.\\
Increasing $N$ both increases the run time of Algorithm \ref{alg:dualsol} but additionally BORAT must compute extra dot products when calculating $Q$. A naive implementation of Algorithm \ref{alg:dualsol} has a time complexity of $\mathcal{O}\left(\sum_{k=1}^N\frac{N!}{k!(N-k)!}k^3\right)$. However if we exploit the parallel nature of this algorithm where the sub problems are solved simultaneously, the time complexity reduces to $\mathcal{O}\left(N^3\right)$ as discussed in Section \ref{dual_solution}. Additionally we need only run Algorithm \ref{alg:dualsol} once every $N-1$ batches so the per epoch time complexity is $\mathcal{O}\left(N^2\right)$. Of course in practice Algorithm 1 is only responsible for a fraction of the run time, where its contribution is determined by the relative size of the model and $N$.\\
Table \ref{tab:cifar_timing} show with a parallel implementation the effect of the this extra computation on the training epoch time isn't significant and the time complexity scales approximately linearly with $N$. Moreover, Table \ref{tab:cifar_timing} shows for large scale learning problems, such as ImageNet the extra run time when increasing $N$ is negligible. 
\begin{table}[H]
  \centering
  \begin{tabular}{c ccccc }
    \toprule
   Optimiser & SGD & BORAT & BORAT & BORAT & BORAT  \\
    \midrule
    $N$& 1 & 2 & 3 & 4 & 5 \\
    Time (s) &51.0&55.6&68.2&69.2&74.3 \\
    \midrule
       Optimiser & BORAT & BORAT & BORAT & BORAT & BORAT  \\
         \midrule
        $N$& 6 & 7 & 8 & 9 & 10 \\
    Time (s) &77.8&82.8&88.7&94.1&99.5 \\
    \bottomrule
  \end{tabular}
     \caption{Average BORAT training epoch time for CIFAR100 data set, shown for varying N. Time quoted using a batch size of 128, CIFAR100, CE loss, a Wide ResNet 40-4, and a parallel implantation of BORAT. All Optimiser had access to 3 CPU cores, and one TITAN Xp GPU.}
     \label{tab:cifar_timing}
\end{table}

\begin{table}[H]
  \centering
  \begin{tabular}{c ccccc }
    \toprule
   Optimiser &  BORAT & BORAT & BORAT   \\
    \midrule
    $N$& 2 & 3 & 5 \\
    Time (s) &885.50&910.49&934.79 \\
    \bottomrule
  \end{tabular} 
     \caption{Average BORAT training epoch time for ImageNet data set, shown for varying N. Time quoted using a batch size of 1024, ImageNet, CE loss, a ResNet18, and a parallel implantation of BORAT. All Optimiser had access to 12 CPU cores, and 4 TITAN Xp GPUs.}
     \label{tab:imagenet_timing}
\end{table}

\vskip 0.2in

\bibliography{standardstrings,oval}

\begin{thebibliography}{64}
\providecommand{\natexlab}[1]{#1}
\providecommand{\url}[1]{\texttt{#1}}
\expandafter\ifx\csname urlstyle\endcsname\relax
  \providecommand{\doi}[1]{doi: #1}\else
  \providecommand{\doi}{doi: \begingroup \urlstyle{rm}\Url}\fi

\bibitem[Abadi et~al.(2015)Abadi, Agarwal, Barham, Brevdo, Chen, Citro,
  Corrado, Davis, Dean, Devin, Ghemawat, Goodfellow, Harp, Irving, Isard, Jia,
  Jozefowicz, Kaiser, Kudlur, Levenberg, Man, Monga, Moore, Murray, Olah,
  Schuster, Shlens, Steiner, Sutskever, Talwar, Tucker, Vanhoucke, Vasudevan,
  Vi, Vinyals, Warden, Wattenberg, Wicke, Yu, and Zheng]{Abadi2015}
Mart\'{i}n Abadi, Ashish Agarwal, Paul Barham, Eugene Brevdo, Zhifeng Chen,
  Craig Citro, GregS. Corrado, Andy Davis, Jeffrey Dean, Matthieu Devin, Sanjay
  Ghemawat, Ian Goodfellow, Andrew Harp, Geoffrey Irving, Michael Isard,
  Yangqing Jia, Rafal Jozefowicz, Lukasz Kaiser, Manjunath Kudlur, Josh
  Levenberg, Dandelion\'{e} Man, Rajat Monga, Sherry Moore, Derek Murray, Chris
  Olah, Mike Schuster, Jonathon Shlens, Benoit Steiner, Ilya Sutskever, Kunal
  Talwar, Paul Tucker, Vincent Vanhoucke, Vijay Vasudevan, Fernanda\'{e}gas Vi,
  Oriol Vinyals, Pete Warden, Martin Wattenberg, Martin Wicke, Yuan Yu, and
  Xiaoqiang Zheng.
\newblock \emph{{TensorFlow}: Large-Scale Machine Learning on Heterogeneous
  Systems}, 2015.
\newblock Software available from tensorflow.org.

\bibitem[Asi and Duchi(2019)]{asi2019}
Hilal Asi and John~C Duchi.
\newblock Stochastic (approximate) proximal point methods: Convergence,
  optimality, and adaptivity.
\newblock \emph{SIAM Journal on Optimization}, 2019.

\bibitem[Auslender(2009)]{Auslender2009}
Alfred Auslender.
\newblock Bundle methods for machine learning.
\newblock \emph{Numerical Methods for Nondifferentiable Convex Optimization.
  Mathematical Programming Study}, 2009.

\bibitem[Bach et~al.(2011)Bach, Jenatton, Mairal, and Obozinski]{Bach2011}
Francis~R. Bach, Rodolphe Jenatton, Julien Mairal, and Guillaume Obozinski.
\newblock Optimization with sparsity-inducing penalties.
\newblock \emph{CoRR}, 2011.

\bibitem[Baydin et~al.(2018)Baydin, Cornish, Rubio, Schmidt, and
  Wood]{Baydin2018}
Atilim~Gunes Baydin, Robert Cornish, David~Martinez Rubio, Mark Schmidt, and
  Frank Wood.
\newblock Online learning rate adaptation with hypergradient descent.
\newblock \emph{International Conference on Learning Representations}, 2018.

\bibitem[Bernstein et~al.(2018)Bernstein, Wang, Azizzadenesheli, and
  Anandkumar]{Bernstein2018}
Jeremy Bernstein, Yu-Xiang Wang, Kamyar Azizzadenesheli, and Anima Anandkumar.
\newblock signsgd: Compressed optimisation for non-convex problems.
\newblock \emph{International Conference on Machine Learning}, 2018.

\bibitem[Berrada et~al.(2019{\natexlab{a}})Berrada, Zisserman, and
  Kumar]{Berrada2019}
Leonard Berrada, Andrew Zisserman, and M~Pawan Kumar.
\newblock Deep {Frank-Wolfe} for neural network optimization.
\newblock \emph{International Conference on Learning Representations},
  2019{\natexlab{a}}.

\bibitem[Berrada et~al.(2019{\natexlab{b}})Berrada, Zisserman, and
  Kumar]{Berrada2019a}
Leonard Berrada, Andrew Zisserman, and M~Pawan Kumar.
\newblock Training neural networks for and by interpolation.
\newblock \emph{International Conference on Machine Learning},
  2019{\natexlab{b}}.

\bibitem[Bertsekas(2009)]{Bertsekas2009}
Dimitri~P. Bertsekas.
\newblock \emph{Convex Optimization Theory}.
\newblock Athena Scientific, 2009.

\bibitem[Bowman et~al.(2015)Bowman, Angeli, Potts, and Manning]{Bowman2015}
Samuel~R Bowman, Gabor Angeli, Christopher Potts, and Christopher~D Manning.
\newblock A large annotated corpus for learning natural language inference.
\newblock \emph{Conference on Empirical Methods in Natural Language
  Processing}, 2015.

\bibitem[Bubeck(2015)]{Bubeck2015}
S{\'e}bastien Bubeck.
\newblock Convex optimization: Algorithms and complexity.
\newblock \emph{Foundations and Trends in Machine Learning}, 2015.

\bibitem[Chen and Gu(2018)]{Chen2018}
Jinghui Chen and Quanquan Gu.
\newblock Padam: Closing the generalization gap of adaptive gradient methods in
  training deep neural networks.
\newblock \emph{arXiv preprint}, 2018.

\bibitem[Chen et~al.(2019)Chen, Liu, Sun, and Hong]{Chen2019}
Xiangyi Chen, Sijia Liu, Ruoyu Sun, and Mingyi Hong.
\newblock On the convergence of a class of adam-type algorithms for non-convex
  optimization.
\newblock \emph{International Conference on Learning Representations}, 2019.

\bibitem[Conneau et~al.(2017)Conneau, Kiela, Schwenk, Barrault, and
  Bordes]{Conneau2017}
Alexis Conneau, Douwe Kiela, Holger Schwenk, Loic Barrault, and Antoine Bordes.
\newblock Supervised learning of universal sentence representations from
  natural language inference data.
\newblock \emph{Conference on Empirical Methods in Natural Language
  Processing}, 2017.

\bibitem[D{\'e}fossez and Bach(2017)]{Defossez2017}
Alexandre D{\'e}fossez and Francis Bach.
\newblock Adabatch: Efficient gradient aggregation rules for sequential and
  parallel stochastic gradient methods.
\newblock \emph{arXiv preprint}, 2017.

\bibitem[Duchi et~al.(2011)Duchi, Hazan, and Singer]{duchi2011}
John Duchi, Elad Hazan, and Yoram Singer.
\newblock Adaptive subgradient methods for online learning and stochastic
  optimization.
\newblock \emph{Journal of Machine Learning Research}, 2011.

\bibitem[Frank and Wolfe(1956)]{Frank1956}
Marguerite Frank and Philip Wolfe.
\newblock An algorithm for quadratic programming.
\newblock \emph{Naval Research Logistics Quarterly}, 1956.

\bibitem[He et~al.(2016)He, Zhang, Ren, and Sun]{He2016}
Kaiming He, Xiangyu Zhang, Shaoqing Ren, and Jian Sun.
\newblock Deep residual learning for image recognition.
\newblock \emph{Conference on Computer Vision and Pattern Recognition}, 2016.

\bibitem[Henriques et~al.(2019)Henriques, Ehrhardt, Albanie, and
  Vedaldi]{Henriques2019}
Jo{\~a}o~F Henriques, Sebastien Ehrhardt, Samuel Albanie, and Andrea Vedaldi.
\newblock Small steps and giant leaps: Minimal newton solvers for deep
  learning.
\newblock \emph{International Conference on Computer Vision}, 2019.

\bibitem[Huang et~al.(2017)Huang, Liu, Weinberger, and van~der
  Maaten]{Huang2017a}
Gao Huang, Zhuang Liu, Kilian~Q Weinberger, and Laurens van~der Maaten.
\newblock Densely connected convolutional networks.
\newblock \emph{Conference on Computer Vision and Pattern Recognition}, 2017.

\bibitem[Kingma and Ba(2015)]{Kingma2015}
Diederik~P. Kingma and Jimmy Ba.
\newblock Adam: A method for stochastic optimization.
\newblock \emph{International Conference on Learning Representations}, 2015.

\bibitem[Kingma and Welling(2014)]{kingma2014}
Diederik~P. Kingma and Max Welling.
\newblock Auto-encoding variational {B}ayes.
\newblock \emph{International Conference on Learning Representations}, 2014.

\bibitem[Lacoste-Julien and Jaggi(2013)]{Lacoste-Julien2013}
Simon Lacoste-Julien and Martin Jaggi.
\newblock Block-coordinate {F}rank-{W}olfe optimization for structural {SVMs}.
\newblock \emph{International Conference on Machine Learning}, 2013.

\bibitem[Lapin et~al.(2016)Lapin, Hein, and Schiele]{Lapin2016}
Maksim Lapin, Matthias Hein, and Bernt Schiele.
\newblock Loss functions for top-k error: Analysis and insights.
\newblock \emph{Conference on Computer Vision and Pattern Recognition}, 2016.

\bibitem[Lemaréchal et~al.(1995)Lemaréchal, Nemirovski, and
  Nesterov]{Lemarechal1995}
Claude Lemaréchal, Arkadi Nemirovski, and Yurii Nesterov.
\newblock New variants of bundle methods.
\newblock \emph{Math. Program}, 1995.

\bibitem[Levy(2017)]{Levy2017}
Kfir Levy.
\newblock Online to offline conversions, universality and adaptive minibatch
  sizes.
\newblock \emph{Neural Information Processing Systems}, 2017.

\bibitem[Li et~al.(2020)Li, Soltanolkotabi, and Oymak]{Mingchen2020}
Mingchen Li, Mahdi Soltanolkotabi, and Samet Oymak.
\newblock Gradient descent with early stopping is provably robust to label
  noise for overparameterized neural networks.
\newblock \emph{Journal of Machine Learning Research}, 2020.

\bibitem[Li and Orabona(2019)]{Li2019}
Xiaoyu Li and Francesco Orabona.
\newblock On the convergence of stochastic gradient descent with adaptive
  stepsizes.
\newblock \emph{International Conference on Artificial Intelligence and
  Statistics}, 2019.

\bibitem[Liu et~al.(2019{\natexlab{a}})Liu, Arnon, Lazarus, Barrett, and
  Kochenderfer]{Liu2019}
Changliu Liu, Tomer Arnon, Christopher Lazarus, Clark Barrett, and Mykel~J
  Kochenderfer.
\newblock Algorithms for verifying deep neural networks.
\newblock \emph{arXiv:1903.06758}, 2019{\natexlab{a}}.

\bibitem[Liu et~al.(2019{\natexlab{b}})Liu, Jiang, He, Chen, Liu, Gao, and
  Han]{Liu2019a}
Liyuan Liu, Haoming Jiang, Pengcheng He, Weizhu Chen, Xiaodong Liu, Jianfeng
  Gao, and Jiawei Han.
\newblock On the variance of the adaptive learning rate and beyond.
\newblock \emph{arXiv preprint}, 2019{\natexlab{b}}.

\bibitem[Locatello et~al.(2017)Locatello, Khanna, Tschannen, and
  Jaggi]{Locatello2017}
Francesco Locatello, Rajiv Khanna, Michael Tschannen, and Martin Jaggi.
\newblock A unified optimization view on generalized matching pursuit and
  frank-wolfe.
\newblock \emph{International Conference on Artificial Intelligence and
  Statistics}, 2017.

\bibitem[Loizou et~al.(2020)Loizou, Vaswani, Laradji, and
  Lacoste-Julien]{Loizou2020}
Nicolas Loizou, Sharan Vaswani, Issam Laradji, and Simon Lacoste-Julien.
\newblock Stochastic polyak step-size for sgd: An adaptive learning rate for
  fast convergence.
\newblock \emph{arXiv preprint}, 2020.

\bibitem[Loshchilov and Hutter(2017)]{loshchilov2016}
Ilya Loshchilov and Frank Hutter.
\newblock Sgdr: Stochastic gradient descent with warm restarts.
\newblock \emph{International Conference on Learning Representations}, 2017.

\bibitem[Loshchilov and Hutter(2019)]{Loshchilov2019}
Ilya Loshchilov and Frank Hutter.
\newblock Fixing weight decay regularization in adam.
\newblock \emph{International Conference on Learning Representations}, 2019.

\bibitem[Luo et~al.(2019)Luo, Xiong, Liu, and Sun]{Luo2019}
Liangchen Luo, Yuanhao Xiong, Yan Liu, and Xu~Sun.
\newblock Adaptive gradient methods with dynamic bound of learning rate.
\newblock \emph{International Conference on Learning Representations}, 2019.

\bibitem[Ma et~al.(2018{\natexlab{a}})Ma, Bassily, and Belkin]{Ma2018a}
Siyuan Ma, Raef Bassily, and Mikhail Belkin.
\newblock The power of interpolation: Understanding the effectiveness of sgd in
  modern over-parametrized learning.
\newblock \emph{International Conference on Machine Learning},
  2018{\natexlab{a}}.

\bibitem[Ma et~al.(2018{\natexlab{b}})Ma, Li, Wang, Erfani, Wijewickrema,
  Schoenebeck, Song, Houle, and Bailey]{ma2018}
Xingjun Ma, Bo~Li, Yisen Wang, Sarah~M Erfani, Sudanthi Wijewickrema, Grant
  Schoenebeck, Dawn Song, Michael~E Houle, and James Bailey.
\newblock Characterizing adversarial subspaces using local intrinsic
  dimensionality.
\newblock \emph{International Conference on Learning Representations},
  2018{\natexlab{b}}.

\bibitem[Martens and Grosse(2015)]{Martens2015}
James Martens and Roger Grosse.
\newblock Optimizing neural networks with {Kronecker}-factored approximate
  curvature.
\newblock \emph{International Conference on Machine Learning}, 2015.

\bibitem[Mukkamala and Hein(2017)]{Mukkamala2017}
Mahesh~Chandra Mukkamala and Matthias Hein.
\newblock Variants of rmsprop and adagrad with logarithmic regret bounds.
\newblock \emph{International Conference on Machine Learning}, 2017.

\bibitem[Oberman and Prazeres(2019)]{Oberman2019}
Adam~M Oberman and Mariana Prazeres.
\newblock Stochastic gradient descent with polyak's learning rate.
\newblock \emph{arXiv preprint}, 2019.

\bibitem[Orabona and P{\'a}l(2015)]{Orabona2015}
Francesco Orabona and D{\'a}vid P{\'a}l.
\newblock Scale-free algorithms for online linear optimization.
\newblock \emph{International Conference on Algorithmic Learning Theory}, 2015.

\bibitem[Paszke et~al.(2017)Paszke, Gross, Chintala, Chanan, Yang, DeVito, Lin,
  Desmaison, Antiga, and Lerer]{Paszke2017}
Adam Paszke, Sam Gross, Soumith Chintala, Gregory Chanan, Edward Yang, Zachary
  DeVito, Zeming Lin, Alban Desmaison, Luca Antiga, and Adam Lerer.
\newblock Automatic differentiation in pytorch.
\newblock \emph{NIPS Autodiff Workshop}, 2017.

\bibitem[Polyak(1969)]{Polyak1969}
Boris~Teodorovich Polyak.
\newblock Minimization of unsmooth functionals.
\newblock \emph{USSR Computational Mathematics and Mathematical Physics}, 1969.

\bibitem[Reddi et~al.(2018)Reddi, Kale, and Kumar]{Reddi2018}
Sashank~J Reddi, Satyen Kale, and Sanjiv Kumar.
\newblock On the convergence of adam and beyond.
\newblock \emph{International Conference on Learning Representations}, 2018.

\bibitem[Robbins and Monro(1951)]{robbins1951}
Herbert Robbins and Sutton Monro.
\newblock A stochastic approximation method.
\newblock \emph{The annals of mathematical statistics}, 1951.

\bibitem[Rolinek and Martius(2018)]{Rolinek2019}
Michal Rolinek and Georg Martius.
\newblock L4: Practical loss-based stepsize adaptation for deep learning.
\newblock \emph{Neural Information Processing Systems}, 2018.

\bibitem[Schaul et~al.(2013)Schaul, Zhang, and LeCun]{Schaul2013}
Tom Schaul, Sixin Zhang, and Yann LeCun.
\newblock No more pesky learning rates.
\newblock \emph{International Conference on Machine Learning}, 2013.

\bibitem[Schneider et~al.(2019)Schneider, Balles, and Hennig]{Schneider2019}
Frank Schneider, Lukas Balles, and Philipp Hennig.
\newblock Deep{OBS}: A deep learning optimizer benchmark suite.
\newblock \emph{International Conference on Learning Representations}, 2019.

\bibitem[Shalev-Shwartz and Zhang(2016)]{Shalev-Shwartz2016}
Shai Shalev-Shwartz and Tong Zhang.
\newblock Accelerated proximal stochastic dual coordinate ascent for
  regularized loss minimization.
\newblock \emph{Mathematical Programming}, 2016.

\bibitem[Shazeer and Stern(2018)]{Shazeer2018}
Noam Shazeer and Mitchell Stern.
\newblock Adafactor: Adaptive learning rates with sublinear memory cost.
\newblock \emph{International Conference on Machine Learning}, 2018.

\bibitem[Smola et~al.(2007)Smola, Vishwanathan, and Le]{Smola2007}
Alexander~J. Smola, S.~V.~N. Vishwanathan, and Quoc~V. Le.
\newblock Bundle methods for machine learning.
\newblock \emph{Neural Information Processing Systems}, 2007.

\bibitem[Szegedy et~al.(2015)Szegedy, Liu, Jia, Sermanet, Reed, Anguelov,
  Erhan, Vanhoucke, and Rabinovich]{Szegedy2015}
Christian Szegedy, Wei Liu, Yangqing Jia, Pierre Sermanet, Scott Reed, Dragomir
  Anguelov, Dumitru Erhan, Vincent Vanhoucke, and Andrew Rabinovich.
\newblock Going deeper with convolutions.
\newblock \emph{Conference on Computer Vision and Pattern Recognition}, 2015.

\bibitem[Tan et~al.(2016)Tan, Ma, Dai, and Qian]{Tan2016}
Conghui Tan, Shiqian Ma, Yu-Hong Dai, and Yuqiu Qian.
\newblock Barzilai-borwein step size for stochastic gradient descent.
\newblock \emph{Neural Information Processing Systems}, 2016.

\bibitem[Tieleman and Hinton(2012)]{Tieleman2012}
Tijmen Tieleman and Geoffrey Hinton.
\newblock Lecture 6.5-rmsprop: Divide the gradient by a running average of its
  recent magnitude.
\newblock \emph{COURSERA: Neural networks for machine learning}, 2012.

\bibitem[Vaswani et~al.(2019{\natexlab{a}})Vaswani, Bach, and
  Schmidt]{vaswani2019}
Sharan Vaswani, Francis Bach, and Mark Schmidt.
\newblock Fast and faster convergence of sgd for over-parameterized models and
  an accelerated perceptron.
\newblock \emph{International Conference on Artificial Intelligence and
  Statistics}, 2019{\natexlab{a}}.

\bibitem[Vaswani et~al.(2019{\natexlab{b}})Vaswani, Mishkin, Laradji, Schmidt,
  Gidel, and Lacoste-Julien]{Vaswani2019a}
Sharan Vaswani, Aaron Mishkin, Issam Laradji, Mark Schmidt, Gauthier Gidel, and
  Simon Lacoste-Julien.
\newblock Painless stochastic gradient: Interpolation, line-search, and
  convergence rates.
\newblock \emph{arXiv preprint}, 2019{\natexlab{b}}.

\bibitem[Wilson et~al.(2017)Wilson, Roelofs, Stern, Srebro, and
  Recht]{Wilson2017}
Ashia~C Wilson, Rebecca Roelofs, Mitchell Stern, Nati Srebro, and Benjamin
  Recht.
\newblock The marginal value of adaptive gradient methods in machine learning.
\newblock \emph{Neural Information Processing Systems}, 2017.

\bibitem[Wu et~al.(2018)Wu, Ward, and Bottou]{Wu2018}
Xiaoxia Wu, Rachel Ward, and L{\'e}on Bottou.
\newblock {WNGrad}: Learn the learning rate in gradient descent.
\newblock \emph{arXiv preprint}, 2018.

\bibitem[Zagoruyko and Komodakis(2016)]{Zagoruyko2016}
Sergey Zagoruyko and Nikos Komodakis.
\newblock Wide residual networks.
\newblock \emph{British Machine Vision Conference}, 2016.

\bibitem[Zaheer et~al.(2018)Zaheer, Reddi, Sachan, Kale, and Kumar]{Zaheer2018}
Manzil Zaheer, Sashank Reddi, Devendra Sachan, Satyen Kale, and Sanjiv Kumar.
\newblock Adaptive methods for nonconvex optimization.
\newblock \emph{Neural Information Processing Systems}, 2018.

\bibitem[Zeiler(2012)]{Zeiler2012}
Matthew Zeiler.
\newblock {ADADELTA:} an adaptive learning rate method.
\newblock \emph{arXiv preprint}, 2012.

\bibitem[Zhang et~al.(2017)Zhang, Wu, and Wang]{Zhang2017a}
Ziming Zhang, Yuanwei Wu, and Guanghui Wang.
\newblock Bpgrad: Towards global optimality in deep learning via branch and
  pruning.
\newblock \emph{Conference on Computer Vision and Pattern Recognition}, 2017.

\bibitem[Zheng and Kwok(2017)]{Zheng2017}
Shuai Zheng and James~T Kwok.
\newblock Follow the moving leader in deep learning.
\newblock \emph{International Conference on Machine Learning}, 2017.

\bibitem[Zhou et~al.(2019)Zhou, Yang, Zhang, Liang, and Tarokh]{Zhou2019}
Yi~Zhou, Junjie Yang, Huishuai Zhang, Yingbin Liang, and Vahid Tarokh.
\newblock Sgd converges to global minimum in deep learning via star-convex
  path.
\newblock \emph{International Conference on Learning Representations}, 2019.

\end{thebibliography}

\end{document}